\documentclass{article} 
\usepackage{iclr2027_conference,times}
\usepackage{graphicx}
\usepackage{subcaption}
\usepackage{amsmath}
\usepackage{amssymb}
\usepackage{amsthm}
\usepackage{algorithm}
\usepackage[noend]{algpseudocode}
\usepackage{float}
\usepackage{enumitem}

\usepackage{amsmath,amsfonts,bm}

\def\eqref#1{equation~\ref{#1}}

\def\1{\bm{1}}

\def\eps{{\epsilon}}

\DeclareMathAlphabet{\mathsfit}{\encodingdefault}{\sfdefault}{m}{sl}
\SetMathAlphabet{\mathsfit}{bold}{\encodingdefault}{\sfdefault}{bx}{n}

\usepackage[nohints]{minitoc}

\usepackage{hyperref}
\usepackage{cleveref}
\usepackage{url}
\usepackage{dsfont}

\newtheorem{lemma}{Lemma}
\newtheorem{theorem}{Theorem}
\newtheorem{remark}{Remark}

\newtheorem{corollary}{Corollary}
\makeatletter
\newtheorem*{rep@theorem}{\rep@title}
\newcommand{\newreptheorem}[2]{%
	\newenvironment{rep#1}[1]{%
		\def\rep@title{\textbf{#2} \ref{##1}}%
		\begin{rep@theorem}}%
		{\end{rep@theorem}}}
\makeatother
\newreptheorem{theorem}{Theorem}
\newreptheorem{lemma}{Lemma}
\newreptheorem{remark}{Remark}
\newtheorem{assumption}{Assumption}
\newcommand{\expectation}{\mathbb{E}}
\newcommand{\variance}{\mathbb{V}}

\newcommand{\bigO}{\mathcal{O}}
\newcommand{\bigint}{\scalebox{1.6}{$\displaystyle\int$}}

\title{Performative Privacy: When Differential Privacy Maximizes Utility}

\author{Uddalak Mukherjee \\
Dauphine PSL,\\
\texttt{uddalak.mukherjee@dauphine.psl.eu}
\And
Edwige Cyffers \\
CNRS Dauphine PSL,\\
\texttt{edwige.cyffers@cnrs.fr} 
\And
Yann Chevaleyre \\
Dauphine PSL,\\
\texttt{yann.chevaleyre@lamsade.dauphine.fr } \\
}

\begin{document}

\iclrfinalcopy
\maketitle
\doparttoc
\faketableofcontents


\begin{abstract}
Privacy-preserving learning is often motivated by the idea that protecting users' data can preserve trust and thus participation, improving utility in the long term. However, this claim has not been formalized so far. In parallel, performative learning provides a framework for studying learning systems whose deployment affects the data they later observe. In this work, we bring these two perspectives together and introduce \emph{performative privacy}, where data leakage reduces future participation. We study a simple model where agents repeatedly contribute data for mean estimation but may leave the system when their data is leaked. Privacy is implemented through differentially private mechanisms, creating a trade-off between estimation noise and future participation. We show, through a theoretical study of the dynamics and numerical experiments, that a finite privacy budget can outperform non-private estimation in the long term when the feedback loop between leakage and participation is sufficiently strong. We then demonstrate a similar behavior for learning tasks with a simple linear regression setting. This provides first evidence that differential privacy can be optimal not only as a protection mechanism, but also from the perspective of long-term utility.
\end{abstract}


\section{Introduction}

Protecting privacy in machine learning is motivated by moral, ethical, and practical reasons \citep{solove2008understanding}. Access to sensitive data, such as health records, is sometimes illegal without adequate protections \citep{GDPR2016a}. However, in many scenarios, privacy protection is motivated by users’ willingness to participate in training, based on the risks it entails for their business or their own privacy if their data is leaked \citep{nissenbaum2004privacy, ibm_cost_data_breach_2025}. Users who generate training data can be seen as agents who decide whether to participate depending on the risk that their data is leaked. A natural feedback loop to consider is thus that, if a user has their data leaked, they may no longer be willing to participate in the future, harming future rounds of training. This is often claimed informally: privacy-preserving machine learning should optimize long-term utility by maintaining users’ trust, and thus their data \citep{dworkAlgorithmicFoundationsDifferential2013}. Surprisingly, this claim that private learning can provide optimal long-term utility has, to the best of our knowledge, never been formalized. We propose first results in this direction.

The gold standard for protecting data is differential privacy (DP)\citep{dwork:2006:calibrating}. DP formalizes the following promise: a single data point should not influence the output of the algorithm too much, ensuring that any attacker with access to the output cannot gain significant information about the data point by observing the output. This protection, however, can only be achieved through randomization, often implemented by injecting noise into the algorithm. The level of privacy protection is then summarized by the privacy budget $\eps$: the smaller the privacy budget, the stronger the privacy guarantee, and the larger the noise required, which decreases the accuracy of the model, creating a privacy-utility trade-off. The right value for $\eps$ is often debated in the community, based on ad hoc perceptions of data sensitivity and risk \citep{Ponomareva2023, Cyffers2025SettingI, Lee2011}. Here, we take an orthogonal perspective, considering the optimal privacy budget to be the one that maximizes long-term utility.

Taking into account how models change the data over time is the motivation of \emph{performative learning} \citep{perdomo_performative}, where the data distribution is parametrized by the same parameter as the model, and the goal is to optimize the overall system rather than find the best model for a fixed distribution, as usually done in supervised learning. Performative learning is an increasingly popular framework, motivated by examples where users may game the system by modifying their features to obtain more favorable outcomes, by settings involving generative models where data can be contaminated over time, or by scenarios where sampling can reinforce existing biases, raising long-term fairness concerns \citep{hardt2023performative}. Feedback loops due to privacy leakage have not been studied so far, but they constitute a natural dynamic to study, as it is well known that data leakage can cause trust issues and reduce participation, or even lead participants to adopt malicious behaviors \citep{ibm_cost_data_breach_2025}. Relatedly, the term \emph{performative privacy} was coined by legal scholars \citep{skinner-thompson2017performative} to describe the actions agents may take to exercise their right to privacy, by making unavailable information that could otherwise be available by default.

In this work, we design and study a simple model that allows us to quantify regimes in which a finite privacy budget maximizes utility in a repeated mean estimation problem. Rather than using the privacy budget directly as a proxy for the decision to participate, we use whether the data point has been leaked in the feedback loop. This makes our setting applicable to non-differentially private mechanisms as well and avoids focusing on a meta-interpretation of the privacy budget, instead focusing on actual algorithm outputs. More precisely, our contributions are the following:

\begin{itemize}[leftmargin=*,nosep]
\item We introduce \emph{performative privacy}, the long-term utility of repeated learning tasks in which privacy drives participation.
\item We study it in three settings. We warm up with binary private values, where agents withdraw if their value is used and new agents join at each step (\Cref{section:randomized_response}). We then move to $d$-dimensional mean estimation, with leakage detected via membership inference attacks (\Cref{section:membership_inference}), and finish with linear regression, where a noisy estimator is released at each step and leakage is detected in the same way (\Cref{sec:regression_perf_dp}).
\item We show that in all settings a \emph{finite privacy budget can be optimal}, proving that privacy can maximize long-term utility under participation feedback loops.
\end{itemize}


\section{Related work}

\paragraph{Privacy and participation} While a direct link between data leakage and future participation has not yet been formalized, differential privacy originates from \citet{warner1965randomized}, where randomization aims to reduce bias in data collection, and thus improve data quality. Inspirations from sociology, notably through the contextual integrity framework, have also been numerous \citep{nissenbaum2004privacy, Cyffers2025SettingI, IndividualSensitivityPreprocessing}. Considering data producers as privacy-aware agents is a popular framework, sometimes formalized as a data market \citep{10.1145/1993574.1993605,Pai2013, Fallah2022OptimalAD}, in particular for deciding participation in surveys \citep{10.1145/2229012.2229076, 10.1007/978-3-642-35311-6_28}. The choice of the privacy budget itself has been analyzed through economic and mechanism-design lenses \citep{hsu2014differential, nissim2012privacy, cummings2015accuracy, ligett2017accuracy}, possibly personalizing it across users \citep{Jorgensen2015}. However, privacy awareness is often translated into a supposedly known optimal privacy budget, computed from expected utility and privacy preferences \citep{Cummings2022OptimalDA, Syomantak}, rather than modeled as a dynamical system. Participation has also been studied as a strategic decision in collaborative and federated learning \citep{donahue2021model, blum2021one, karimireddy2022mechanisms}, without the link to privacy leakage that we consider. On the empirical side, data breaches measurably affect customer behavior and firm value \citep{ibm_cost_data_breach_2025}, and the way privacy guarantees are communicated changes users' willingness to share their data \citep{xiong2020towards, cummings2021need}, supporting the feedback loop we model.

\paragraph{Performativity in Trustworthy Machine Learning.}
Performative learning \citep{perdomo_performative} is motivated by the social impact of machine learning deployments \citep{hardt2023performative, hardt2022performative}, with dedicated optimization algorithms \citep{mendler2020stochastic} and regret analyses \citep{jagadeesan2022regret, zhu2025lookahead}, and with links to robustness \citep{cyffers2024optimal} and fairness \citep{Somerstep2024}. Closest to our setting are retention dynamics, where user groups shrink after receiving poor utility \citep{hashimoto2018fairness, zhang2019group}: we study the privacy analogue, where users leave after data leakage. As the population size depends on the whole history of deployments, our model is an instance of stateful performativity \citep{brown2022performative}. Self-reinforcing feedback loops have also been documented in predictive policing \citep{ensign2018runaway}, dataset bias amplification \citep{taori2023data}, and generative model collapse \citep{shumailov2024ai}. Interestingly, participants can also contribute to steering the model toward desirable properties, as studied in collective action \citep{pmlr-v202-hardt23a}. To the best of our knowledge, there is no theoretical work, nor simulated feedback-loop model, tackling privacy concerns from this perspective so far.

\paragraph{Membership Inference Attacks.}
Membership inference attacks (MIAs) were introduced in the context of genomic studies by \citet{Homer2008}, who demonstrated that aggregate statistics can reveal an individual's participation in a dataset. MIAs have played an important role in motivating DP and establishing lower bounds for private data analysis \citep{BUV2014,Nasr2021}. In machine learning, numerous MIAs have been proposed under different threat models and access assumptions \citep{shokri2017membership,pmlr-v97-sablayrolles19a,belletcausalmia, Carlini2021MembershipIA, Yeom2017PrivacyRI, ye2022enhanced}. These attacks are also widely used to empirically audit machine-learning systems and evaluate privacy defenses \citep{jagielski,aerni2024,jayaraman2019evaluating,thudi2022bounding}. MIA is naturally cast in the hypothesis-testing formulation of DP \citep{wasserman2010statistical, kairouz2015composition}, culminating in Gaussian DP \citep{dong2022gaussian}, whose privacy-loss random variable is the leakage statistic we analyze in \Cref{sec:miaGaussian}. Beyond measuring average privacy leakage, recent work has shown that privacy risk can vary substantially across individuals \citep{leemann2024is,azize2025some}, motivating per-instance guarantees and individual privacy accounting \citep{wang2019perinstance, feldman2021individual, yu2023individual, koskela2023individual}, and that worst-case privacy guarantees may be overly pessimistic in finite-sample regimes \citep{haghifam2025samplecomplexitymembershipinference}. Most closely related to our setting, recent work has studied membership inference under repeated interactions \citep{debsequentialMIA}.

\section{Estimation of Bernoulli parameters}
\label{section:randomized_response}

In this section, we study the estimation of a Bernoulli parameter when the ground truth changes at each time step, from i.i.d. samples locally privatized using randomized response (RR). An answer is considered leaked whenever the true answer is preserved by RR. More precisely, over $T$ time steps, at each step $t$ the learner computes an estimate $\hat{\mu}_t$ of the true parameter $\mu_t$ of an underlying Bernoulli distribution, from privatized responses $(Y_{i,t})_{i=1}^{N_t}$ of the users' samples $(X_{i,t})_{i=1}^{N_t} \in \{0, 1\}^{N_t}$, drawn i.i.d.\ from $\mathcal{B}(\mu_t)$. The learner aims to minimize the mean squared error at the final horizon, $\mathcal{E} = \mathbb{E}\big[(\mu_T - \hat{\mu}_T)^2\big]$, or alternatively, as developed in the next section, the error over the whole trajectory, $\mathcal{R}_T = \sum_{t=1}^T \mathbb{E}(\mu_t - \hat{\mu}_t)^2$.

\paragraph{Randomized response}
We consider local differential privacy, where anyone can access $(Y_{i,t})_{i=1}^{N_t}$, and privacy is ensured at the user level via randomized response \citep{dworkAlgorithmicFoundationsDifferential2013,ioannidis2024randomized,wang2016using}. We assume that all users have the same privacy budget $\eps$, constant over time, and define
\begin{align}\label{eq:eps_ldp_beta}
Y_{i,t} \triangleq \begin{cases} X_{i,t}, & \text{with prob. } \frac{1}{2} + \beta \\ 1 - X_{i,t}, & \text{with prob. } \frac{1}{2} - \beta \end{cases}
\qquad \text{with} \quad \beta = \frac{1}{2} \tanh\left(\frac{\eps}{2}\right) \in \left[0, \tfrac{1}{2}\right),
\end{align}
which guarantees $\eps$-local differential privacy ($\eps$-LDP). The learner's best unbiased estimator is then
\begin{align}\label{eq:unbiased_estimator}
\hat \mu_t \triangleq \frac{\hat\nu_t-\frac{1}{2}+\beta}{2\beta} \quad \text{with} \quad \hat{\nu}_t \triangleq \frac{1}{N_t}\sum_{i=1}^{N_t} Y_{i,t}.
\end{align}
Both claims are proved in Appendix~\ref{app:proofs}.

\paragraph{Population evolution}
The population changes as participating users leave in reaction to data leakage, defined as a response remaining unperturbed ($X_{i,t-1} = Y_{i,t-1}$), and as new users join. A leaked participant leaves with a fixed probability $q$, independent of time and of the user (the model has no memory of past leakages, which minimizes the number of parameters), and each participant whose data is not leaked recruits a new participant with probability $r$.\footnote{Growth proportional to the total number of users, regardless of leakage, is studied in Appendix~\ref{app:total_model}.} The expected one-step change of the population is then summarized by a single growth factor (Lemma~\ref{lemm:unleaked_growth} in Appendix~\ref{app:proofs}):
\begin{equation}\label{eq:growth_factor_rr}
    C(\beta) \triangleq \frac{\mathbb{E}[N_{t+1}|N_t]}{N_t} = 1-q\left(\frac{1}{2}+\beta\right) + r\left(\frac{1}{2}-\beta\right) = 1 + \frac{r-q}{2} - \beta(q+r).
\end{equation}
The expected population is stable or growing if $C(\beta) \ge 1$ and declines otherwise, which bounds how little private the mechanism can be:

\begin{theorem}[Threshold for Nonnegative Expected Growth]\label{theo:min_beta_threshold}
For $r>q$, nonnegative expected growth, $C(\beta)\ge1$, holds exactly when $\beta\le\beta_c$, or equivalently $\eps\le\eps_c\triangleq\ln(r/q)$, where:
\begin{equation}
    \beta_c \triangleq \frac{r-q}{2(q+r)}
\end{equation}
\end{theorem}

We prove this theorem in Appendix~\ref{app:proofs}.
From this formula, we can see two regimes: if $q > r$, the population always collapses, regardless of the privacy protection. However, when $r > q$, the population can grow or shrink depending on the choice of $\beta$, and $C(\beta) \ge 1$ gives a hard operational constraint on $\beta$. Simulations confirm this behavior (\Cref{fig:three_subfigures}): the population collapses for $q>r$, whereas for $r>q$ it reaches a stable equilibrium or grows to the maximum size, and a finite privacy budget optimally balances privacy noise against user retention. Moreover, where a finite optimum exists, the optimal $\eps$ decreases monotonically with the leaving rate $q$: stronger privacy is required when users are more sensitive to data leakage. Beyond expected growth, $(N_t)$ is exactly a Galton--Watson branching process, whose extinction probability $e^{-N_0(\eps_c-\eps)_+}$ and fluctuations around $N_0C(\beta)^t$ are given in \Cref{theo:gw_rr} (Appendix~\ref{app:gw_rr}).

\begin{figure}[tb]
    \centering

    \begin{subfigure}[b]{0.34\textwidth}
        \centering
        \includegraphics[width=\textwidth]{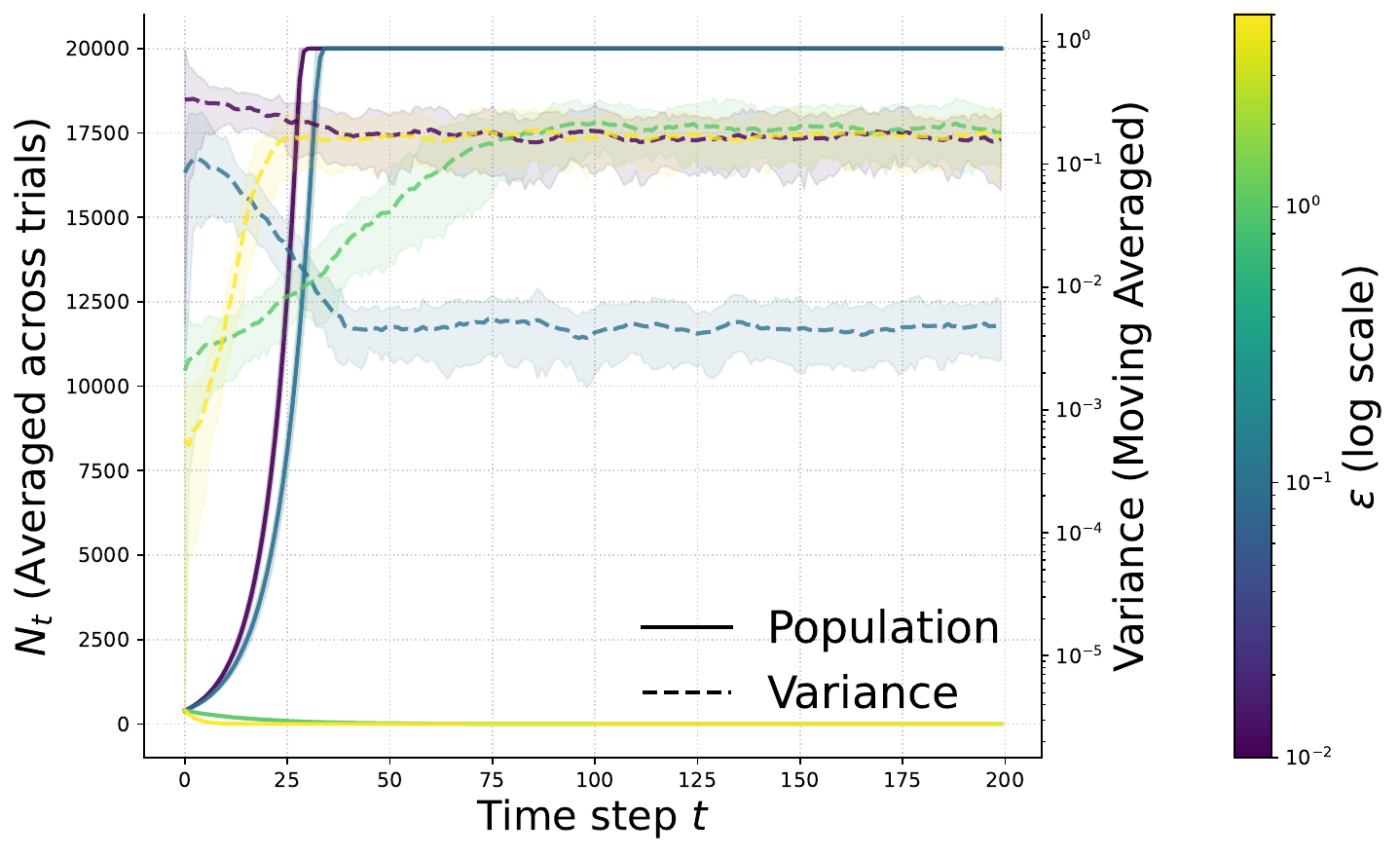}
        \caption{$q=0.3, r=0.6$}
        \label{fig:compressed_low_attrition_rr}
    \end{subfigure}
    \hfill
    \begin{subfigure}[b]{0.34\textwidth}
        \centering
        \includegraphics[width=\textwidth]{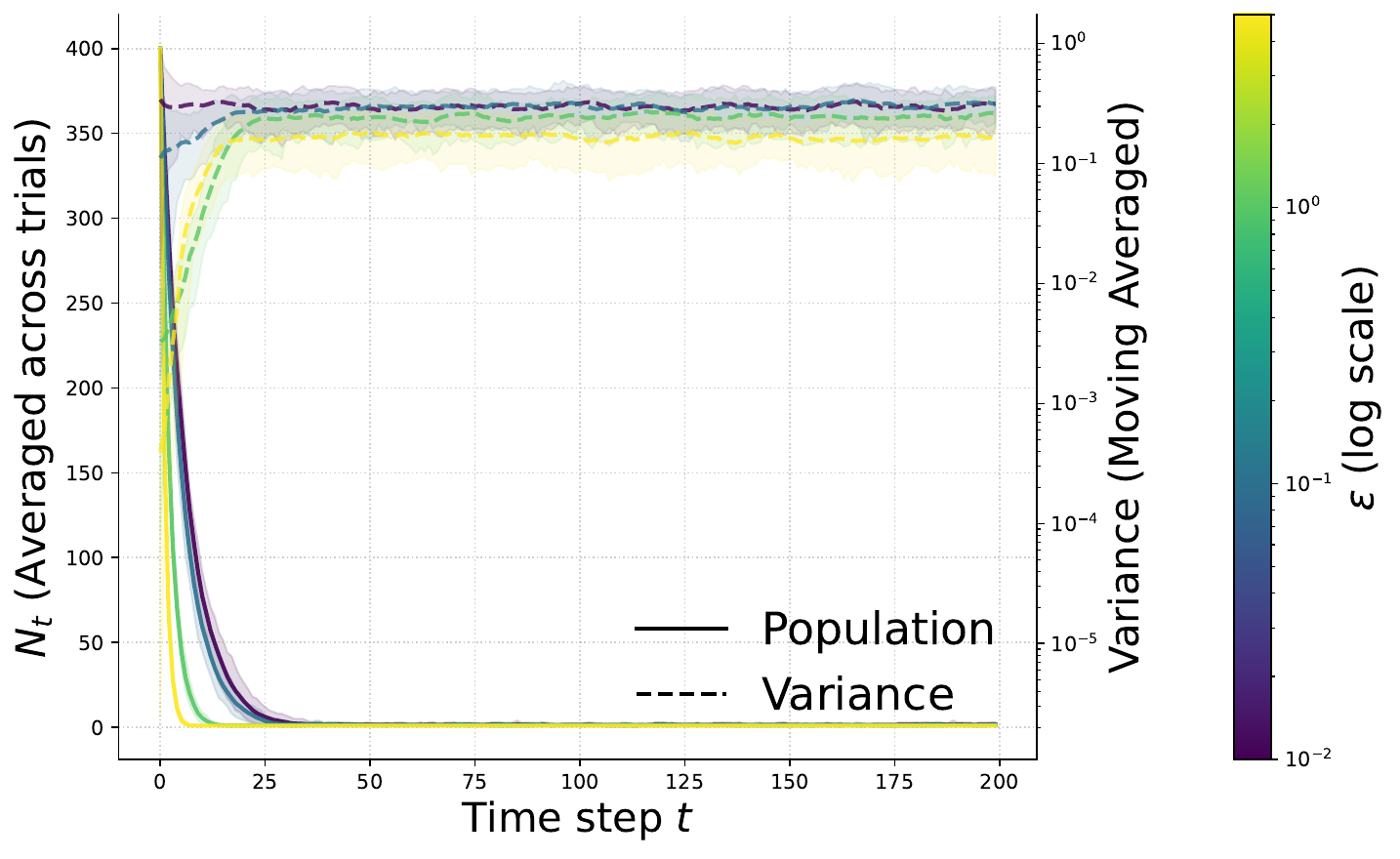}
        \caption{$q=0.6, r=0.3$}
        \label{fig:compressed_high_attrition_rr}
    \end{subfigure}
    \hfill
    \begin{subfigure}[b]{0.26\textwidth}
        \centering
        \includegraphics[width=\textwidth]{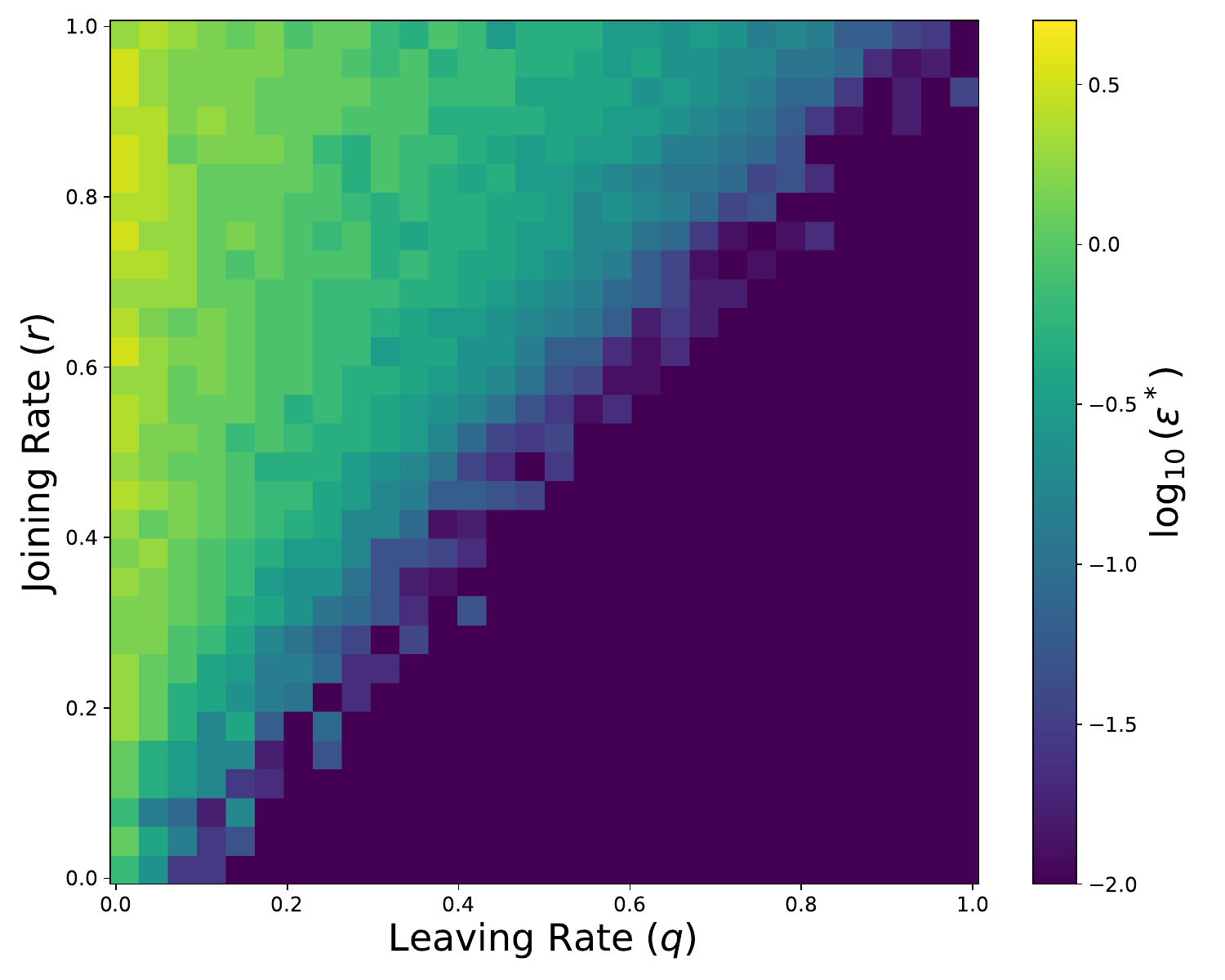}
        \caption{Heatmap}
        \label{fig:heatmap_rr}
    \end{subfigure}

    \caption{Randomized response simulations. (a)--(b): population size $N_t$ (clipped to $[1, 20000]$) and moving-averaged variance of $\hat\mu_t$ (window $15$) for $10$ values of $\eps$ (color intensity), in the low- ($q<r$) and high-attrition ($q>r$) regimes, averaged over $30$ trials. (c): utility-maximizing $\eps$ as a function of $(q,r)$, where utility is the squared error at the final horizon.}
    \label{fig:three_subfigures}
    
\end{figure}


\section{Mean estimation with membership inference detection}\label{section:membership_inference}

\subsection{Setting}\label{sec:mia}

The previous section examined performative population dynamics under randomized response, where data is binary and leakage is defined as the exact release of the true bit. In that setting, utility and privacy risk were strictly opposed, as leakage was defined as sharing the single bit of information. Fortunately, richer data decouples a record's informational utility from its membership detection, as we show in this section by estimating the mean of continuous vectors in $\mathbb{R}^d$. A record located exactly at the population mean is statistically difficult to detect, but valuable for the estimation, whereas an outlier is easy to detect and hurts utility. Finally, this high-dimensional continuous setting is closer to realistic machine learning deployments, allowing us to formalize user vulnerability through standard Membership Inference Attacks (MIAs) based on an individual's geometric footprint.

More precisely, at each round $t$, the goal is to estimate the mean $\mu$ of a distribution $\mathcal{D}$ over $\mathbb{R}^d$ with covariance $\Sigma$ from a set of $N_t$ samples $Z_t = \{z_1, \dots, z_{N_t}\}$ drawn independently from $\mathcal{D}$.
To ensure bounded sensitivity, each record is clipped to a maximum $\ell_2$-norm $R$: $\bar{z}_i = z_i \min\left(1, \frac{R}{\|z_i\|_2}\right)$ denotes the clipped vector $z_i$. In our theoretical utility results, we assume that $R$ is large enough to neglect the clipping bias.
The output $o_t$ is an empirical mean protected by the Gaussian Mechanism:  
\begin{equation}
    o_t = \mathcal{M}_\gamma(Z_t) = \frac{1}{N_t} \sum_{i=1}^{N_t} \bar{z}_i + \eta_t, \quad \eta_t \sim \mathcal{N}\left(0, \frac{{\gamma^2}}{N_t} I_{d \times d}\right)
    \label{eq:mecgau}
\end{equation}
This corresponds to each user adding an independent noise $\mathcal{N}(0,\gamma^2 I_d)$ to their clipped record, whose $\ell_2$-sensitivity is $2R$. To satisfy a given zero-concentrated differential privacy (zCDP) budget $\rho$ \citep{bun2016concentrated} for each individual contribution, the platform must thus set $\gamma^2 = \frac{2R^2}{\rho}$. Utility is measured through the expected squared Mahalanobis distance between the private release and the true mean $\mu$ \citep{biswas2020coinpress,azize2025some,leemann2024is}, which, conditioned on $N_t$, is (Lemma~\ref{lemm:Mahalanobis_error} in Appendix~\ref{app:gaussian_mia_proofs})
\begin{equation}
    \mathbb{E} \left[ \|o_t - \mu\|_{\Sigma^{-1}}^2 \;\middle|\; N_t \right] = \frac{\gamma^2 \mathrm{Tr}(\Sigma^{-1}) + d}{N_t}.
\end{equation}
While \Cref{section:randomized_response} defined leakage via the discrete event $X_{i,t}=Y_{i,t}$, exposure in continuous aggregates ($o_t \in \mathbb{R}^d$) is quantified by a record's statistical influence on $o_t$. By the Neyman-Pearson lemma, the optimal test statistic for the hypothesis $z \in Z_t$ is the log-likelihood ratio:
\begin{equation}
    \Lambda(o_t|z) \triangleq \ln \left( \frac{\Pr[\mathcal{M}_\gamma(Z_t) = o_t \mid z \in Z_t]}{\Pr[\mathcal{M}_\gamma(Z_t \setminus \{z\}) = o_t]} \right) = \ln \left( \frac{P^{\mathcal{M}}_{\text{in}}(o_t)}{P^{\mathcal{M}}_{\text{out}}(o_t)}\right)
\end{equation}
where $\Pr[\cdot]$ denotes the density of the output, and $P^{\mathcal{M}}_{\text{in}}$ and $P^{\mathcal{M}}_{\text{out}}$ the factual and counterfactual output distributions. In the leave-one-out comparison, the round's noise covariance is held fixed; only the empirical mean is recomputed. The magnitude of $\Lambda(o_t|z)$ indicates how much skew $z$ exerts on the noisy sample mean $o_t$ given we observe the set $Z_t$. Practically, an individual cannot compute $\Lambda(o_t|z)$ without access to $Z_t$. However, it is the quantity approximated by external adversaries and auditing tools in MIAs \citep{shokri2017membership, Carlini2021MembershipIA}. For simplicity, we assume this approximation to be exact, which provides a worst-case upper bound against any adversary.
The privacy tolerance is captured via a constant threshold $\tau > 0$, such that a record is compromised if $\Lambda(o_t|z) > \tau$. From a hypothesis testing perspective, maintaining a fixed false positive rate would require scaling the threshold with the dimension $d$ and the population size $N_t$ over time, but we keep it constant for simplicity.

At round $t$, each participant $z \in Z_t$ updates their status for the subsequent round: if $\Lambda(o_t|z) > \tau$, the individual departs with probability $q$; otherwise ($\Lambda(o_t|z) \le \tau$), the individual remains and recruits one additional participant with probability $r$, as described in~\Cref{alg:memoryless_dp}.

        
        
        
            

\subsection{General Bound on the Leakage Probability}
\label{sec:genboundmia}

Computing the exact trajectory of the population size requires the knowledge of the marginal leakage probability $p^{(t)}_\tau = \Pr(\Lambda(o_{t}|z) > \tau)$. For arbitrary data distributions, deriving the exact probability density function of the log-likelihood ratio is intractable. Instead, we provide an upper bound on the leakage probability that involves the R\'enyi divergence, and thus makes the link with R\'enyi DP \citep{Mironov2017}, via Markov's inequality.

\begin{lemma}[Tail Bound on the Leakage Probability]\label{lemm:Renyi_bound}
With the above notations, for any threshold $\tau > 0$ and order $\alpha > 1$, the tail probability of leakage is bounded by:
\begin{align*}
\Pr_{o_t \sim P^{\mathcal{M}}_{\text{in}}}(\Lambda(o_t|z) > \tau) \le \exp \Big( -(\alpha - 1) \big( \tau - D_\alpha(P^{\mathcal{M}}_{\text{in}} \parallel P^{\mathcal{M}}_{\text{out}}) \big) \Big) 
\end{align*}
where $D_\alpha(P^{\mathcal{M}}_{\text{in}} \parallel P^{\mathcal{M}}_{\text{out}})$ is the Rényi divergence of order $\alpha$.
\end{lemma}

In Appendix~\ref{app:gaussiangeneric}, we show that, when applied to the Gaussian mechanism, this general bound is tight up to the $\Phi$-versus-$\exp$ gap.

\noindent
\begin{minipage}[t]{0.48\textwidth}
\begin{algorithm}[H]
\small\caption{Perf. MIA Mean Estimation }\label{alg:memoryless_dp}
\begin{algorithmic}[1]
\State \textbf{Input:} $N_0$, $\mathcal{D}$, bounds $R$, $\tau, q, r, T, \gamma$
\State \textbf{Set} $N_1 = N_0$
\For{$t = 1, \dots, T$}
    \If{$N_t \le 1$} 
        \State \textbf{break} 
    \EndIf
    \State Sample $Z_t \sim \mathcal{D}^{N_t}$, clip $\|z_i\|_2 \le R$
    \State Compute exact mean $\mu_t = \frac{1}{N_t}\sum_{z_i \in Z_t} z_i$
    \State Share private output $o_t \sim \mathcal{N}(\mu_t, \frac{\gamma^2}{N_t}I_d)$
    \vspace{1.5em}
    \State \textbf{Initialize} $N_{t+1} \leftarrow N_t$
    \For{\textbf{each} $z_i \in Z_t$}
        \State $\mu_t^{-i} \leftarrow \frac{N_t \mu_t - z_i}{N_t - 1}$
        \State $\Lambda_i \leftarrow \ln \left( \frac{\Pr[o_t \mid \mu_t]}{\Pr[o_t \mid \mu_t^{-i}]} \right)$
        \If{$\Lambda_i > \tau$}
            \State $N_{t+1} \leftarrow N_{t+1} - 1$ w.p. $q$
        \Else
            \State $N_{t+1} \leftarrow N_{t+1} + 1$ w.p. $r$
        \EndIf
    \EndFor
\EndFor
\State \textbf{Return} $(N_1, \dots, N_T)$
\end{algorithmic}
\end{algorithm}
\end{minipage}
\hfill
\begin{minipage}[t]{0.50\textwidth}
\begin{algorithm}[H]
\small\caption{Perf. MIA Ridge Regression}\label{alg:dp_ols}
\begin{algorithmic}[1]
\State \textbf{Input:} $N_0$, $\mathcal{D}$, bounds $R, Y$, reg $\lambda$, $\tau, q, r, T, \gamma$
\State \textbf{Set} $N_1 = N_0$
\For{$t = 1, \dots, T$}
    \If{$N_t \le d$} 
        \State \textbf{break} 
    \EndIf
    \State Sample $Z_t \sim \mathcal{D}^{N_t}$, clip $\|x_i\|_2 \le R, |y_i| \le Y$
    \State $A_t \leftarrow X_t^\top X_t + \lambda I_d$
    \State Compute exact ridge $\theta_t = A_t^{-1} X_t^\top y_t$
    \State Share private output $\tilde \theta_t \sim \mathcal{N}(\theta_t, \frac{\gamma^2}{N_t}I_d)$
    \State \textbf{Initialize} $N_{t+1} \leftarrow N_t$
    \For{\textbf{each} $(x_i, y_i) \in Z_t$}
        \State $\theta_t^{-i} \leftarrow \theta_t - \frac{y_i - x_i^\top \theta_t}{1 - h_i} A_t^{-1} x_i$ 
        \State $\Lambda_i \leftarrow \ln \left( \frac{\Pr[\tilde \theta_t \mid \theta_t]}{\Pr[\tilde\theta_t \mid \theta_t^{-i}]} \right)$
        \If{$\Lambda_i > \tau$}
            \State $N_{t+1} \leftarrow N_{t+1} - 1$ w.p. $q$
        \Else
            \State $N_{t+1} \leftarrow N_{t+1} + 1$ w.p. $r$
        \EndIf
    \EndFor
\EndFor
\State \textbf{Return} $(N_1, \dots, N_T)$
\end{algorithmic}
\end{algorithm}
\end{minipage}

\subsection{The Gaussian Mechanism Case}
\label{sec:miaGaussian}

We now exploit the fact that the noise $\eta_t \sim \mathcal{N}(0, \frac{\gamma^2}{N_t} I_{d \times d})$ is Gaussian, ignoring clipping for simplicity. Let $\mu_{z,t} \triangleq \frac{1}{N_t}\sum_{i=1}^{N_t} z_{i}$ denote the sample mean at round $t$. By direct computation (Lemma~\ref{lemm:log_likelihood_distribution} in Appendix~\ref{app:gaussian_mia_proofs}), the log-likelihood ratio $\Lambda(o_t|z)$ conditioned on $Z_t$ is distributed as
\begin{equation}\label{eq:llr_gaussian}
    \Lambda(o_t|z) \mid Z_t
    \sim
    \mathcal{N}\left(
    \frac{N_t\|z-\mu_{z,t}\|^2}{2(N_t-1)^2\gamma^2},
    \;
    \frac{N_t\|z-\mu_{z,t}\|^2}{(N_t-1)^2\gamma^2}
    \right).
\end{equation}
From this, we deduce the exact leakage probability:
\begin{lemma}\label{lemm:conditional_leakage_prob}
Given the aggregated empirical mean mechanism, the exact probability that a specific individual data point $z$ is exposed (i.e., its log-likelihood ratio breaches the threshold $\tau$) conditioned on $Z_t$ is given in closed form by:
\begin{equation*}
    \Pr(\Lambda(o_t|z) > \tau\mid Z_t)
    =
    1 - \Phi\left(
    \frac{\tau (N_t-1)\gamma}
    {\sqrt{N_t}\|z-\mu_{z,t}\|}
    -
    \frac{\sqrt{N_t}\|z-\mu_{z,t}\|}
    {2(N_t-1)\gamma}
    \right)
\end{equation*}
where $\Phi(\cdot)$ denotes the cumulative distribution function of the standard normal distribution.
\end{lemma}
This probability increases with the distance of $z$ to the empirical mean, so outliers are the most exposed, and decreases with $N_t$, reflecting privacy amplification by aggregation: a growing population becomes less exposed, creating a virtuous circle.

\begin{remark}[Connection to Gaussian DP]\label{rem:gdp}
With $\mu_{z',t} \triangleq \sqrt{N_t}\|z-\mu_{z,t}\|/(\gamma({N_t}-1)) \approx \|z-\mu_{z,t}\|/(\gamma\sqrt{N_t})$, \eqref{eq:llr_gaussian} reads $\Lambda(o_t|z) \mid Z_t \sim \mathcal{N}(\mu_{z',t}^2/2, \mu_{z',t}^2)$, which is the privacy loss of a $\mu_{z',t}$-Gaussian-DP pair \citep{dong2022gaussian}: for the record $z$, membership inference is exactly as hard as testing $\mathcal{N}(0,1)$ against $\mathcal{N}(\mu_{z',t},1)$. Lemma~\ref{lemm:conditional_leakage_prob} is then the power of the Neyman--Pearson test at likelihood-ratio threshold $e^\tau$ --- equivalently, of the Bayes-optimal attack with prior odds of membership $e^{-\tau}$. Thus, maintaining a fixed FPR $\nu$ corresponds to the record-dependent threshold $\tau = \mu_{z',t} \Phi^{-1}(1-\nu) - \mu_{z',t}^2/2$, as discussed in \Cref{sec:mia}, and since $\|\bar z\|_2,\|\mu_{\bar z,t}\|_2 \le R$, one recovers the worst-case $(2R\sqrt{N_t}/(\gamma(N_t-1)))$-Gaussian-DP guarantee.
\end{remark}

\subsection{The Gaussian Data Case}

We now further assume that the data is drawn from a Gaussian distribution $\mathcal{N}(\mu,\Sigma)$ which allows us to characterize precisely the dynamics.
\begin{lemma}[Expected Population Evolution]\label{lemm:gaussian_mia_poplin_evolution}
We define the \textit{marginal leakage probability}, $p^{(t)}_\tau$, as the expected fraction of the population that breaches the privacy threshold $\tau$ at round $t$, i.e., $p^{(t)}_{\tau} = \mathbb{E}[\mathds{1}(\Lambda(o_t|z) > \tau)]$. 
Letting $\bar{z}$ denote the record subjected to the clipping and $\mu_{\bar z,t}$ denote the sample mean of clipped vectors, the marginal leakage probability evaluates to:
\begin{align}\label{eq:p_tau_integral}
p^{(t)}_\tau
&=
\bigint_{\mathbb{R}^{dN_t}}
\left[
1-\Phi\!\left(
\frac{\tau (N_t-1)\gamma}
{s\sqrt{N_t}}
-
\frac{s\sqrt{N_t}}
{2(N_t-1)\gamma}
\right)
\right]
\prod_{j=1}^{N_t} f_{\mathcal{N}}(z_j)\,
d z_{1:N_t},
\end{align}
where $s\triangleq\|\bar z_1-\mu_{\bar z,t}\|_2$ and $d z_{1:N_t} \triangleq d z_1 \cdot\dots d {z_{N_t}}$. The expected population size at round $t+1$ in \Cref{alg:memoryless_dp} is then given by:
\begin{equation}
    \expectation[N_{t+1} \mid N_t] = N_t \Big( 1 + r - (q + r) p^{(t)}_\tau \Big)
\end{equation}
\end{lemma}

Lemma \ref{lemm:gaussian_mia_poplin_evolution} gives an exact characterization of the user population evolution. As in Lemma~\ref{lemm:unleaked_growth}, the expected population is the previous population multiplied by a growth factor. However, unlike the previous case, this factor is determined by the marginal leakage probability $p_\tau$, whose integral form captures the dependence on several parameters: the dimension $d$, the threshold $\tau$, the clipping bound $R$ and the noise scale $\gamma$, instead of relying only on the privacy parameter $\beta$.

To determine whether the expected user population grows or declines, we denote similarly as before $C_t(\gamma)$ the expected ratio of the population size from round $t$ to round $t+1$:
\begin{equation}
    C_t(\gamma) \triangleq \frac{\mathbb{E}[N_{t+1} \mid N_t]}{N_t} = 1 + r - (q + r)p^{(t)}_\tau(\gamma)
\end{equation}
The continuous growth factor $C_t(\gamma) = 1 + r - (q+r)p^{(t)}_\tau(\gamma)$ mirrors the discrete model of Lemma \ref{lemm:unleaked_growth} by substituting the leakage probability $\frac12+\beta$ of randomized response with the marginal leakage probability $p^{(t)}_\tau(\gamma)$. This substitution preserves the system's monotonicity: increasing the noise $\gamma$ strictly decreases $p^{(t)}_\tau(\gamma)$, improving population stability. However, because $p^{(t)}_\tau(\gamma)$ requires evaluating a non-linear integral over a multivariate data distribution, the boundary condition $C_t(\gamma) \ge 1$ cannot be inverted analytically in general, but it can easily be computed numerically. Moreover, unlike in \Cref{section:randomized_response}, $C_t$ is now time-dependent, introducing an additional layer of performativity.

\begin{theorem}[Sufficient Noise Threshold for Expected Growth]\label{theo:noise_threshold_mia}
 The condition $C_t(\gamma) \ge 1$ is guaranteed as soon as the mechanism's noise scale satisfies the explicit lower bound $\gamma \ge \gamma_c$. For $q,r>0$, $N_t>1$ and clipping bound $R$, a sufficient threshold $\gamma_c$ is defined as:
\begin{equation}
    \gamma_c \triangleq
    \frac{2R\sqrt{N_t}}{N_t-1}
    \left(
    \frac{\Phi_{q,r}^{-1}
    +\sqrt{(\Phi_{q,r}^{-1})^2+2\tau}}
    {2\tau}
    \right)
\end{equation}
where $\Phi_{q,r}^{-1} \triangleq \Phi^{-1}\left( \frac{q}{q+r} \right)$ is the probit function (the inverse of the standard normal cumulative distribution function) evaluated at the performative survival threshold.
\end{theorem}

Here $C_t(\gamma)<1$ indicates expected population decline, which can reinforce leakage and further decline. Conversely, as the marginal leakage probability decreases with $N_t$ (proof of Lemma~\ref{lemm:C_gamma_increasing}), small populations leak more and shrink while large ones leak less and grow. This phenomenon is known as the Allee effect in population ecology \citep{allee1931animal,courchamp1999inverse}. For $R=\Theta(\sqrt d)$, the sufficient threshold scales as $\sqrt{d/N_t}$, separating the data geometry from the behavioral parameters.

In Appendix, theorem~\ref{theo:bounds_on_error} gives expected-error bounds of the same order, proportional to $(\gamma^2\mathrm{Tr}(\Sigma^{-1})+d)/(C_{t-1}(\gamma)N_{t-1})$ for large $N_{t-1}$ and $q<1$. The unconditional error $\mathcal E_t(\gamma)\triangleq\mathbb E[\|o_t-\mu\|^2_{\Sigma^{-1}}]$ depends on $\mathbb E[1/N_t]$, which we approximate by the inverse mean-field population. Empirically accurate (\Cref{fig:gamma_comparison}), this approximation is further justifiable globally via sequential chaining on Lemma \ref{lemm:single_step_concentration} (Appendix \ref{app:gaussian_mia_proofs}). A fixed noise level sufficient for growth throughout the trajectory can be calibrated at its minimum population size, since $\gamma_c\propto1/\sqrt{N_{\min}}$.

\begin{remark}[Finite-Horizon Optimization]
In the mean-field dynamics, low noise can initiate a feedback loop of leakage, population decline, and increasing estimation error. Conversely, sufficiently strong privacy can sustain population growth, eventually offsetting its noise penalty (Corollaries~\ref{cor:zero_noise} and~\ref{cor:infinite_horizon}). At a finite horizon, however, excessive noise remains costly, motivating the trade-off studied below.
\end{remark}

\begin{figure}[t]
    \centering


    \begin{subfigure}[t]{0.32\textwidth}
        \centering
        \includegraphics[width=\linewidth]{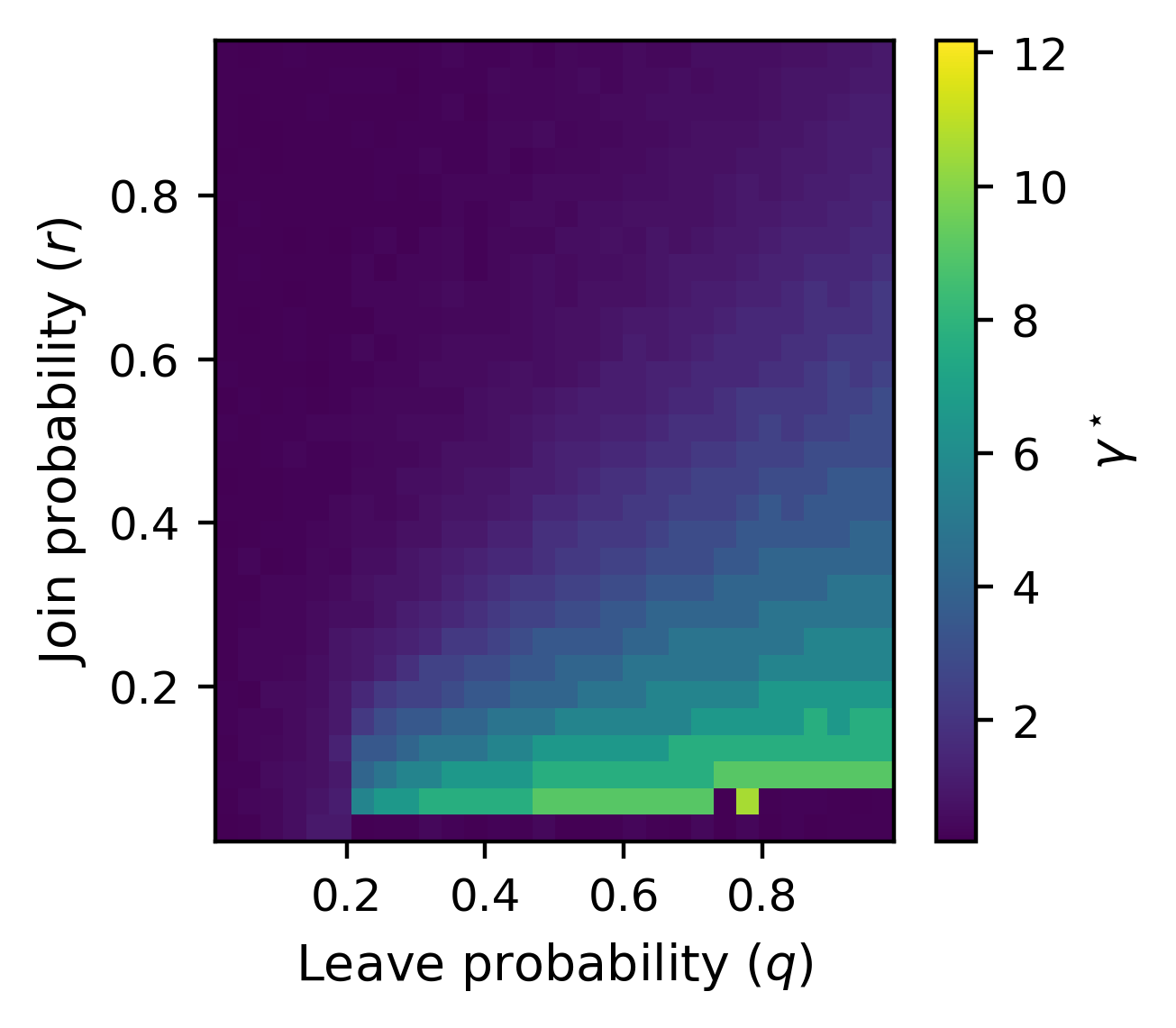}
        \caption{Empirical heatmap}
    \end{subfigure}
    \hfill
    \begin{subfigure}[t]{0.32\textwidth}
        \centering
        \includegraphics[width=\linewidth]{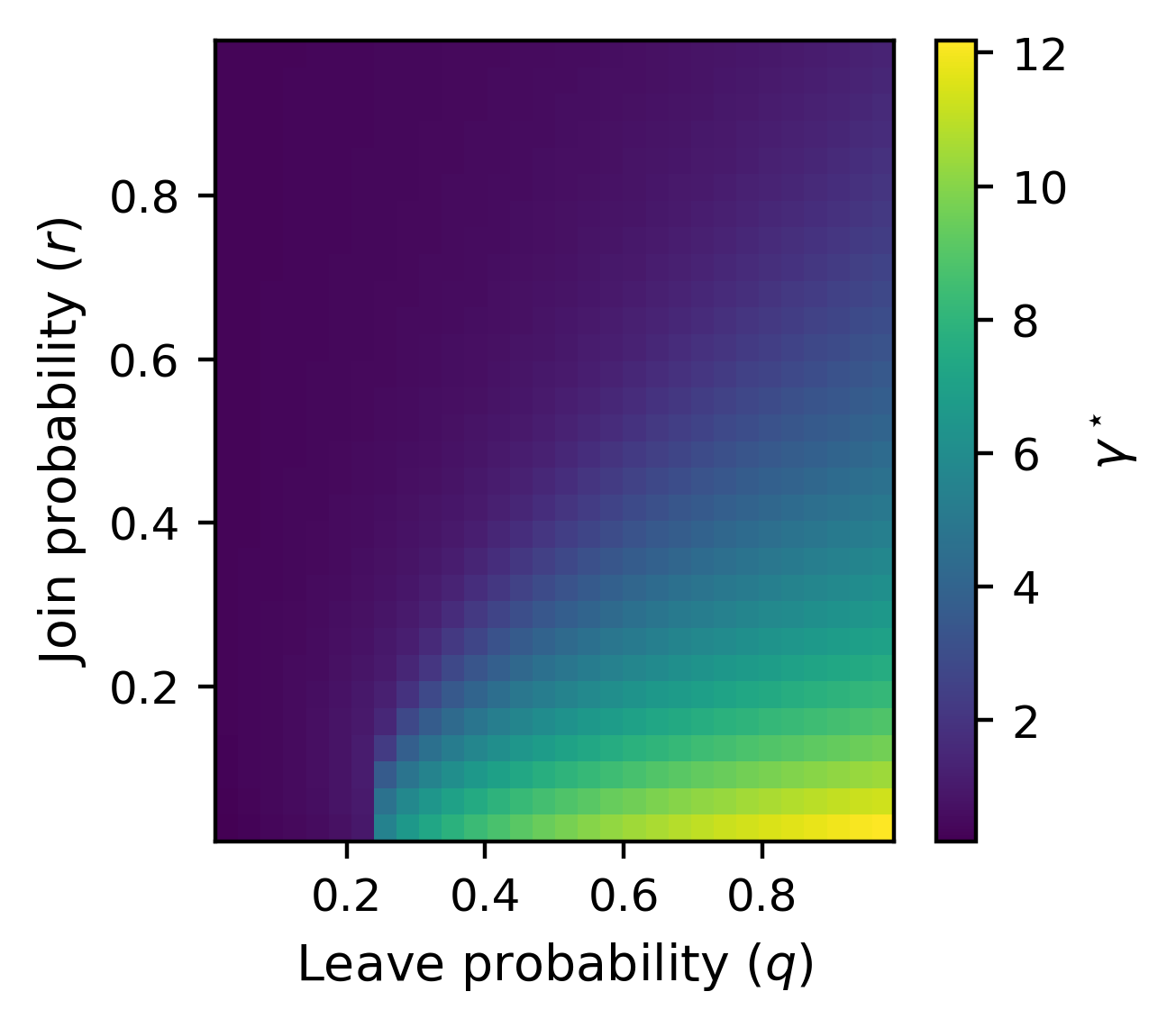}
        \caption{Theoretical heatmap}
    \end{subfigure}
    \hfill
    \begin{subfigure}[t]{0.32\textwidth}
        \centering
        \includegraphics[width=\linewidth]{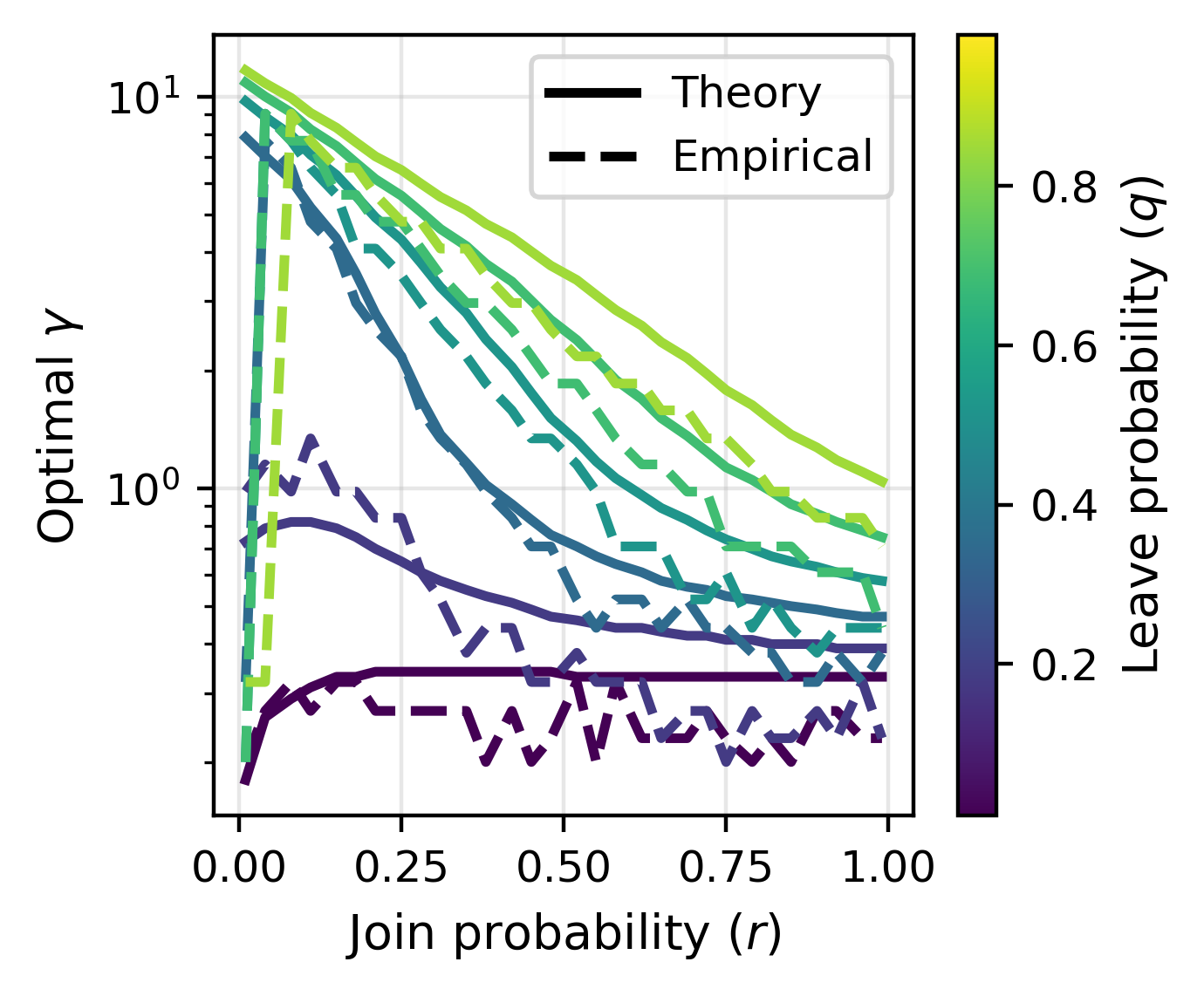}
        \caption{$\gamma^*$ line plot comparison}
    \end{subfigure}
    \caption{
    Comparison of the optimal noise scale $\gamma^*$ obtained from theory and simulations in mean estimation for $\tau = 0.02$. (c) compares the theoretical and empirical $\gamma^*$ as a function of the join probability. Other values of $\tau$ are reported in \Cref{fig:gamma_comparison_more}.
    }
    \label{fig:gamma_comparison}
\end{figure}

\paragraph{Optimization of the Global Noise Scale ($\gamma^*$).}
In the growth regime $C_1(\gamma)>1$, the cumulative error $\mathcal{E}_{\mathrm{total}}(\gamma) \triangleq \sum_{t=1}^T \mathcal{E}_t(\gamma)$ is upper bounded, under the mean-field approximation, by (Lemma~\ref{lemm:cumulative_error} in Appendix~\ref{app:remarkgaussian})
\begin{equation}\label{eq:E_hat}
    \hat{\mathcal{E}}_{\mathrm{total}}(\gamma) \triangleq
    \frac{C_1(\gamma)(\gamma^2 \mathrm{Tr}(\Sigma^{-1}) + d)}{N_0 (C_1(\gamma) - 1)} \left( 1 - \frac{1}{[C_1(\gamma)]^T} \right).
\end{equation}
This surrogate captures the noise--growth trade-off and depends on the initial growth factor $C_1(\gamma)$, avoiding tracking the sequence $(C_t(\gamma))_t$. The closed-form objective \eqref{eq:E_hat} makes explicit the trade-off between providing better utility at a given round and a larger population at the next round. Increasing $\gamma$ penalizes the numerator via the quadratic variance term ($\gamma^2$), but simultaneously reduces the error by increasing the growth factor $C(\gamma)$ in the denominator.
We cannot compute a closed form of $\gamma^* = \arg\min_{\gamma > \gamma_c} \hat{\mathcal{E}}_{\mathrm{total}}(\gamma)$. Indeed, the marginal leakage $p^{(1)}_\tau$ nested within $C_1(\gamma)$ is governed by the standard normal cumulative distribution function $\Phi(\cdot)$, so setting the derivative of this objective to zero does not yield a closed-form algebraic solution. However, $\hat{\mathcal{E}}_{\mathrm{total}}(\gamma)$ is a smooth, one-dimensional objective function and $\gamma^*$ can easily be computed with numerical solvers.

We thus compute the theoretical $\gamma^*$ heatmap over the population's recruitment rate $r$ and leaving rate $q$. As reported in \Cref{fig:gamma_comparison} (and \Cref{fig:gamma_comparison_more} in Appendix~\ref{app:more_experiments_mia}), the theoretical $\gamma^*$ closely aligns with the optimal empirical parameters discovered via grid search.

\section{Linear Regression With Membership Inference Detection} \label{sec:regression_perf_dp}

Our framework extends beyond mean estimation to supervised learning, which we illustrate with differentially private ridge regression. At round $t$, each of the $N_t$ records contains a feature vector $x_i \in \mathbb{R}^d$, $x_i \sim \mathcal{N}(0,\Sigma)$, and a label $y_i \mid x_i \sim \mathcal{N}(x_i^\top \theta^*_t,\sigma^2)$, where $\theta_t^*$ is the true parameter at round $t$. Features are clipped to $\|x_i\|_2 \le R$ and labels to $|y_i| \le Y$. The mechanism releases the output-perturbed ridge estimator
\begin{equation}\label{eq:ridge_mech}
    \widetilde{\theta}_t = \theta_t + \eta_t, \quad \theta_t = A_t^{-1} X_t^\top y_t, \quad A_t = X_t^\top X_t + \lambda I_d, \quad \eta_t \sim \mathcal{N}\Big(0, \frac{\gamma^2}{N_t}I_d\Big),
\end{equation}
where $X_t$ and $y_t$ stack the features and labels, and $\lambda$ is the regularization parameter. Leakage is detected as in \Cref{section:membership_inference}, holding both the noise covariance and penalty $\lambda$ fixed in the leave-one-out comparison. By the Sherman--Morrison formula, removing record $i$ shifts the estimator by $\Delta\theta_i = \frac{e_i}{1-h_i} A_t^{-1} x_i$, with residual $e_i = y_i - x_i^\top \theta_t$ and leverage $h_i = x_i^\top A_t^{-1} x_i$. Hence $\Lambda(\widetilde{\theta}_t \mid x_i, y_i) \mid Z_t \sim \mathcal{N}(a_i, 2a_i)$ with $a_i = N_t\|\Delta\theta_i\|^2/(2\gamma^2)$: this is exactly \eqref{eq:llr_gaussian}, with $\Delta\theta_i$ in place of the mean shift $(z-\mu_{z,t})/(N_t-1)$. The leakage probability, the growth factor $C_t(\gamma)$ and the population dynamics thus keep the same form (Appendix~\ref{app:reg_theory}), and bounding $\|\Delta\theta_i\|$ yields the critical noise threshold. Algorithm \ref{alg:dp_ols} presents the procedure underlying this setting.

\begin{theorem}\label{theo:noise_threshold_mia_ols}
For $q,r>0$ and $\lambda>0$, the condition $C_t(\gamma) \ge 1$ is guaranteed as soon as the noise scale of the ridge mechanism \eqref{eq:ridge_mech} satisfies $\gamma \ge \gamma_c$, where
\begin{align*}
    \gamma_c \triangleq
   \sqrt{N_{t}}\, \Delta_{\max}
    \left(
    \frac{\Phi_{q,r}^{-1}
    +\sqrt{(\Phi_{q,r}^{-1})^2+2\tau}}
    {2\tau}
    \right),
    \qquad
    \Delta_{\max} \triangleq \frac{RY}{\lambda}\left(1+R\sqrt{\frac{N_t}{\lambda}}\right),
\end{align*}
$\Phi_{q,r}^{-1} \triangleq \Phi^{-1}\left( \frac{q}{q+r} \right)$ as in \Cref{theo:noise_threshold_mia}, and $\Delta_{\max}$ bounds the shift $\|\Delta\theta_i\|_2$ induced by any single record.
\end{theorem}

\begin{remark}[Influence of a record compared to mean estimation]\label{rem:sensitivity_ols}
In mean estimation, the influence of a record on the release is the mean shift $(z-\mu_{z,t})/(N_t-1)$, so that outliers are the most exposed. In ridge regression, it becomes the leave-one-out shift $\Delta\theta_i = \frac{e_i}{1-h_i}A_t^{-1}x_i$, which combines the residual $e_i$ and the leverage $h_i$: the most exposed records are those that are both atypical in feature space and poorly fitted by the model. Since the residual can grow with $\|x_i\|$, the worst-case bound $\Delta_{\max}$ scales with $R^2$ instead of $R$, but this worst case is conservative: for well-specified data, residuals are of order $\sigma$ and $A_t \approx N_t\Sigma + \lambda I$, so typical shifts remain of order $\sqrt{d}/N_t$, as in mean estimation.
\end{remark}

Analogous to \Cref{section:membership_inference}, under the mean-field approximation with nondecreasing growth factors, the cumulative expected excess risk is upper bounded by $\hat{\mathcal E}_{\mathrm{total}}(\gamma)$ (Lemma~\ref{lemm:cumulative_excess_loss_ols}). This bound has the same form as \eqref{eq:E_hat}, with $\gamma^2\mathrm{Tr}(\Sigma^{-1})+d$ replaced by $\gamma^2\mathrm{Tr}(\Sigma)+\nu$, where $\nu=O(d)$ accounts for non-private estimation error, including regularization. We again select $\gamma^*$ by minimizing this surrogate. \Cref{fig:ols_comparison} compares the $\gamma^*$ minimizing this objective in the growth regime ($\gamma > \gamma_c$) with the one found by simulation:  their agreement supports applying the mean-estimation surrogate to the performative dynamics of parametric learning. As with mean estimation, providing privacy guarantees to users maximizes the long-term utility in repeated ridge regression, and the optimal level of protection increases with the probability $q$ that a user leaves.

 \begin{figure}[t]
    \centering
    \begin{subfigure}[t]{0.32\textwidth}
        \centering
        \includegraphics[width=\linewidth]{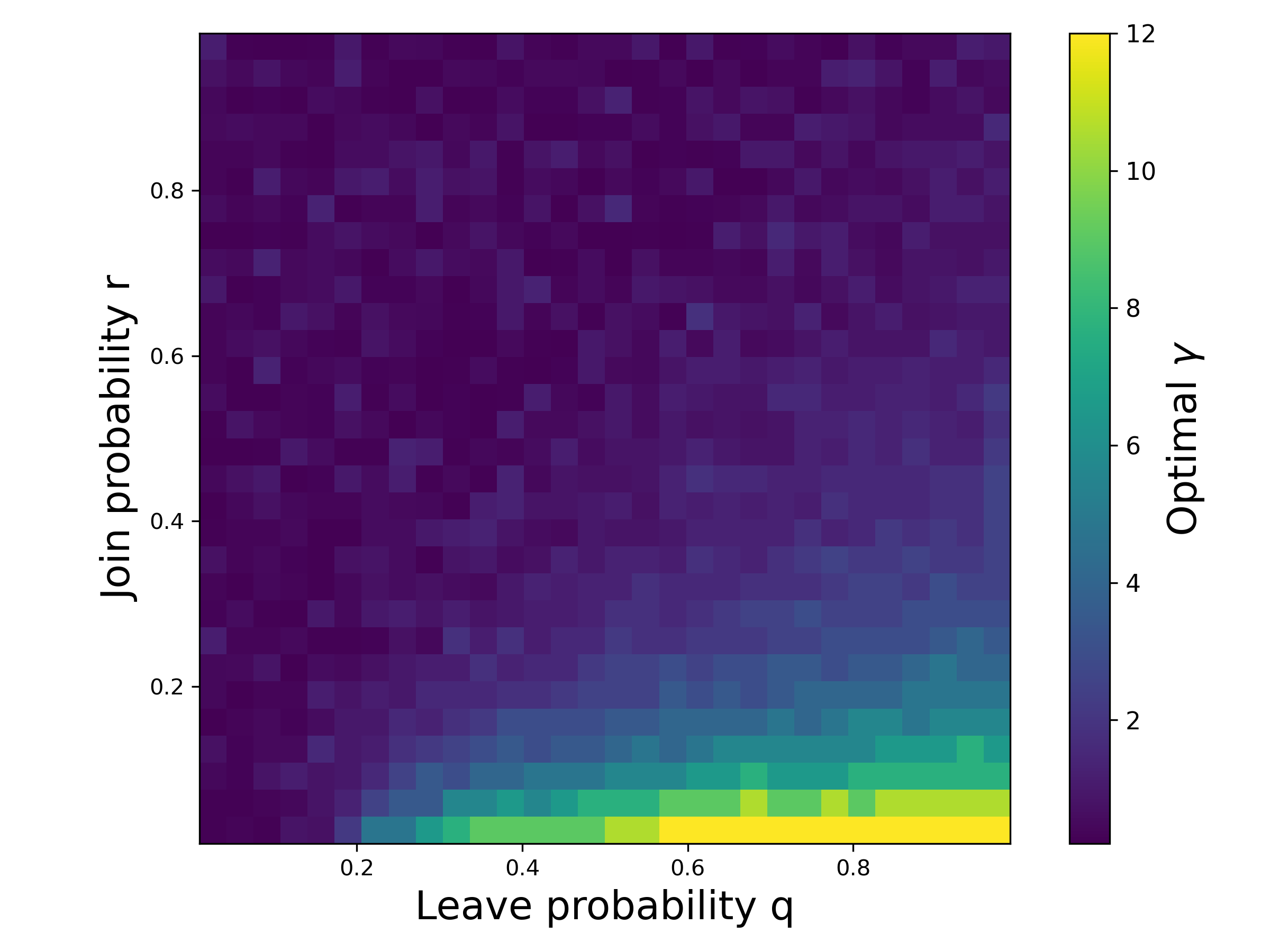}
        \caption{Empirical heatmap}
    \end{subfigure}
    \hfill
    \begin{subfigure}[t]{0.32\textwidth}
        \centering
        \includegraphics[width=\linewidth]{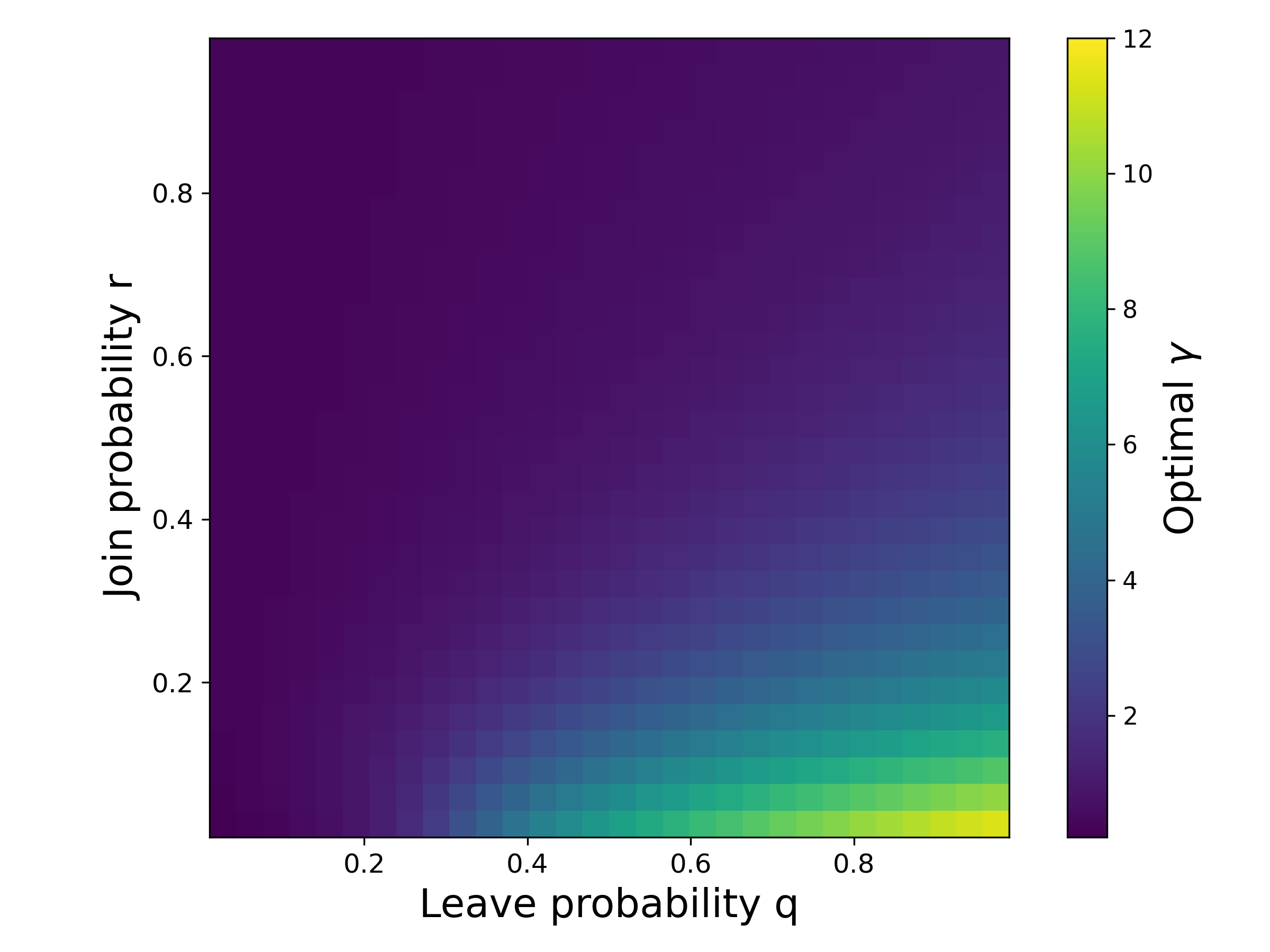}
        \caption{Theoretical heatmap}
    \end{subfigure}
    \hfill
    \begin{subfigure}[t]{0.32\textwidth}
        \centering
        \includegraphics[width=\linewidth]{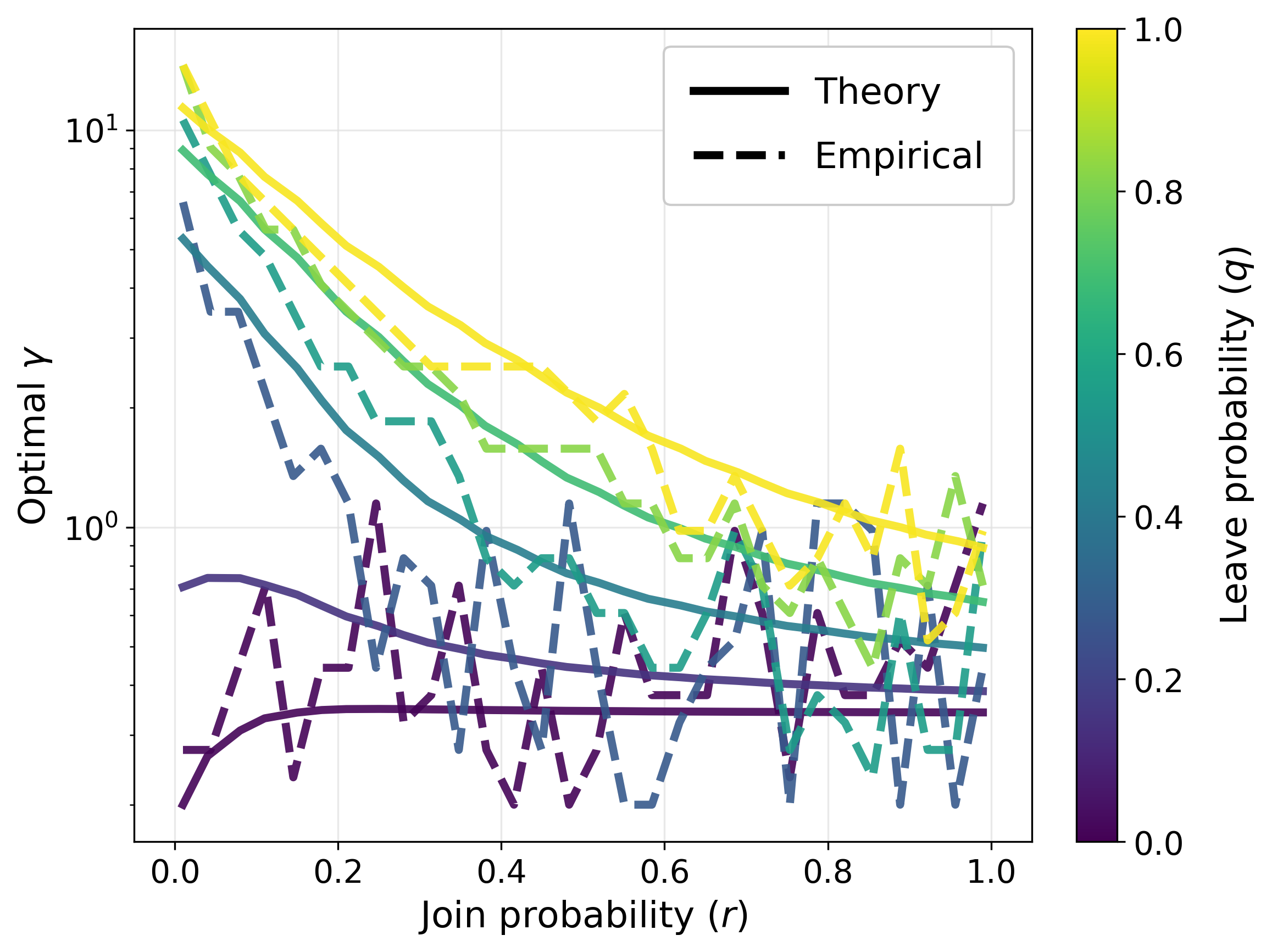}
        \caption{$\gamma^*$ line plot comparison}
    \end{subfigure}
    \caption{
    Comparison of the optimal noise scale $\gamma^*$ obtained from theory and simulations in linear regression.
    }
    \label{fig:ols_comparison}
\end{figure}

\section{Discussion and future work}

We showed that private estimation can outperform non-private estimation in the long term when participants drop out after their data is leaked. These results demonstrate that utility maximization can align with strong privacy protection when there is a sufficiently strong feedback loop. This paper introduces performative privacy as this study of the long-term interplay between privacy and utility under repeated model deployment, at the intersection of privacy-preserving machine learning and performative learning. 

We focused on mean estimation and linear regression which enable precise theoretical characterization of the dynamics, but we plan to extend our results to more complex learning tasks. The feedback loop is also kept as simple as possible, but it could be made more realistic using epidemic models and more complex privacy preferences and composition over time, and by studying different definition of what is a privacy leakage, for instance by moving from membership inference to reconstruction. Interesting dynamics could also arise from having different platform competing for a finite based of user and using different privacy and utility guarantees. Additionally, the privacy budget can vary over time, creating more complex feedback loops as well. 




\section*{Ethics statement}

The authors confirm that this manuscript complies with all applicable ethical standards and does not violate any relevant code of ethics.

\section*{Reproducibility statement}

All experiments were conducted on synthetically generated datasets that mimic the real world and have been sufficiently motivated in the paper as well as explained clearly in the code. We describe the simulated setup in sufficient detail to enable reproduction of the paper's figures. An anonymus Github link along with additional experimental details has been provided in Appendix \ref{app:exp_details}. All proofs are provided in Appendices~\ref{app:proofs}, \ref{app:gaussian_mia_proofs} and \ref{app:reg_theory}.
\section*{Acknowledgments}
The authors are supported by the National Research Agency under France 2030, reference “ANR-23-IACL-0008”. Preliminary work while Edwige Cyffers was visiting the Simons Institute for the Theory of Computing. Edwige Cyffers thanks Max Cairney-Leeming for the joke which started this line of work. Uddalak Mukherjee would like to acknowledge``Verified Deep Learning: Formal Methods Perspective" project under Ansuman Banerjee of ACMU, Indian Statistical Institute, Kolkata.

\bibliography{iclr-2027-style-files/iclr2027/iclr2027_conference}
\bibliographystyle{iclr2027_conference}

\appendix
\part{Appendix}
\parttoc
\clearpage
\section{Numerical Experiments For Section \ref{section:membership_inference}} \label{app:more_experiments_mia}


\begin{figure}[H]
    \centering


    \begin{subfigure}[t]{0.32\textwidth}
        \centering
        \includegraphics[width=\linewidth]{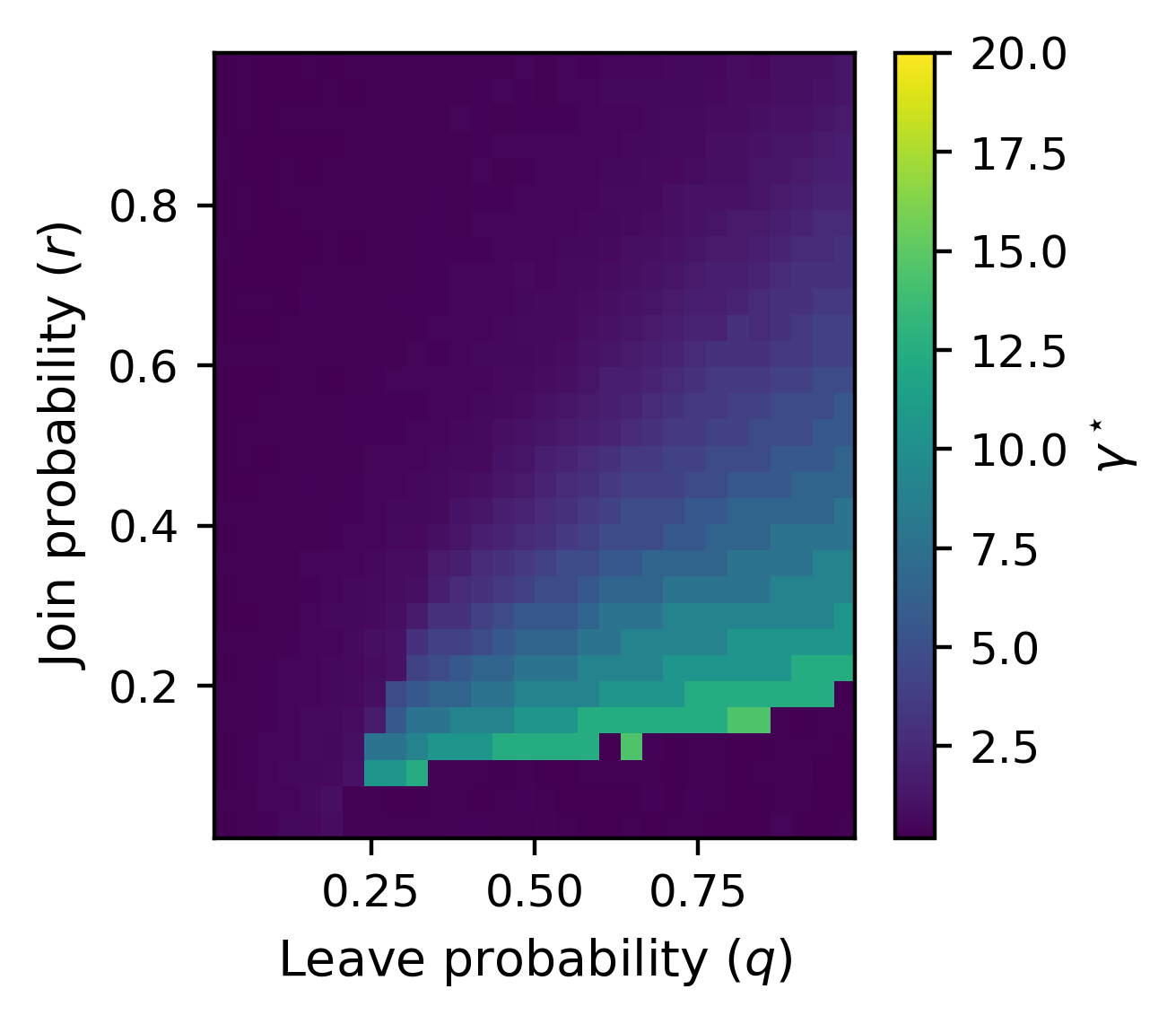}
        \caption{Empirical heatmap}
    \end{subfigure}
    \hfill
    \begin{subfigure}[t]{0.32\textwidth}
        \centering
        \includegraphics[width=\linewidth]{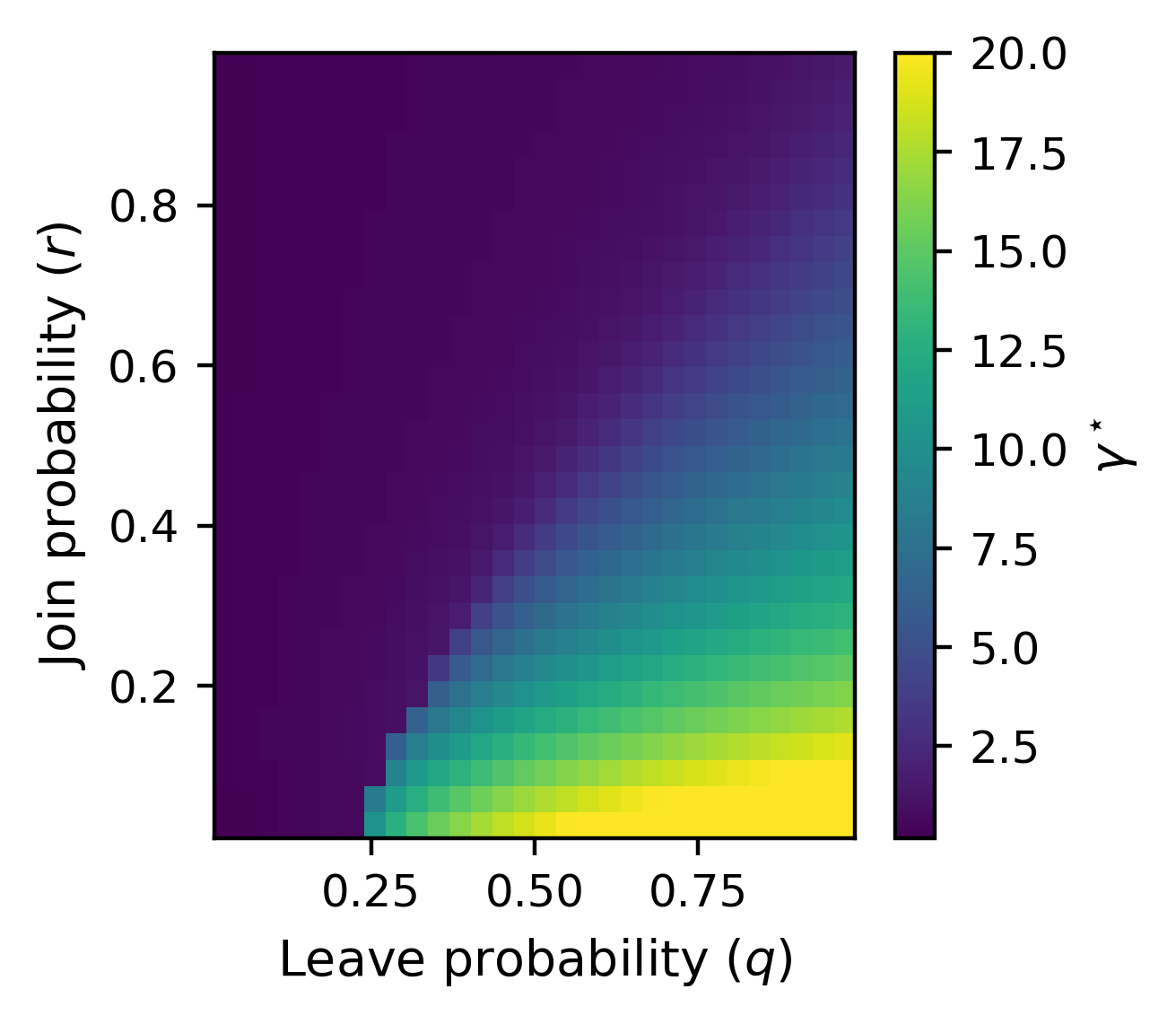}
        \caption{Theoritical heatmap}
    \end{subfigure}
    \hfill
    \begin{subfigure}[t]{0.32\textwidth}
        \centering
        \includegraphics[width=\linewidth]{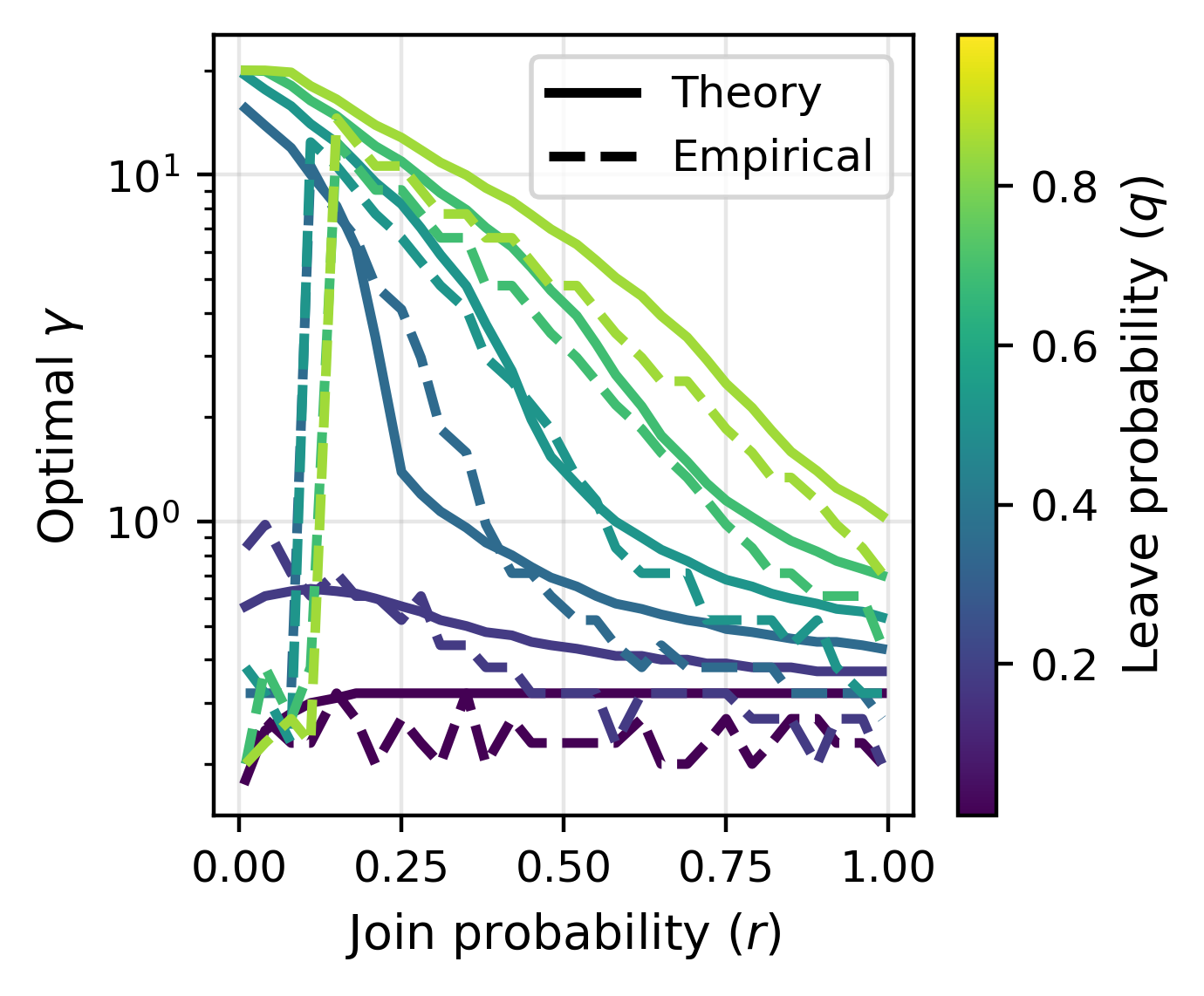}
        \caption{$\gamma$ line plot comparison}
    \end{subfigure}

    \vspace{0.5cm}


      \begin{subfigure}[t]{0.32\textwidth}
        \centering
        \includegraphics[width=\linewidth]{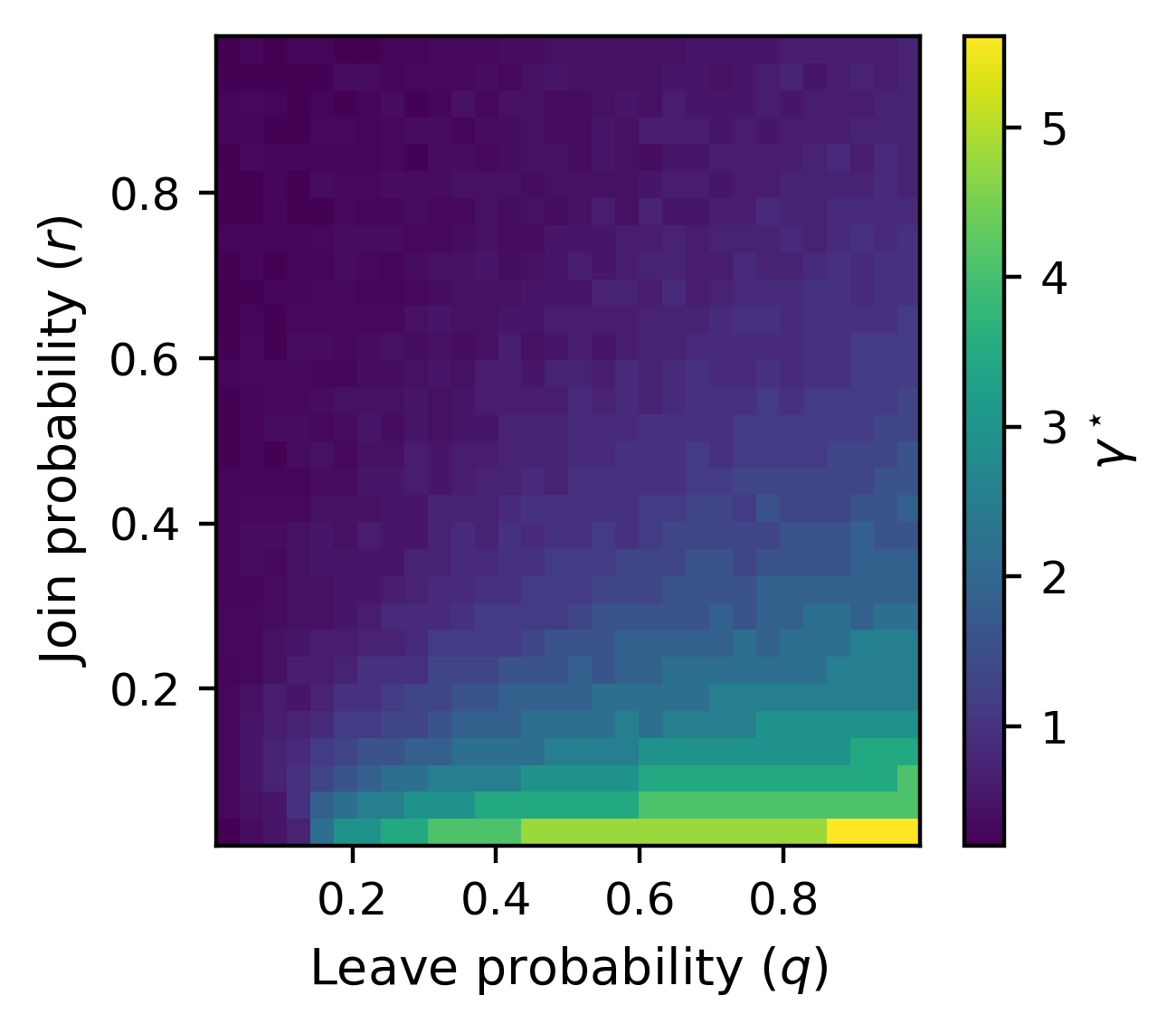}
        \caption{Empirical heatmap}
    \end{subfigure}
    \hfill
    \begin{subfigure}[t]{0.32\textwidth}
        \centering
        \includegraphics[width=\linewidth]{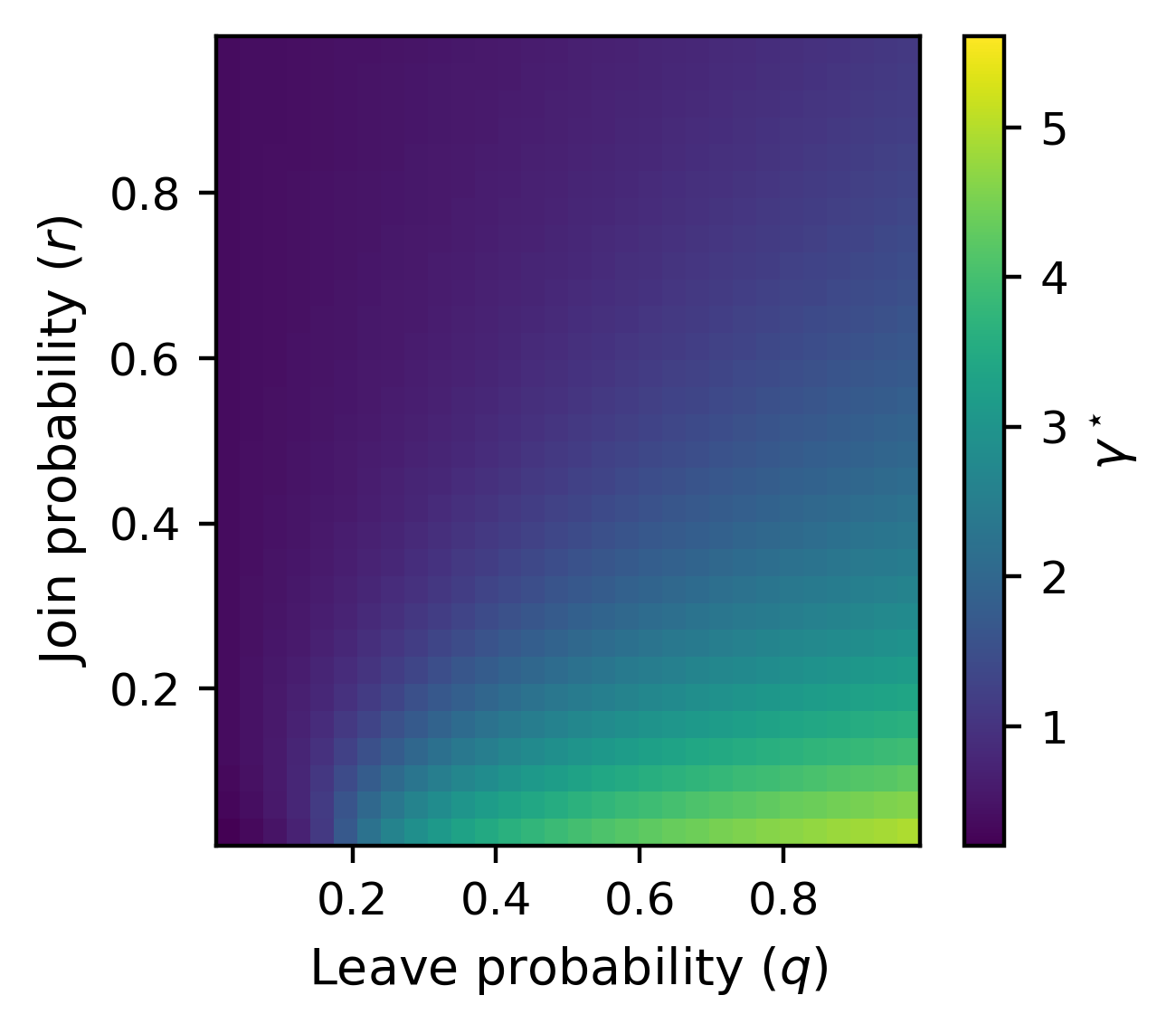}
        \caption{Theoretical heatmap}
    \end{subfigure}
    \hfill
    \begin{subfigure}[t]{0.32\textwidth}
        \centering
        \includegraphics[width=\linewidth]{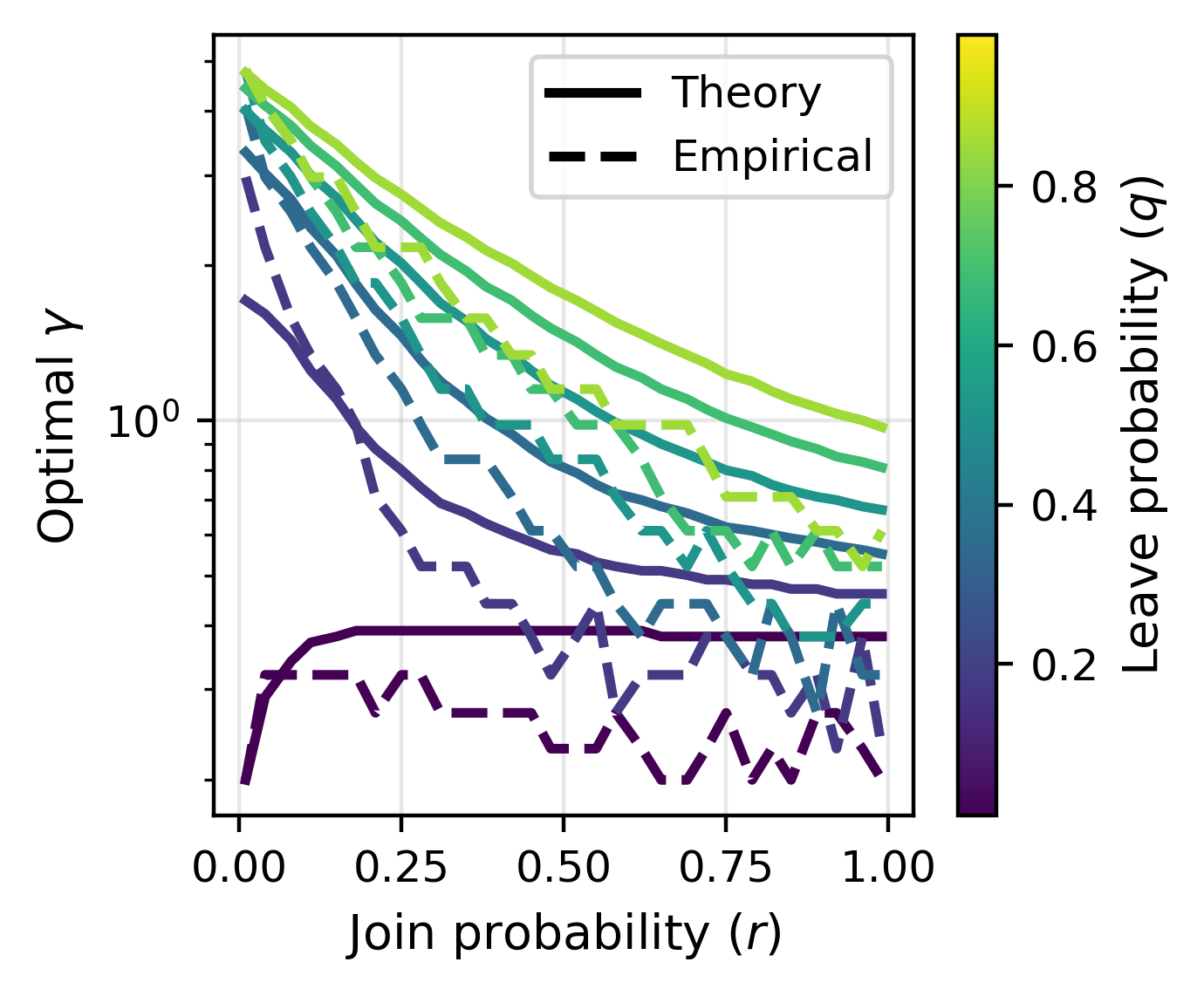}
        \caption{$\gamma$ line plot comparison}
    \end{subfigure}

    \begin{subfigure}[t]{0.32\textwidth}
        \centering
        \includegraphics[width=\linewidth]{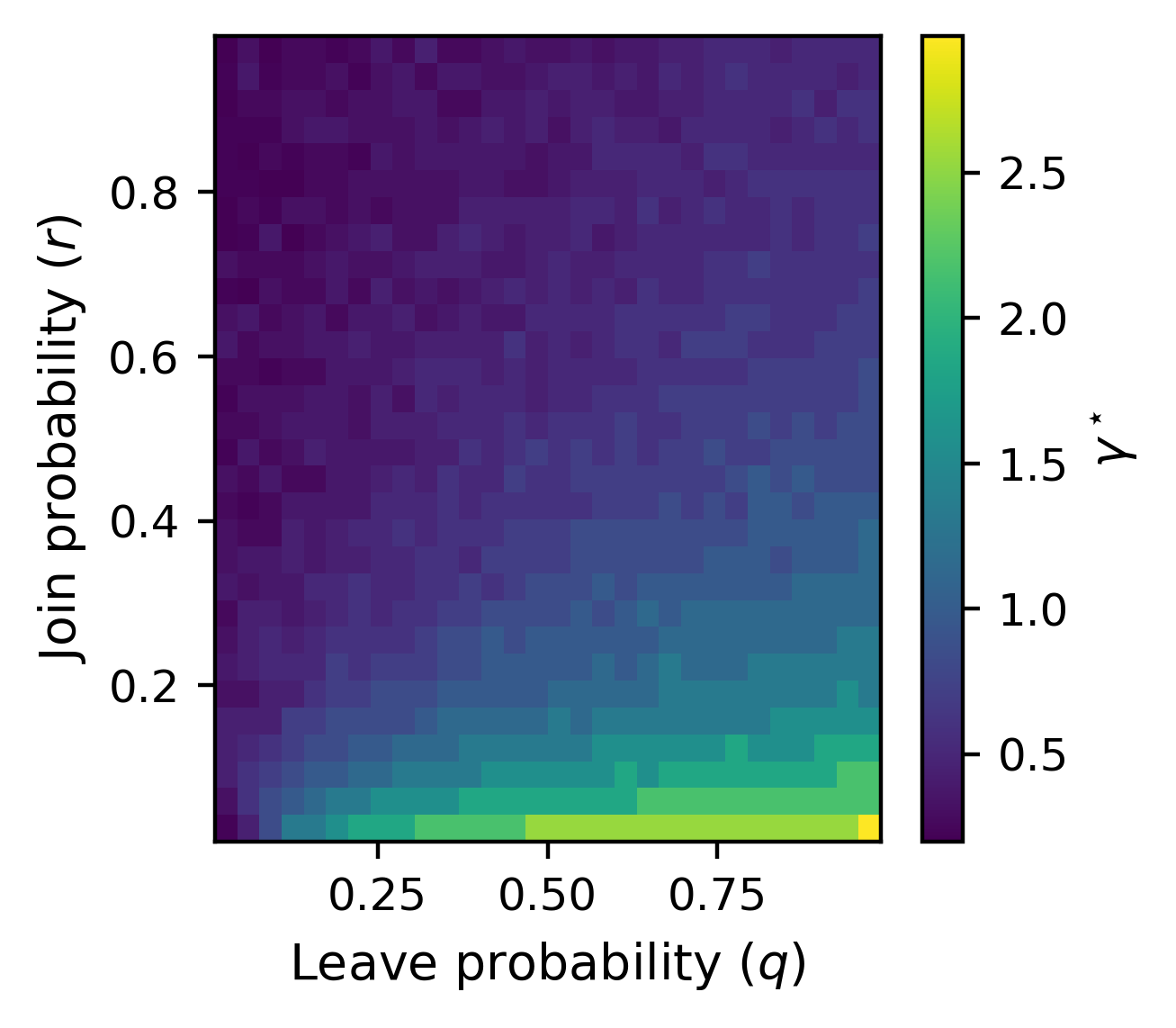}
        \caption{Empirical heatmap}
    \end{subfigure}
    \hfill
    \begin{subfigure}[t]{0.32\textwidth}
        \centering
        \includegraphics[width=\linewidth]{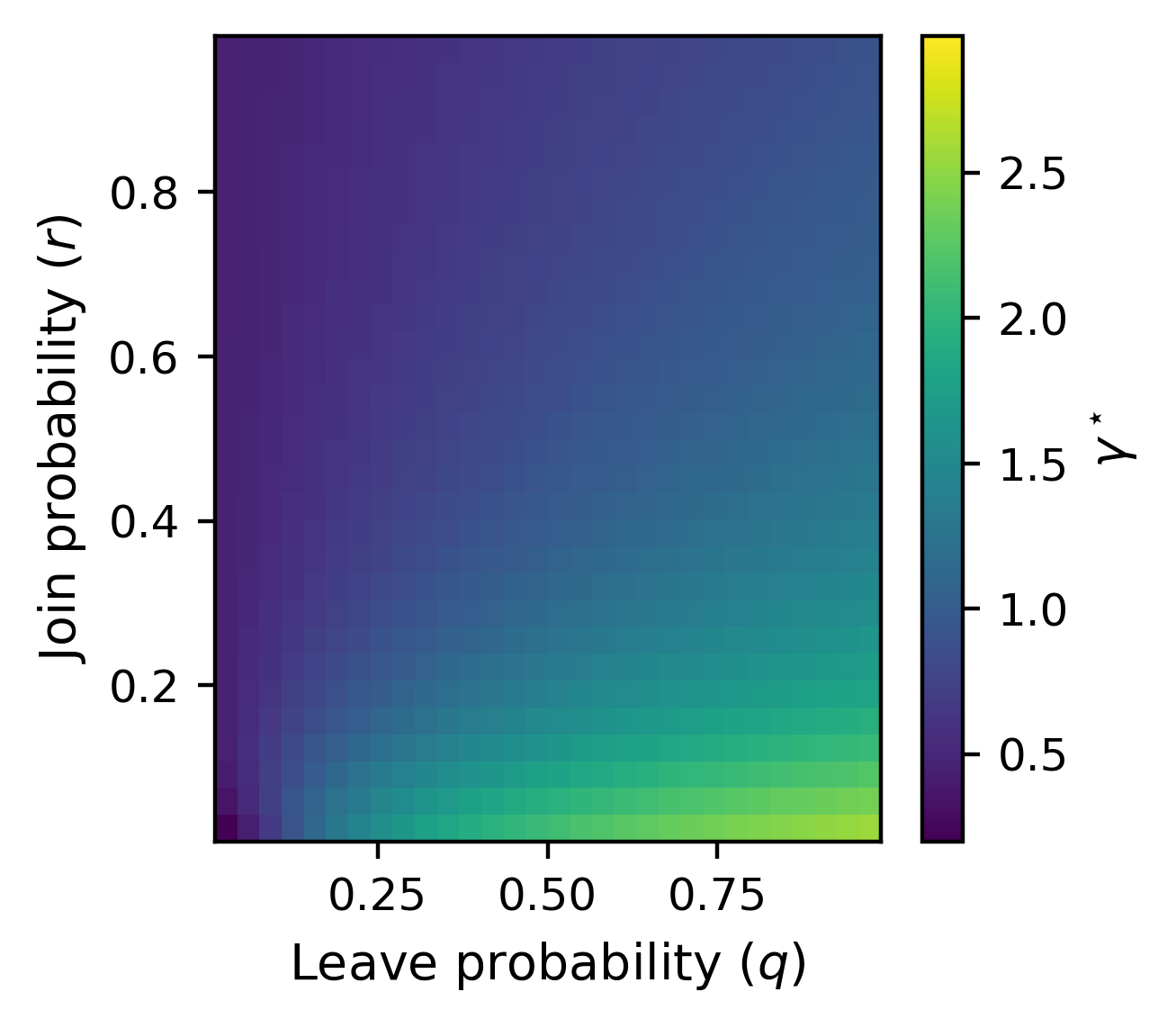}
        \caption{Theoritical heatmap}
    \end{subfigure}
    \hfill
    \begin{subfigure}[t]{0.32\textwidth}
        \centering
        \includegraphics[width=\linewidth]{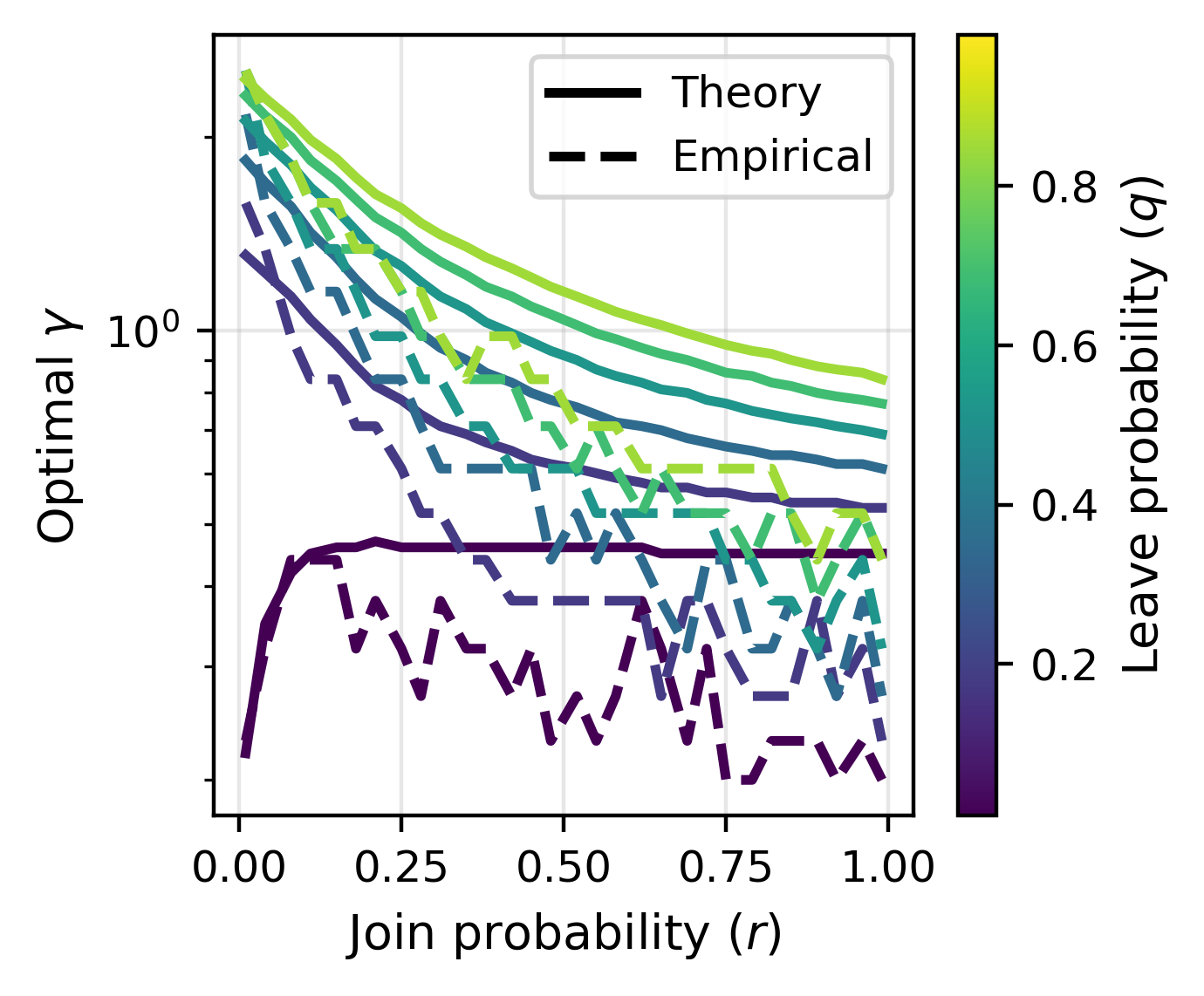}
        \caption{$\gamma$ line plot comparison}
    \end{subfigure}

    \begin{subfigure}[t]{0.32\textwidth}
        \centering
        \includegraphics[width=\linewidth]{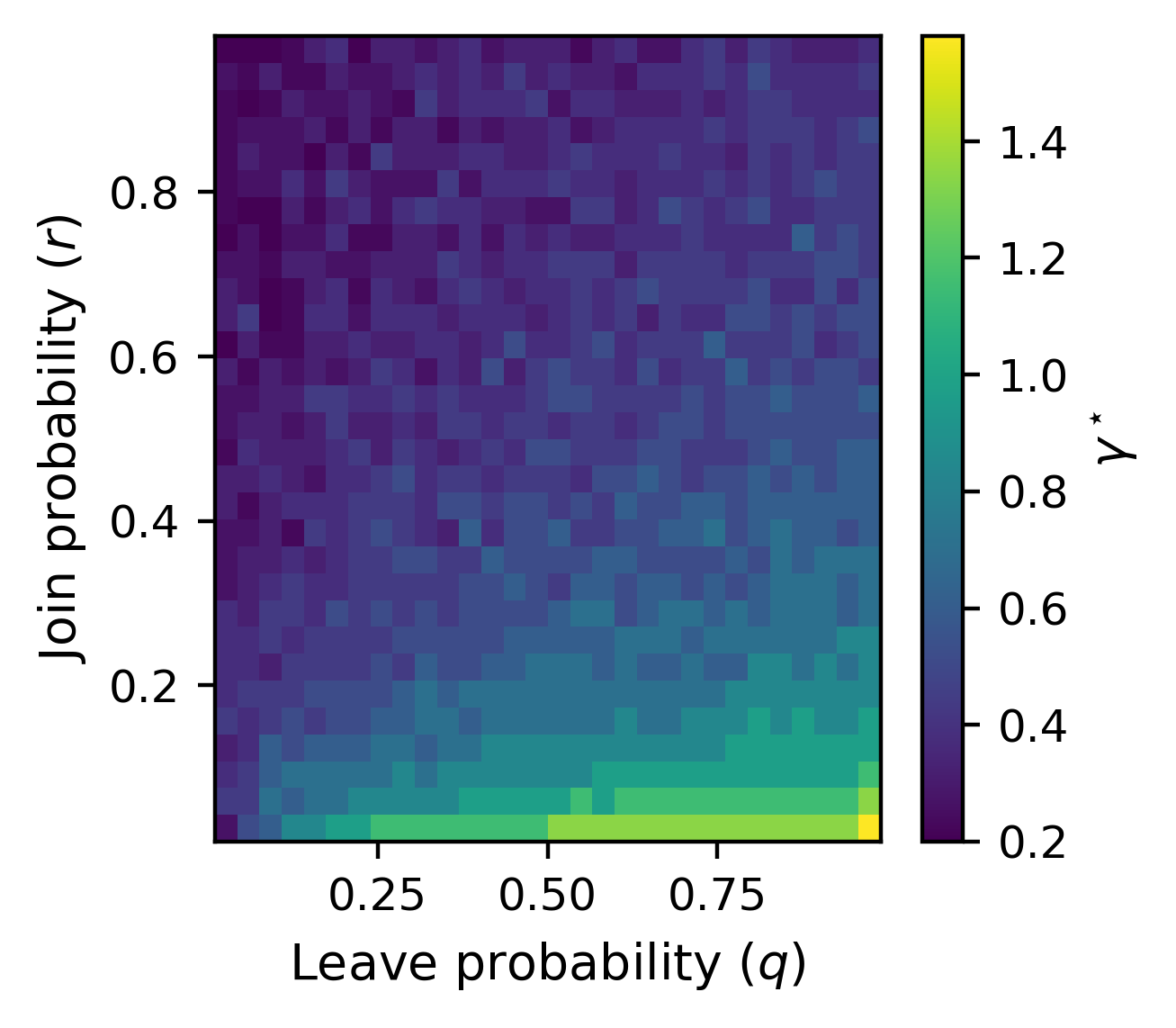}
        \caption{Empirical heatmap}
    \end{subfigure}
    \hfill
    \begin{subfigure}[t]{0.32\textwidth}
        \centering
        \includegraphics[width=\linewidth]{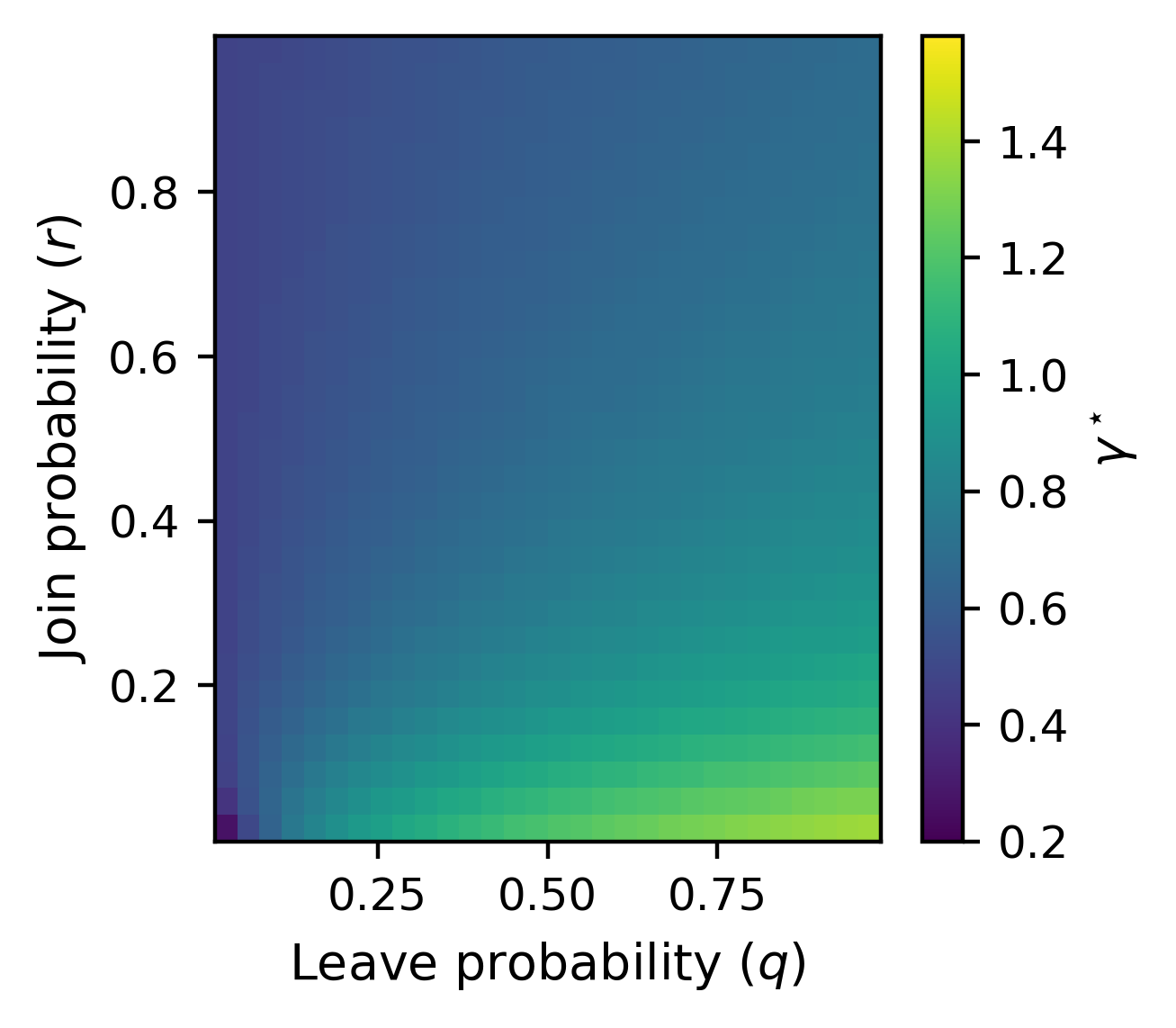}
        \caption{Theory heatmap}
    \end{subfigure}
    \hfill
    \begin{subfigure}[t]{0.32\textwidth}
        \centering
        \includegraphics[width=\linewidth]{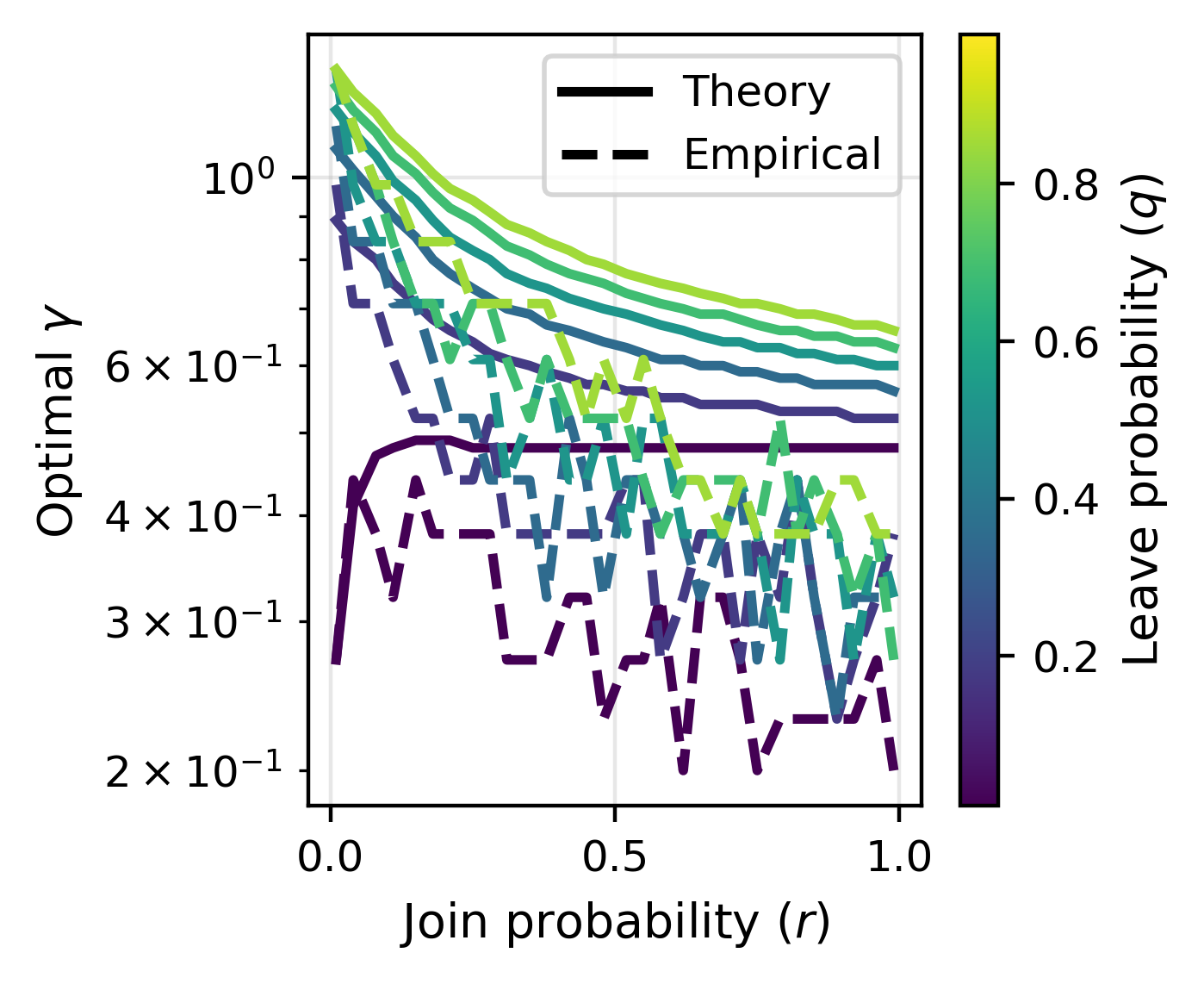}
        \caption{$\gamma$ line plot comparison}
    \end{subfigure}

    \vspace{0.5cm}

    \caption{
    Comparison of the optimal threshold $\gamma$ obtained from theory and simulations.
    The first row corresponds to $\tau = 0.01$, the second row corresponds to $\tau = 0.05$, the third row corresponds to $\tau = 0.1$ and the last row corresponds to $\tau = 0.2$. For each value of $\tau$, the left panel shows the empirical heatmap, the middle panel shows the theoretical heatmap, and the right panel compares the theoretical and empirical optimal $\gamma$ as a function of the join probability.
    }
    \label{fig:gamma_comparison_more}
\end{figure}

Figures \ref{fig:gamma_comparison} and \ref{fig:gamma_comparison_more} reveal an interesting observation: as $\tau$ increases, the empirical noise increasingly underestimates the theoretical prediction because leakage becomes more sensitive to rare, extreme outliers. While the theoretical model assumes an unbounded population, the experimental setup imposes a population cap of 20,000 so the discrepancy between theory and experiment becomes more pronounced at larger $\tau$.


\begin{figure}[tb]
    \centering
    \begin{subfigure}[b]{0.28\textwidth}
        \centering
        \includegraphics[width=\textwidth]{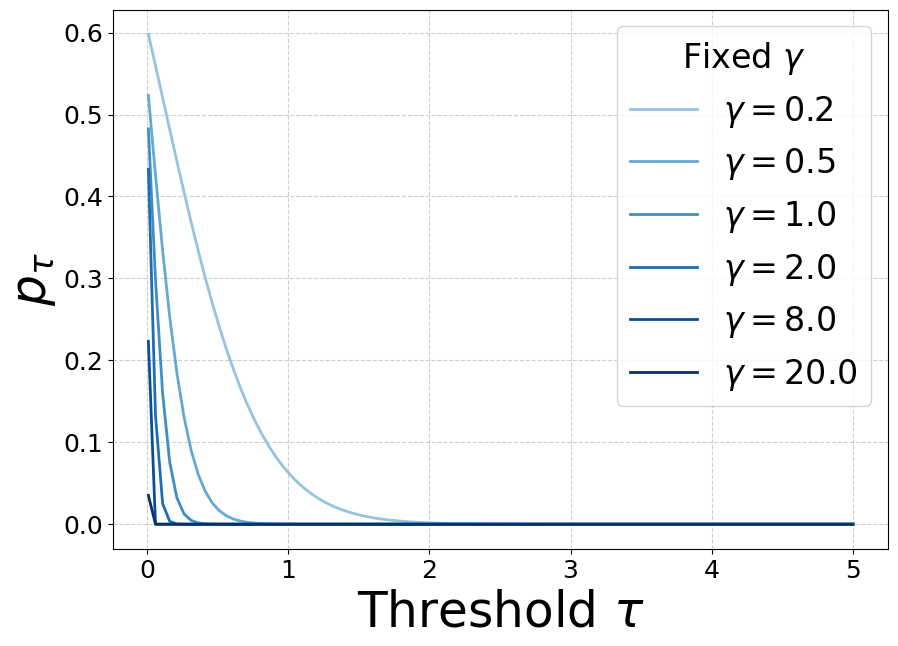}
        \caption{$p_\tau$ vs $\tau$}
        \label{fig:sub_ptau}
    \end{subfigure}
    \hfill
    \begin{subfigure}[b]{0.32\textwidth}
        \centering
        \includegraphics[width=\textwidth]{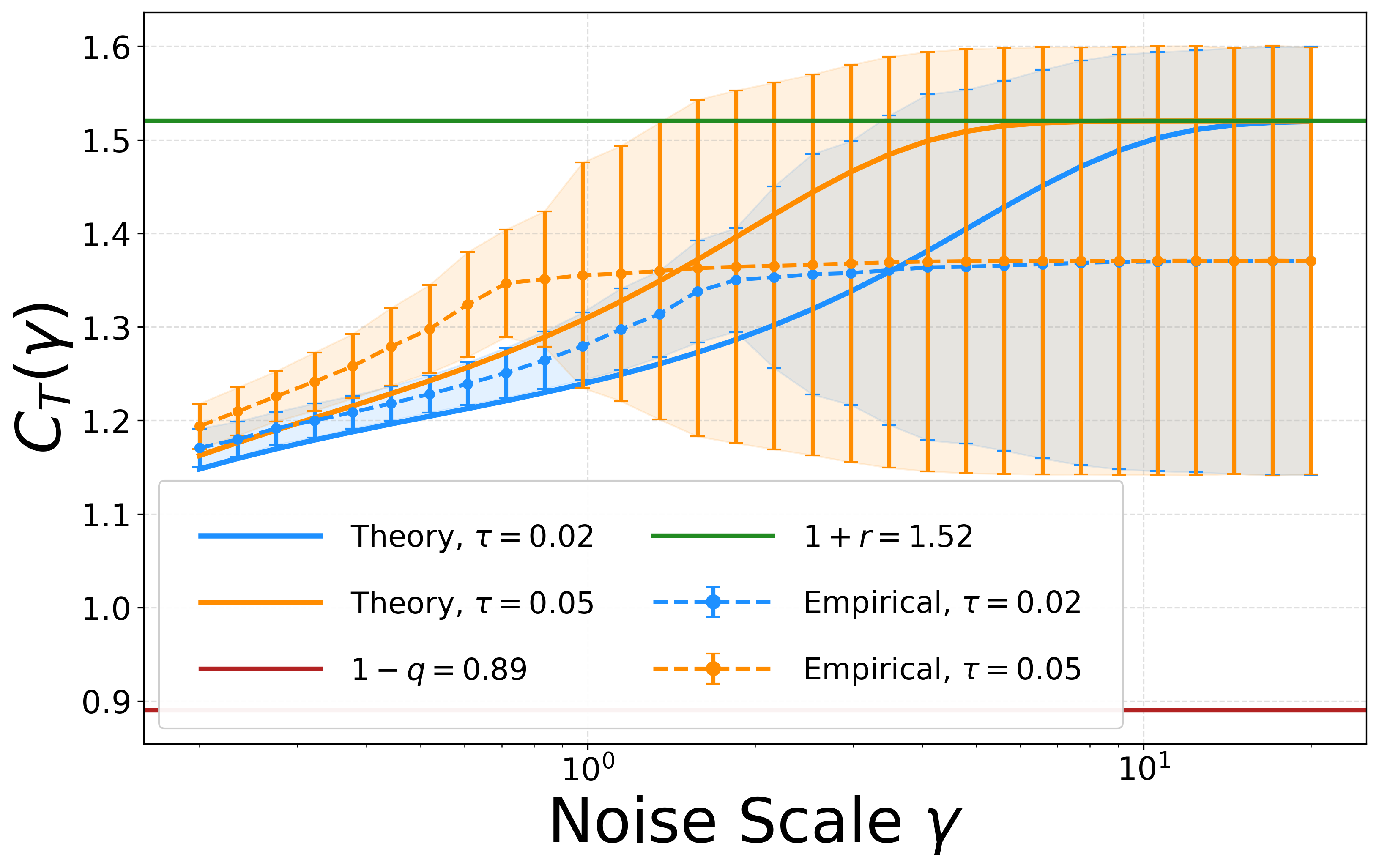} 
        \caption{$C(\gamma)$ vs $\gamma$ for $r>q$}
        \label{fig:sub_cgamma_theory}
    \end{subfigure}
    \hfill
    \begin{subfigure}[b]{0.32\textwidth}
        \centering
        \includegraphics[width=\textwidth]{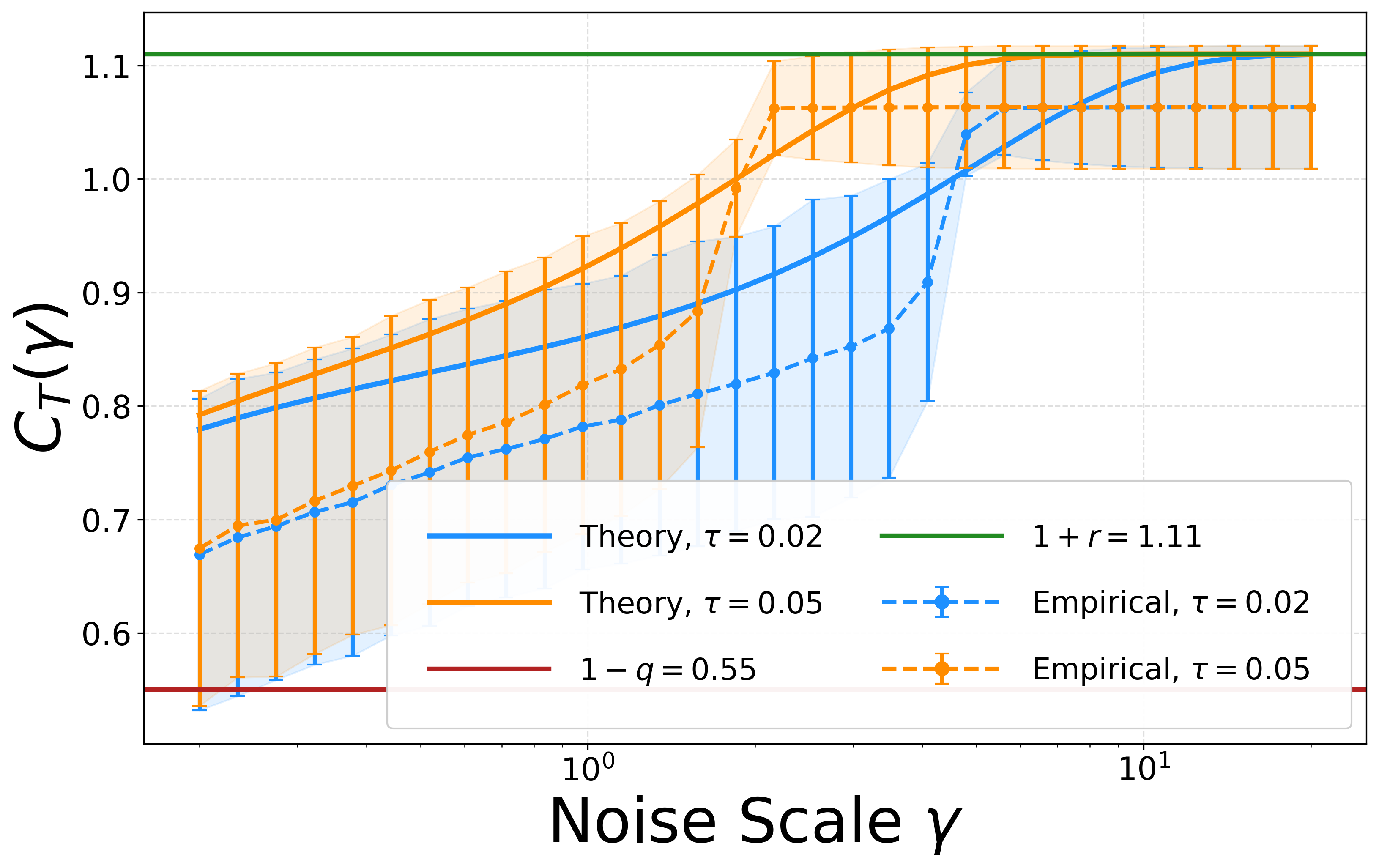} 
        \caption{$C(\gamma)$ vs $\gamma$ for $q>r$}
        \label{fig:sub_cgamma_practice}
    \end{subfigure}
    
    \caption{Behavior of \(p_\tau\) with \(\tau\) and \(C(\gamma)\) with \(\gamma\), comparing theoretical predictions with empirical simulations. (a) \(N=100\); (b) \(q=0.11, r=0.52, T=10\); (c) \(q=0.45, r=0.11, T=10\), with the population clipped at \(20{,}000\).
}
    \label{fig:combined_behaviour}
\end{figure}
\section{Detailed Plots from Experiments in Section \ref{section:randomized_response}} \label{app:detailed_plots_rr}

\begin{figure}[htbp]
    \centering

    \begin{subfigure}[b]{0.36\textwidth}
        \centering
        \includegraphics[width=\textwidth]{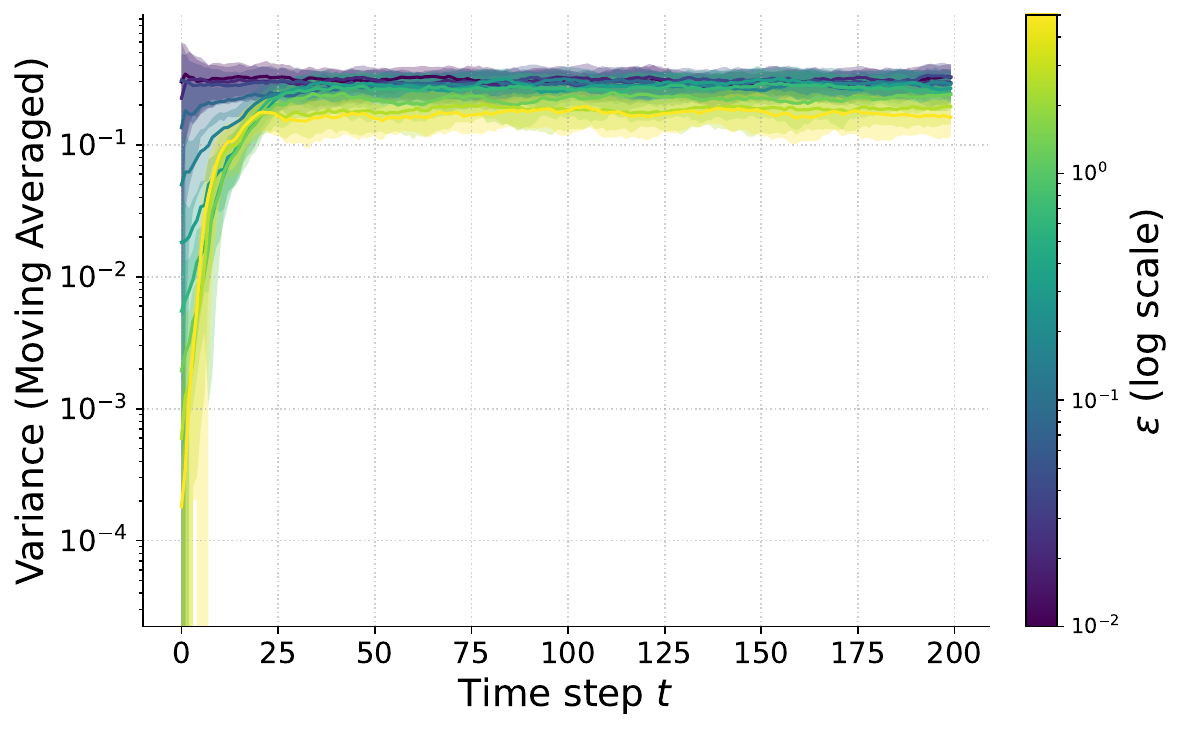}
        \caption{Moving-averaged variance ($q=0.6$, $r=0.3$)}
        \label{fig:var_theta06_lambda03}
    \end{subfigure}
    \hspace{0.5cm}
    \begin{subfigure}[b]{0.36\textwidth}
        \centering
        \includegraphics[width=\textwidth]{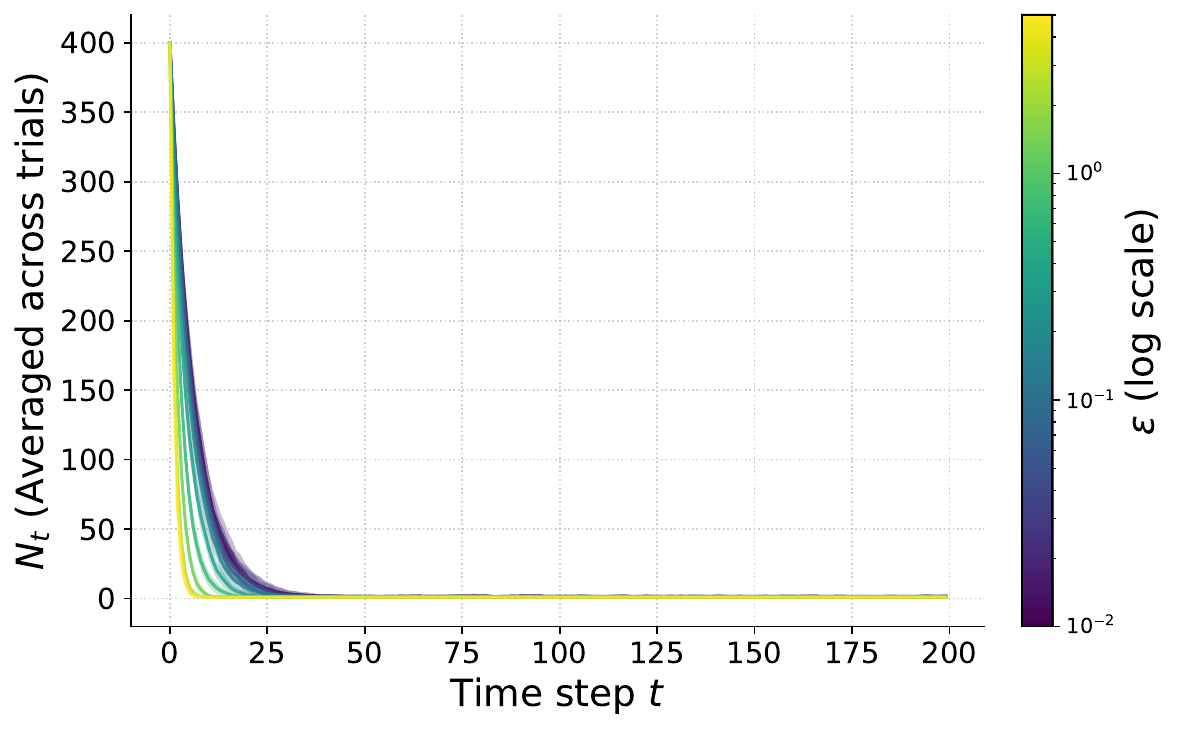}
        \caption{Population evolution ($q=0.6$, $r=0.3$)}
        \label{fig:pop_theta06_lambda03}
    \end{subfigure}

    \vspace{0.5cm}

    \begin{subfigure}[b]{0.36\textwidth}
        \centering
        \includegraphics[width=\textwidth]{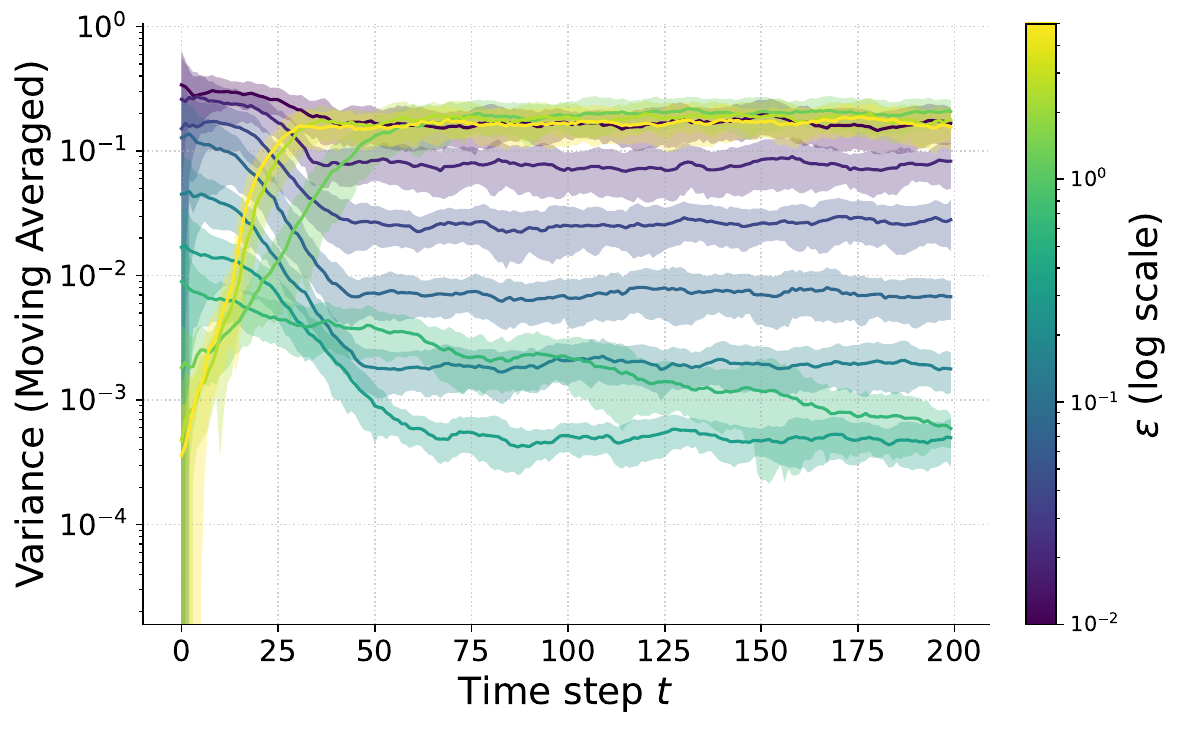}
        \caption{Moving-averaged variance ($q=0.3$, $r=0.6$)}
        \label{fig:var_theta03_lambda06}
    \end{subfigure}
    \hspace{0.5cm}
    \begin{subfigure}[b]{0.36\textwidth}
        \centering
        \includegraphics[width=\textwidth]{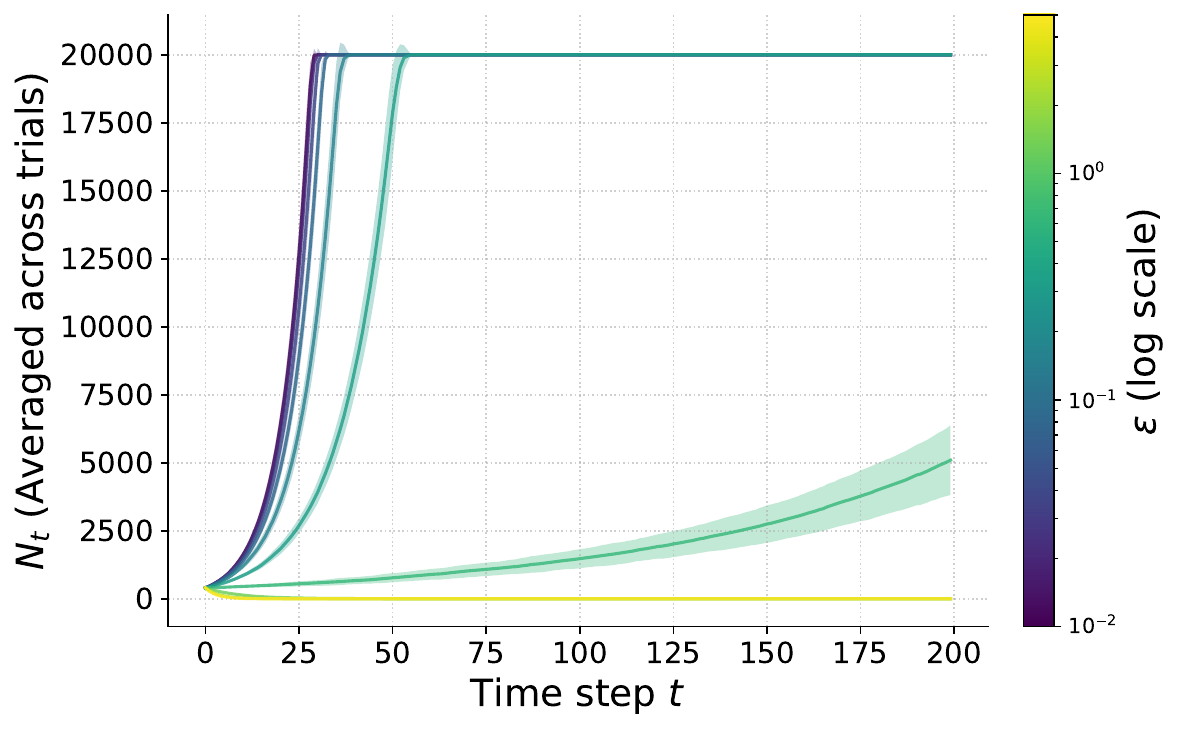}
        \caption{Population evolution ($q=0.3$, $r=0.6$)}
        \label{fig:pop_theta03_lambda06}
    \end{subfigure}

    \caption{Comparison of moving-averaged variance and population evolution under two parameter settings: the top row corresponds to $(q,r)=(0.6,0.3)$, while the bottom row corresponds to $(q,r)=(0.3,0.6)$.}
    \label{fig:population_evolution_rr}
\end{figure}


\section{The Agnostic Model}
\label{app:total_model}
This section explores the agnostic model mentioned in Section 3. Under this framework, we assume population growth is directly proportional to the number of participants in the previous iteration, $t-1$, regardless of wether or not their data were leaked. Applying a derivation analogous to the proof of Lemma \ref{lemm:unleaked_growth}, we characterize the evolution of the participant pool as follows:
\begin{equation} 
    \mathbb E[N_t\mid N_{t-1}] = N_{t-1} - N_{t-1}q\left(\frac{1}{2}+\beta\right) + r N_{t-1} 
    = N_{t-1} \left(1-q\left(\frac{1}{2}+\beta\right) + r\right)
\end{equation}
Under this model, we once again simulate the behavior of the population and variance under the same two regimes of $(r,q)=(0.3,0.6)$ and $(r,q)=(0.6,0.3)$ respectively as demonstrated in Figure \ref{fig:total_simulations}. In the regime ($r > q$), rapid growth outweighs attrition, driving the population to maximum capacity regardless of the privacy budget. Unlike the previous model --- where low privacy could stifle growth or lead to suboptimal stability, this model provides no incentive for privacy, as the highest $\eps$ consistently yields the lowest variance. Conversely, when $q > r$, the variance mirrors the original model’s behavior, yet population decay is notably more gradual, with certain $\eps$ values inducing high-variance stochastic collapse across different runs.

\begin{figure}[t!]
     \centering
     \begin{subfigure}[b]{0.32\textwidth}
         \centering
         \includegraphics[width=\textwidth]{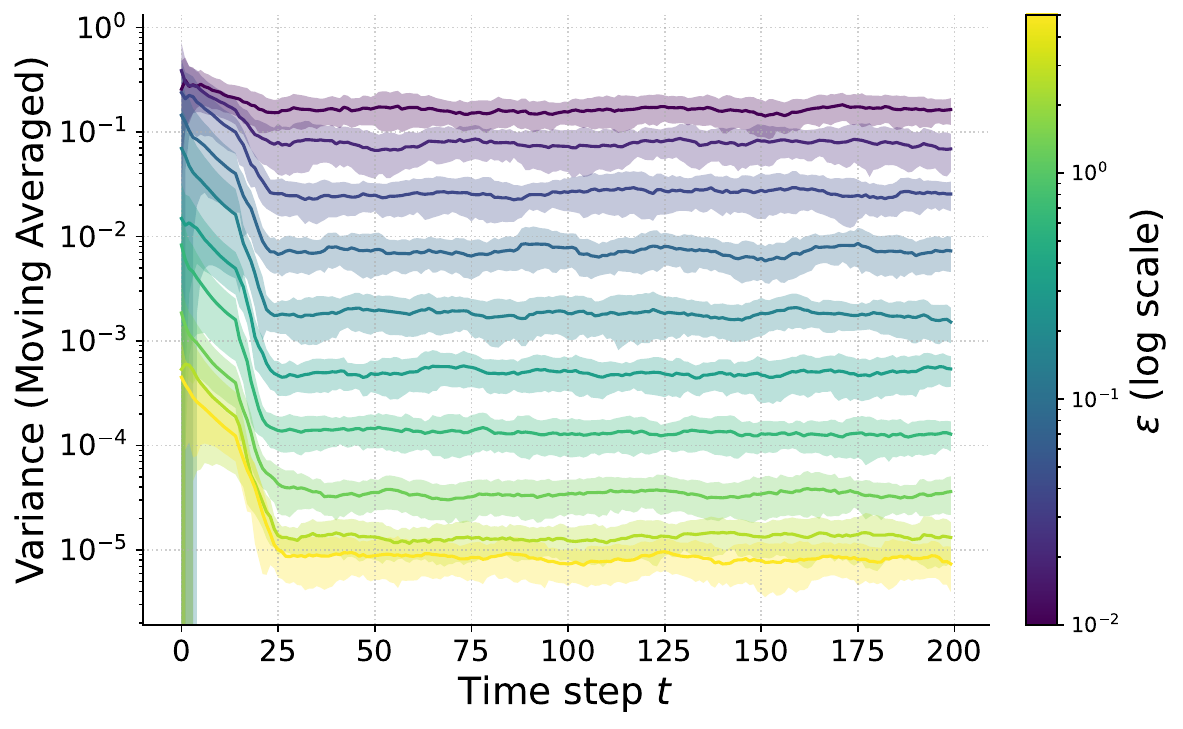}
         \caption{Moving averaged variance for $(r,q)=(0.6,0.3)$}
         \label{fig:left_top}
     \end{subfigure}
     \hspace{0.5cm}
     \begin{subfigure}[b]{0.32\textwidth}
         \centering
         \includegraphics[width=\textwidth]{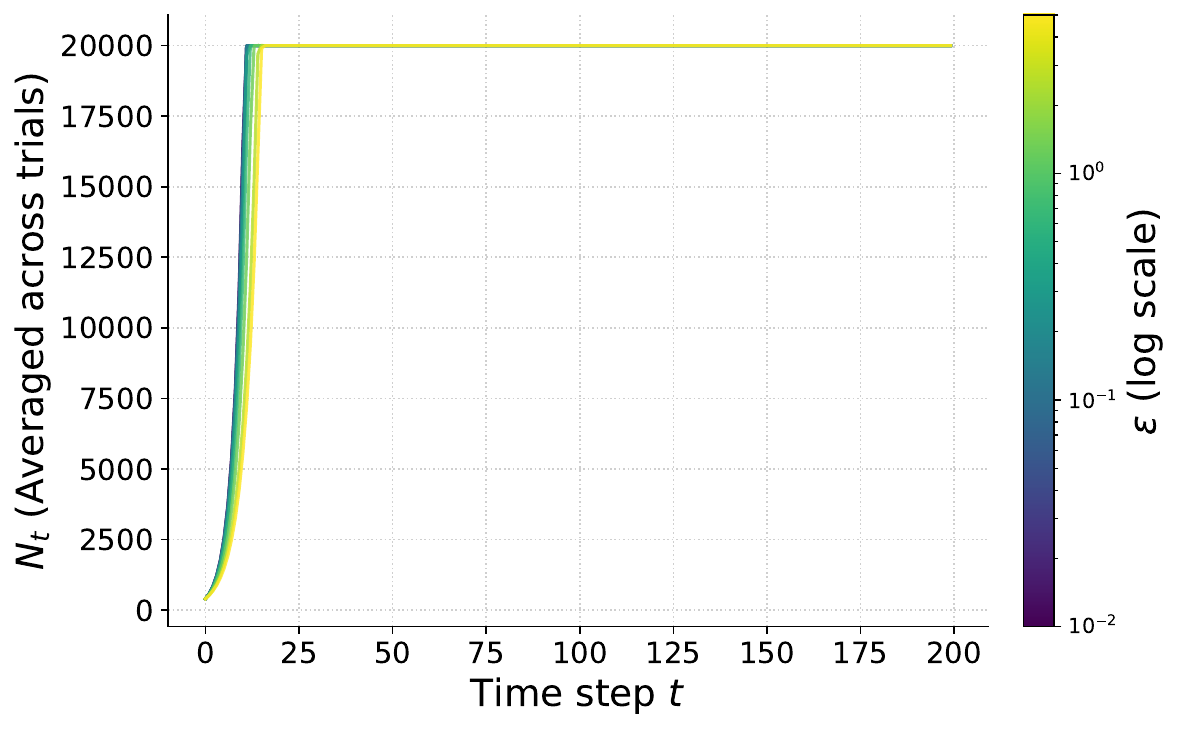}
         \caption{Population evolution for $(r,q)=(0.6,0.3)$}
         \label{fig:right_top}
     \end{subfigure}
     
     \vspace{1cm} 
     
     \begin{subfigure}[b]{0.32\textwidth}
         \centering
         \includegraphics[width=\textwidth]{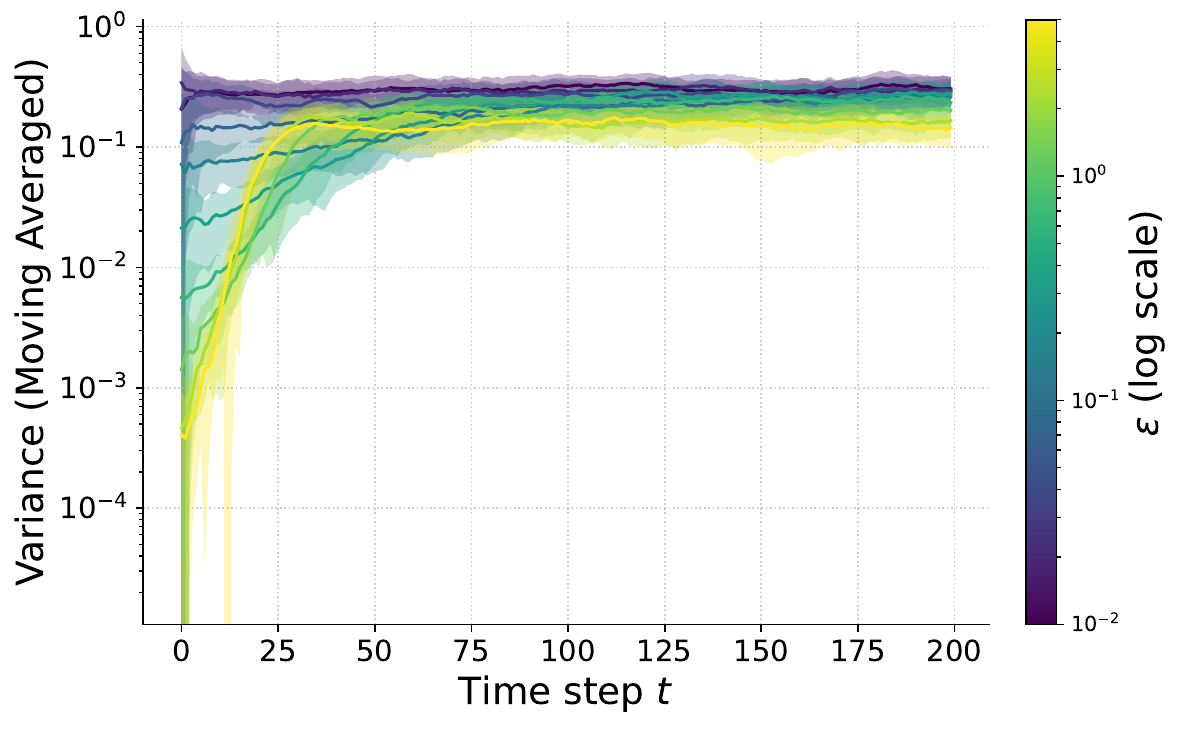}
         \caption{Moving averaged variance for $(r,q)=(0.3,0.6)$}
         \label{fig:left_bottom}
     \end{subfigure}
     \hspace{0.5cm}
     \begin{subfigure}[b]{0.32\textwidth}
         \centering
         \includegraphics[width=\textwidth]{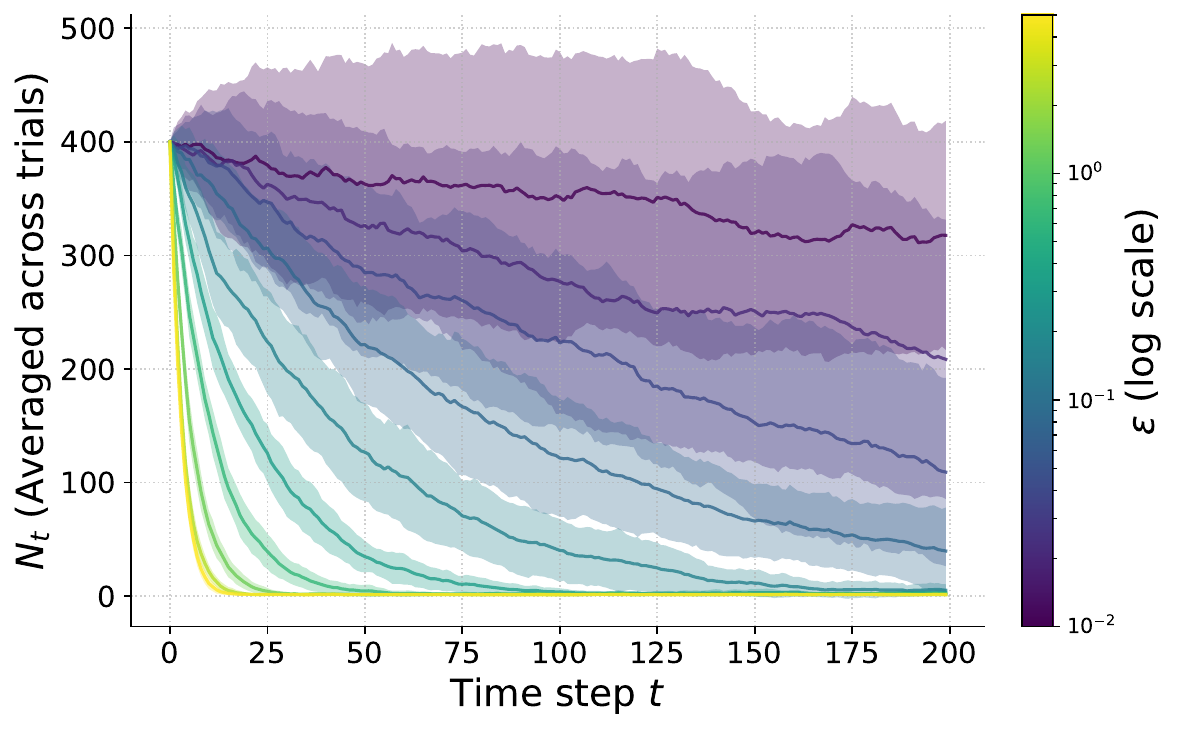}
         \caption{Population evolution for $(r,q)=(0.3,0.6)$}
         \label{fig:right_bottom}
     \end{subfigure}
     
     \caption{In a growing population ($r=0.6,q=0.3$), the system quickly reaches capacity across all privacy budgets, causing variance to stratify based on noise levels and favoring low privacy due to the high growth rate. In contrast, a shrinking population ($r=0.3,q=0.6$) triggers performative collapse; while weak privacy (high $\eps$) initially yields higher utility, it accelerates population decline. This causes variance to spike and eventually converge with stricter privacy regimes as the disappearing user base eliminates any initial advantage.}
     \label{fig:total_simulations}
\end{figure}

We also report the utility-maximizing privacy budget $\eps$ for each parameter pair in Figure~\ref{fig:heatmap_total}. For growing populations ($r > q$), the models diverge sharply: the original model shows a broad transition favoring intermediate privacy, whereas the agnostic model overwhelmingly favors minimal privacy (highest $\eps$), as rapid growth drives the population to capacity regardless of data use. Maximum privacy is beneficial only within a narrow band just below $r=q$; outside this regime, strict privacy provides no utility gain.

\begin{figure}[t!]
    \centering
    \includegraphics[width=0.32\linewidth]{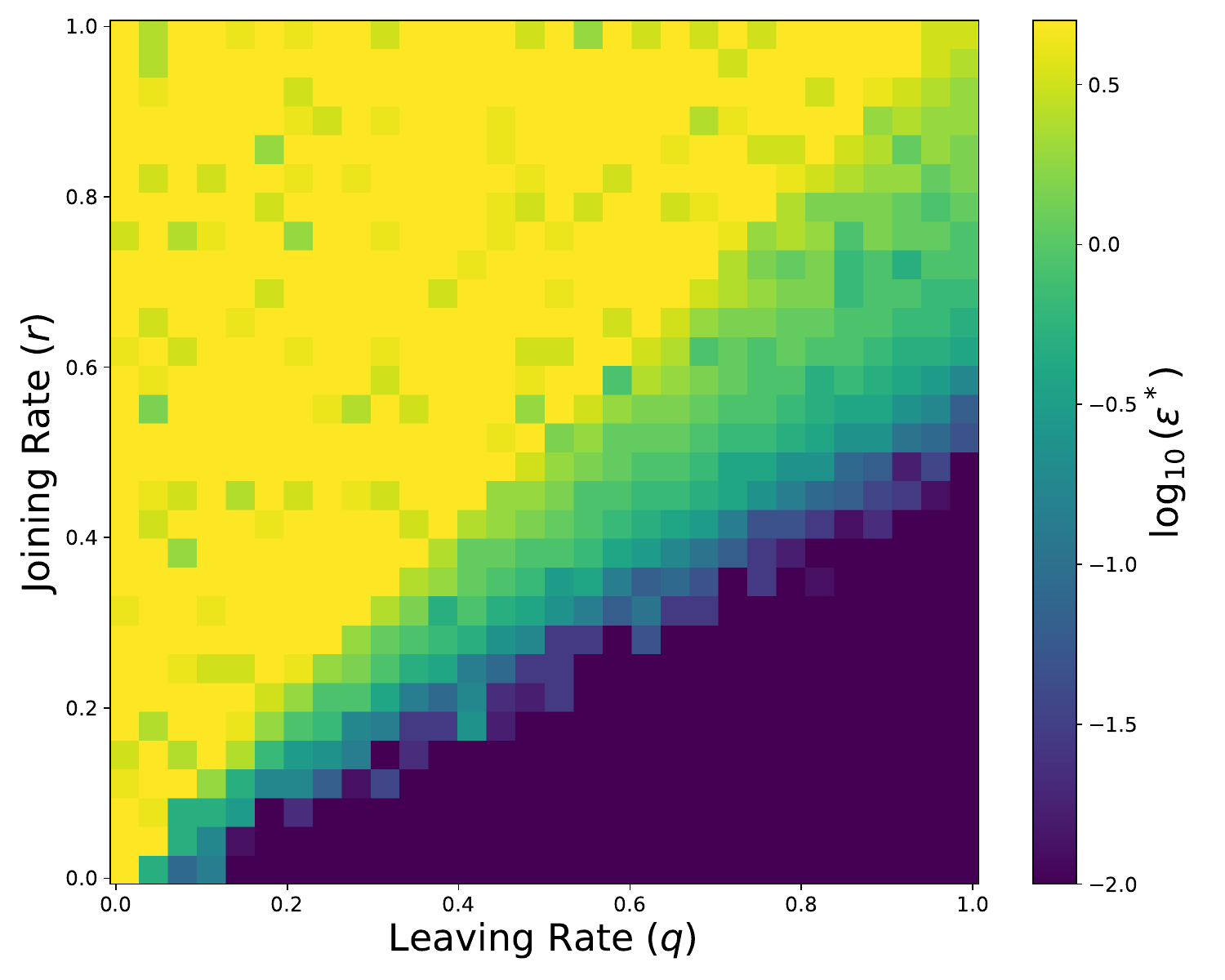}
    \caption{Global mapping of the optimal privacy budget $\eps^*$ across the $(r, q)$ parameter space for the agnostic model}
    \label{fig:heatmap_total}
\end{figure}




\section{Details About Numerical Experiments} \label{app:exp_details}

To empirically validate our theoretical framework, we simulate both the DP mean estimation and DP-OLS regression mechanisms over a finite horizon of $T=50$ rounds with an initial population of $N_0=1000$ agents and feature dimension $d=12$. For the mean estimation task, individual data records are drawn from a standard Gaussian $\mathcal{N}(0, I_d)$ for simplicity; while our theoretical analysis explicitly accommodates a shifting true mean $\mu_t$, the stationary setting sufficiently demonstrates the population dynamics. To prove that our results hold robustly under continuous distributional drift, we implement a dynamic target in the DP-OLS regression experiments: at each round $t$, the true regression coefficient vector $\theta^*_t$ is freshly sampled from the surface of the $d$-dimensional unit ball, inducing persistent concept drift in the target variables $y_i$.  For evaluating the theoretical cumulative error (from \eqref{eq:ridge_objective_closed form}) in the regression setting, the term $\nu$, which bounds the non-private estimation error, is explicitly configured as $\nu = d \sigma_\epsilon^2$. Model utility is measured via Mahalanobis distance for the mean estimator and via mean squared prediction error on a fresh, held-out test set ($N_{\text{test}} = 2000$) for the regression model. To ensure our evaluation accurately reflects the risk of demographic collapse, any trajectory where the population drops below the necessary OLS rank requirement ($N_t \le d$) is assigned a fixed maximum penalty loss for all remaining rounds. Complete details regarding data clipping radii, ridge regularization constants for matrix stability, grid search intervals for the noise scale $\gamma$, and the full implementation code can be found in our anonymous GitHub repository: (\url{https://anonymous.4open.science/r/performative-privacy-70F1/}).

\section{Proofs from Section \ref{section:randomized_response}}\label{app:proofs}


\subsection{Properties of the Estimator}

\begin{lemma}
    The estimator $\hat \mu_t$ for population Bernoulli mean $\mu_t$ as described in  \eqref{eq:unbiased_estimator} is unbiased.
\end{lemma}
\begin{proof}
To prove that the estimator $\hat{\mu}_t$ is unbiased, we first find the expected value of a single privatized response $Y_{i,t}$. As $X_{i,t}$ is a Bernoulli variable with mean $\mu_t$, we have:
\begin{align*}
\mathbb{E}[Y_{i,t}] &= \mathbb{P}(X_{i,t} = 1) \mathbb{E}[Y_{i,t} \mid X_{i,t} = 1] + \mathbb{P}(X_{i,t} = 0) \mathbb{E}[Y_{i,t} \mid X_{i,t} = 0] \\
&= \mu_t \left( \frac{1}{2} + \beta \right) + (1 - \mu_t) \left( \frac{1}{2} - \beta \right) \\
&= \mu_t \left( \frac{1}{2} + \beta \right) + \frac{1}{2} - \beta - \mu_t \left( \frac{1}{2} - \beta \right) \\
&= \mu_t \left( \frac{1}{2} + \beta - \frac{1}{2} + \beta \right) + \frac{1}{2} - \beta \\
&= 2\beta\mu_t + \frac{1}{2} - \beta
\end{align*} 
Now we evaluate the expected value of the sample mean $\hat{\nu}_t$. By linearity of expectation:
\begin{align*}
\mathbb{E}[\hat{\nu}_t] &= \mathbb{E} \left[ \frac{1}{N_t}\sum_{i=1}^{N_t} Y_{i,t} \right] 
= \frac{1}{N_t}\sum_{i=1}^{N_t} \mathbb{E}[Y_{i,t}] \\
&= \frac{1}{N_t} \cdot N_t \left( 2\beta\mu_t + \frac{1}{2} - \beta \right) 
= 2\beta\mu_t + \frac{1}{2} - \beta
\end{align*}
Substituting $\mathbb{E}[\hat{\nu}_t]$ into the expected value of our estimator $\hat{\mu}_t$:
\begin{align*}
\mathbb{E}[\hat{\mu}_t] &= \mathbb{E} \left[ \frac{\hat{\nu}_t - \frac{1}{2} + \beta}{2\beta} \right] 
= \frac{\mathbb{E}[\hat{\nu}_t] - \frac{1}{2} + \beta}{2\beta} \\
&= \frac{\left( 2\beta\mu_t + \frac{1}{2} - \beta \right) - \frac{1}{2} + \beta}{2\beta} \\
&= \frac{2\beta\mu_t}{2\beta} 
= \mu_t
\end{align*}
This concludes the proof.
\end{proof}

\begin{lemma}\label{lemm:eps_ldp_beta}
For any given privacy budget $\eps > 0$, setting the perturbation parameter to
\begin{align*}
    \beta = \frac{1}{2} \tanh\left(\frac{\eps}{2}\right)
\end{align*}
guarantees that the randomized response mechanism satisfies $\eps$-local differential privacy ($\eps$-LDP).
\end{lemma}

 \begin{proof}
    To satisfy $\eps$-LDP, the maximum ratio of output probabilities (true reporting vs. false reporting) must equal $e^\eps$:
    \begin{align*}
        \frac{\frac{1}{2} + \beta}{\frac{1}{2} - \beta} &= e^\eps
    \end{align*}
    Solving directly for $\beta$ gives:
    \begin{align*}
        \beta &= \frac{1}{2} \left( \frac{e^\eps - 1}{e^\eps + 1} \right)
    \end{align*}
    Using the identity $\tanh(z) = \frac{e^{2z} - 1}{e^{2z} + 1}$ with $z = \frac{\eps}{2}$, this simplifies to:
    \begin{align*}
        \beta &= \frac{1}{2} \tanh\left(\frac{\eps}{2}\right)
    \end{align*}
\end{proof}

\subsection{Population Evolution and Critical Noise Threshold}

\begin{lemma}\label{lemm:unleaked_growth}
    The evolution of the number of participants is given by:
    \begin{align*} 
    \mathbb{E}[N_t|N_{t-1}] &= N_{t-1} - N_{t-1}q\left(\frac{1}{2}+\beta\right) + r \left(\frac{1}{2}-\beta\right) N_{t-1} \nonumber\\ 
    &= N_{t-1} \left(1-q\left(\frac{1}{2}+\beta\right) + r\left(\frac{1}{2}-\beta\right)\right)
\end{align*}
\end{lemma}

\begin{proof}
    Let $D_t$ and $A_t$ be the number of departures and arrivals at $t$. The population at time $t$ follows the basic balance equation:
    \begin{align*}
        N_t = N_{t-1} - D_t + A_t
    \end{align*}
By construction of the model, each of the $N_{t-1}$ participants independently leaves with probability $q(\frac12+\beta)$, recruits a new participant with probability $r(\frac12-\beta)$, or does neither, so that, conditionally on $N_{t-1}$, $A_t$ and $D_t$ follow $\text{Binomial}(N_{t-1},r(1/2-\beta))$ and $\text{Binomial}(N_{t-1},q(\beta+1/2))$ respectively (jointly, $(D_t,A_t,N_{t-1}-D_t-A_t)$ is multinomial).

    Thus,  $\mathbb{E}[A_t|N_{t-1}] = r\left(\frac{1}{2}-\beta\right) N_{t-1}$ and $\mathbb{E}[D_t|N_{t-1}]=N_{t-1}q\left(\frac{1}{2} + \beta\right)$. Taking the conditional expectation of the balance equation yields, exactly:
    \begin{align*} 
    \mathbb{E}[N_t|N_{t-1}] &= N_{t-1} - N_{t-1}q\left(\frac{1}{2}+\beta\right) + r \left(\frac{1}{2}-\beta\right) N_{t-1} \nonumber\\ 
    &= N_{t-1} \left(1-q\left(\frac{1}{2}+\beta\right) + r\left(\frac{1}{2}-\beta\right)\right)
\end{align*}
\end{proof}

\begin{reptheorem}{theo:min_beta_threshold}
To maintain $C(\beta) \ge 1$, the parameter $\beta$ must satisfy the explicit upper bound $\beta \le \beta_c$, where the critical threshold $\beta_c$ is defined as:
\begin{equation*}
    \beta_c \triangleq \max \left\{ 0,\frac{r-q}{2(q+r)}\right\}
\end{equation*}
For $r>q$, this is equivalent to $\eps\le\eps_c\triangleq\ln(r/q)$.
\end{reptheorem}

\begin{proof}
Substituting the simplified linear expression for the performative growth factor yields:
\begin{align*}
    &1 + \frac{r-q}{2} - \beta(q+r) \ge 1 \\
    &\implies \frac{r-q}{2} - \beta(q+r) \ge 0
\end{align*}
Rearranging the inequality to isolate the $\beta$ term yields:
\begin{equation*}
    \beta(q+r) \le \frac{r-q}{2}
\end{equation*}
Dividing both sides by $(q+r)$,
\begin{equation*}
    \beta \le \frac{r-q}{2(q+r)} 
\end{equation*}
However, when $r \le q$, the critical point get clipped to zero. Thus,
\begin{align*}
    \beta \le \max \left\{ 0,\frac{r-q}{2(r+q)}\right\}
\end{align*}
Finally, for $r>q$, since $2\beta=\tanh(\eps/2)=\frac{e^\eps-1}{e^\eps+1}$ is increasing in $\eps$, $\beta\le\beta_c$ is equivalent to $(e^\eps-1)(r+q)\le(r-q)(e^\eps+1)$, i.e., $q e^\eps\le r$, i.e., $\eps\le\ln(r/q)$.
This completes the proof.
\end{proof}

\subsection{Galton--Watson Analysis of the Population}\label{app:gw_rr}

Lemma~\ref{lemm:unleaked_growth} and \Cref{theo:min_beta_threshold} only describe the expected population. Since leakage, departures, and recruitments are independent across participants and rounds, $(N_t)$ is in fact exactly a Galton--Watson branching process \citep{athreya1972branching}, which yields an exact description of the random population.

\begin{theorem}[Extinction and Growth]\label{theo:gw_rr}
Let $q,r>0$, $\eps_c=\ln(r/q)$, and $\sigma^2(\beta)\triangleq q(\frac12+\beta)+r(\frac12-\beta)-(C(\beta)-1)^2$ the offspring variance. Starting from $N_0$ participants, the population goes extinct with probability
\begin{equation*}
    \Pr(\exists t,\ N_t=0) = \min\Big(1,\frac{q}{r}e^{\eps}\Big)^{N_0} = e^{-N_0(\eps_c-\eps)_+}.
\end{equation*}
If $\eps<\eps_c$, $N_t/C(\beta)^t$ converges almost surely to a random variable $W$, positive exactly on survival, with $\mathbb{E}[W]=N_0$ and $\mathrm{Var}(W)=N_0\sigma^2(\beta)/\big(C(\beta)(C(\beta)-1)\big)$.
\end{theorem}

In particular, if $q\ge r$, then $\eps_c\le0$ and the population goes extinct almost surely for every $\eps$. If $r>q$, extinction is only likely within $\bigO(1/N_0)$ of $\eps_c$, and the deterministic trajectory $N_0C(\beta)^t$ is accurate up to the random factor $W/N_0$, whose relative standard deviation is small unless $C(\beta)-1=\bigO(1/N_0)$.

\begin{proof}
\textbf{Branching structure.} Let $a\triangleq q(\frac12+\beta)$ and $c\triangleq r(\frac12-\beta)$. Conditionally on the past, each of the $N_t$ participants of round $t$ independently produces $\xi\in\{0,1,2\}$ participants for round $t+1$: $\xi=0$ if their answer is leaked and they leave (probability $a$), $\xi=2$ if it is not leaked and they recruit (probability $c$), and $\xi=1$ otherwise. Hence $(N_t)$ is a Galton--Watson process \citep{athreya1972branching} with offspring generating function $f(s)=a+(1-a-c)s+cs^2$, offspring mean $f'(1)=1-a+c=C(\beta)$, and offspring variance $f''(1)+f'(1)-f'(1)^2=2c+(1-a+c)-(1-a+c)^2=a+c-(c-a)^2=\sigma^2(\beta)$. By Lemma~\ref{lemm:eps_ldp_beta}, $\frac{1/2+\beta}{1/2-\beta}=e^\eps$, so that $\frac{a}{c}=\frac{q}{r}e^{\eps}=e^{\eps-\eps_c}$.

\textbf{Extinction.} Since $a>0$, the offspring distribution is not concentrated on $1$, so the extinction probability of the process started from a single participant is the smallest root in $[0,1]$ of $f(s)=s$ \citep{athreya1972branching}. Here, $f(s)-s=cs^2-(a+c)s+a=(s-1)(cs-a)$, whose roots are $1$ and $a/c$, so this probability is $\min(1,a/c)=\min(1,e^{\eps-\eps_c})$. The $N_0$ initial participants generate independent Galton--Watson processes whose sum is $(N_t)$, and the population goes extinct if and only if all of them do, which happens with probability $\min(1,e^{\eps-\eps_c})^{N_0}=e^{-N_0(\eps_c-\eps)_+}$.

\textbf{Growth.} Assume $\eps<\eps_c$, i.e., $a<c$ and $C\triangleq C(\beta)>1$, and let $M_t\triangleq N_t/C^t$. Since $\mathbb{E}[N_{t+1}\mid N_t]=CN_t$, $(M_t)$ is a nonnegative martingale with $\mathbb{E}[M_t]=N_0$. By the law of total variance, $\mathrm{Var}(N_{t+1})=\sigma^2\,\mathbb{E}[N_t]+C^2\,\mathrm{Var}(N_t)=\sigma^2N_0C^t+C^2\mathrm{Var}(N_t)$, so that $\mathrm{Var}(M_{t+1})=\mathrm{Var}(M_t)+\sigma^2N_0C^{-t-2}$ and, as $\mathrm{Var}(M_0)=0$,
\begin{equation*}
    \mathrm{Var}(M_t)=\sigma^2N_0\sum_{k=0}^{t-1}C^{-k-2}=\frac{N_0\sigma^2\left(1-C^{-t}\right)}{C(C-1)}.
\end{equation*}
The martingale $(M_t)$ is thus bounded in $L^2$, so it converges almost surely and in $L^2$ to a random variable $W$ with $\mathbb{E}[W]=N_0$ and $\mathrm{Var}(W)=N_0\sigma^2/(C(C-1))$. Finally, since the offspring number is bounded, $\mathbb{E}[\xi\log\xi]<\infty$, and the Kesten-Stigum theorem \citep{kesten1966limit} gives $\Pr(W=0)=\Pr(\text{extinction})$.
\end{proof}

\begin{remark}[Beyond randomized response]\label{rem:gw_no_transfer}
In \Cref{section:membership_inference,sec:regression_perf_dp}, conditionally on the release noise $\eta_t$, each round is still approximately a Galton--Watson generation, with $\frac12+\beta$ replaced by a leakage probability depending on $\eta_t$ and $N_t$, but two features break the Galton--Watson structure. First, $\eta_t$ is shared by all users, so the fraction of leaked users fluctuates across rounds without averaging out over users: $(N_t)$ is a branching process in a random environment, whose almost sure extinction is characterized, for i.i.d.\ environments, by the expected logarithm of the random growth factor rather than by its mean \citep{smith1969branching}. Second, the leakage probability decreases with $N_t$ (proof of Lemma~\ref{lemm:C_gamma_increasing}), so the growth factor increases with the population size, a positive dependence known in population ecology as an Allee effect \citep{allee1931animal,courchamp1999inverse}. This is why we rely on the mean-field approximation in these sections.
\end{remark}
\section{Proofs from Section \ref{section:membership_inference}} \label{app:gaussian_mia_proofs}

\subsection{Generic Results}

\begin{lemma}\label{lemm:Mahalanobis_error}
Let $Z_t$ be a dataset drawn from a distribution with true mean $\mu$ and covariance matrix $\Sigma$,  and neglect the clipping bias (i.e., assume $\bar z_i = z_i$). Assuming the aggregated noise vector $\eta_d$ has covariance $C_{\bm\gamma} = \frac{\gamma^2}{N_t} I$, the expected squared Mahalanobis error is:
\begin{equation}
    \mathbb{E} \left[ \|o_t - \mu_t\|_{\Sigma^{-1}}^2 \;\middle|\; N_t \right] = \frac{\gamma^2 \mathrm{Tr}(\Sigma^{-1}) + d}{N_t}
\end{equation}
\end{lemma}

\begin{proof}
Since the mechanism's injected noise $\eta_d$ is sampled independently from the underlying data distribution $\mathcal{D}$, the expected squared error decouples into two components: the variance introduced by the privacy mechanism and the intrinsic geometric variance of the data sample itself. 

By applying the linearity of expectation and utilizing the trace trick for quadratic forms ($\mathbb{E}[x^T A x] = \mathrm{Tr}(A \Sigma_x) + \mu_x^T A \mu_x$), we expand the conditional error:
\begin{align*}
    \mathbb{E} \left[ \|o_t - \mu_t\|_{\Sigma^{-1}}^2 \;\middle|\; N_t \right] 
    &=  \mathbb{E}_{\eta_d} \left[ \|\eta_d\|_{\Sigma^{-1}}^2 \right] + \frac{1}{N_t^2} \mathbb{E}_{\mathcal{D}} \left[ \Big\| \sum_{z_i \in Z_t} (z_i - \mu) \Big\|_{\Sigma^{-1}}^2 \right] \\
    &= \mathrm{Tr}(\Sigma^{-1} C_{\bm\gamma}) + \frac{1}{N_t^2} \sum_{z_i \in Z_t} \mathbb{E}_{\mathcal{D}} \left[ \|z_i - \mu_t\|_{\Sigma^{-1}}^2 \right] \\
    &= \frac{\gamma^2 \mathrm{Tr}(\Sigma^{-1})}{N_t} + \frac{1}{N_t^2} \sum_{z_i \in Z_t} \mathrm{Tr}(\Sigma^{-1}\Sigma) \\
    &= \frac{\gamma^2 \mathrm{Tr}(\Sigma^{-1}) + d}{N_t}
\end{align*}
where we substitute $C_{\bm\gamma} = \frac{\gamma^2}{N_t} I$ and recognize that $\mathrm{Tr}(\Sigma^{-1}\Sigma) = \mathrm{Tr}(I_d) = d$. This bounds the conditional expected error, completing the proof.
\end{proof}

\begin{replemma}{lemm:Renyi_bound}
Let $Z_{t-1}$ be the dataset at round $t$ and $z \in Z_{t}$ be a target record. Let $o_t \sim \mathcal{M}(Z_{t})$ be the candidate output. Define the factual and counterfactual output distributions as $P^{\mathcal{M}}_{\text{in}}(\cdot) = \Pr[\mathcal{M}(Z_{t}) = \cdot \mid z \in Z_{t}]$ and $P^{\mathcal{M}}_{\text{out}}(\cdot) = \Pr[\mathcal{M}(Z_{t} \setminus \{z\}) = \cdot]$. 

For the log-likelihood ratio $\Lambda(o_t|z) = \ln \left( \frac{P^{\mathcal{M}}_{\text{in}}(o_t)}{P^{\mathcal{M}}_{\text{out}}(o_t)} \right)$ evaluated at output $o_t$, and for any threshold $\tau > 0$ and order $\alpha > 1$, the tail probability of structural compromisation is bounded by:
\begin{align*}
\Pr_{o_t \sim P^{\mathcal{M}}_{\text{in}}}(\Lambda(o_t|z) > \tau) \le \exp \Big( -(\alpha - 1) \big( \tau - D_\alpha(P^{\mathcal{M}}_{\text{in}} \parallel P^{\mathcal{M}}_{\text{out}}) \big) \Big) 
\end{align*}
where $D_\alpha(P^{\mathcal{M}}_{\text{in}} \parallel P^{\mathcal{M}}_{\text{out}})$ is the Rényi divergence of order $\alpha$.
\end{replemma}

\begin{proof}
We use the Cramer-Chernoff technique to bound the leakage probability of $z$. By the monotonicity of the exponential function, for any $\alpha > 1$, the tail event $\Lambda(o_t|z) > \tau$ is strictly equivalent to:
$$ e^{(\alpha - 1)\Lambda(o_t|z)} > e^{(\alpha - 1)\tau} $$

Applying Markov's inequality to the non-negative random variable $e^{(\alpha - 1)\Lambda(o_t|z)}$ over the factual distribution $P^{\mathcal{M}}_{\text{in}}$ yields:
\begin{align}\label{eq:Markov_ineq}
 \Pr_{o_t \sim P^{\mathcal{M}}_{\text{in}}}(\Lambda(o_t|z) > \tau) \le \frac{\expectation_{o_t \sim P^{\mathcal{M}}_{\text{in}}} \left[ e^{(\alpha - 1)\Lambda(o_t|z)} \right]}{e^{(\alpha - 1)\tau}}
\end{align}

Expanding the definition of $\Lambda(o_t|z)$ within the expectation, we have:
$$ \expectation_{o_t \sim P^{\mathcal{M}}_{\text{in}}} \left[ e^{(\alpha - 1) \ln \left( \frac{P^{\mathcal{M}}_{\text{in}}(o_t)}{P^{\mathcal{M}}_{\text{out}}(o_t)} \right)} \right] = \expectation_{o_t \sim P^{\mathcal{M}}_{\text{in}}} \left[ \left( \frac{P^{\mathcal{M}}_{\text{in}}(o_t)}{P^{\mathcal{M}}_{\text{out}}(o_t)} \right)^{\alpha - 1} \right] $$

By the definition of R\'enyi divergence, $D_\alpha(P^{\mathcal{M}}_{\text{in}} \parallel P^{\mathcal{M}}_{\text{out}}) = \frac{1}{\alpha - 1} \ln \expectation_{P^{\mathcal{M}}_{\text{in}}} \left[ \left( \frac{P^{\mathcal{M}}_{\text{in}}(o_t)}{P^{\mathcal{M}}_{\text{out}}(o_t)} \right)^{\alpha - 1} \right]$. Exponentiating both sides to isolate the expectation yields:
$$ \expectation_{o_t \sim P^{\mathcal{M}}_{\text{in}}} \left[ \left( \frac{P^{\mathcal{M}}_{\text{in}}(o_t)}{P^{\mathcal{M}}_{\text{out}}(o_t)} \right)^{\alpha - 1} \right] = \exp \Big( (\alpha - 1) D_\alpha(P^{\mathcal{M}}_{\text{in}} \parallel P^{\mathcal{M}}_{\text{out}}) \Big) $$

Substituting this equivalence back into the numerator of \eqref{eq:Markov_ineq} gives:
$$ \Pr_{o_t \sim P^{\mathcal{M}}_{\text{in}}}(\Lambda(o_t|z) > \tau) \le \frac{\exp \Big( (\alpha - 1) D_\alpha(P^{\mathcal{M}}_{\text{in}} \parallel P^{\mathcal{M}}_{\text{out}}) \Big)}{\exp \Big( (\alpha - 1)\tau \Big)} $$

Combining the exponents completes the proof:
$$ \Pr_{o_t \sim P^{\mathcal{M}}_{\text{in}}}(\Lambda(o_t|z) > \tau) \le \exp \Big( -(\alpha - 1) \big( \tau - D_\alpha(P^{\mathcal{M}}_{\text{in}} \parallel P^{\mathcal{M}}_{\text{out}}) \big) \Big) $$
\end{proof}

\subsection{Results for Gaussian Mechanism}

\begin{lemma}\label{lemm:tail_bound_gaussian_mech}
Let $\mathcal{M}$ be a Gaussian mechanism for mean estimation over a dataset $Z_{t}$ of size $N_t$, which adds noise scaled to the population such that the variance is $\frac{\gamma^2}{N_t} I$. Let $P^{\mathcal{M}}_{\text{in}}(\cdot\mid Z_t)$ and $P^{\mathcal{M}}_{\text{out}}(\cdot\mid Z_t)$ denote the conditional distributions of the mechanism's output on adjacent datasets $Z_{t}$ and $Z_{t} \setminus \{z\}$, respectively. For any order $\alpha > 1$, the Rényi divergence between these distributions is given by:
\begin{align*}
D_\alpha(P^{\mathcal{M}}_{\text{in}}(\cdot\mid Z_t) || P^{\mathcal{M}}_{\text{out}}(\cdot \mid Z_t))
&=
\frac{\alpha N_t}{2\gamma^2 (N_t-1)^2}
\|z-\mu_{z,t}\|_2^2
\end{align*}
\end{lemma}

\begin{proof}
Let the output distributions of the mechanism on the adjacent datasets be defined as $P^{\mathcal{M}}_{\text{in}}(\cdot\mid Z_t) = \mathcal{N}\left(\mu_{z,t}, C_{\gamma}\right)$ and $P^{\mathcal{M}}_{\text{out}}(\cdot\mid Z_t) = \mathcal{N}\left(\mu_{-z}, C_{\gamma}\right)$, with shared covariance matrix $C_\gamma = \frac{\gamma^2}{N_t}I_{d \times d}$.

Assuming the mechanism computes the mean using the respective dataset sizes, the true, un-noised mean vectors and their difference are:
\begin{align*}
\mu_{z,t} = \frac{1}{N_t} \sum_{x \in Z_{t}} x, \quad
\mu_{-z} = \frac{1}{N_t-1} \sum_{x \in Z_{t} \setminus \{z\}} x, \quad
\mu_{z,t} - \mu_{-z}
= \frac{z-\mu_{z,t}}{N_t-1}
\end{align*}

For two multivariate Gaussian distributions sharing the same covariance matrix $C_\gamma$, the Rényi divergence of order $\alpha$ simplifies to the scaled Mahalanobis distance between their means. Substituting the known difference and the inverse covariance matrix $C_\gamma^{-1} = \frac{N_t}{\gamma^2} I$, we obtain:
\begin{align*}
D_\alpha(P^{\mathcal{M}}_{\text{in}}(\cdot\mid Z_t) || P^{\mathcal{M}}_{\text{out}}(\cdot\mid Z_t))
&=
\frac{\alpha}{2} (\mu_{z,t} - \mu_{-z})^\top C_{\gamma}^{-1} (\mu_{z,t} - \mu_{-z}) \\
&=
\frac{\alpha}{2}
\left( \frac{z-\mu_{z,t}}{N_t-1} \right)^\top
\left( \frac{N_t}{\gamma^2} I \right)
\left( \frac{z-\mu_{z,t}}{N_t-1} \right) \\
&=
\frac{\alpha}{2} \cdot
\frac{N_t}{\gamma^2} \cdot
\frac{1}{(N_t-1)^2}
(z-\mu_{z,t})^T I (z-\mu_{z,t}) \\
&=
\frac{\alpha N_t}{2\gamma^2 (N_t-1)^2}
\|z-\mu_{z,t}\|_2^2
\end{align*}
\end{proof}

\begin{lemma}\label{lemm:log_likelihood_distribution}
Let a mechanism output $o_t$ be perturbed by additive isotropic Gaussian noise $\eta_t \sim \mathcal{N}(0, \frac{\gamma^2}{N_t} I_{d \times d})$, representing the noise. The log-likelihood ratio $\Lambda(o_t|z)$ conditioned on the dataset $Z_t$ is given by:
\begin{equation*}
    \Lambda(o_t|z)  \mid Z_t
    =
    \frac{N_t\|z-\mu_{z,t}\|^2}{2(N_t-1)^2\gamma^2}
    +
    \frac{N_t\langle z-\mu_{z,t}, \eta_t\rangle}{(N_t-1)\gamma^2}
\end{equation*}
And it is distributed as:
\begin{equation*}
    \Lambda(o_t|z) \mid Z_t
    \sim
    \mathcal{N}\left(
    \frac{N_t\|z-\mu_{z,t}\|^2}{2(N_t-1)^2\gamma^2},
    \;
    \frac{N_t\|z-\mu_{z,t}\|^2}{(N_t-1)^2\gamma^2}
    \right)
\end{equation*}
\end{lemma}

\begin{proof}
By definition, the log-likelihood ratio compares the probability density of observing the noisy output $o_t$ when $z$ is present (with mean $\mu_{z,t}$) versus when $z$ is absent (with mean $\mu_{-z}$). Because the aggregated additive noise $\eta_t$ is Gaussian with distribution $\mathcal{N}(0, \frac{\gamma^2}{N_t} I)$, the log-likelihood conditioned on $Z_t$ evaluates to:
\begin{align*}
    \Lambda(o_t|z) \mid Z_t
    &=
    \ln \left(
    \frac{\exp \left( -\frac{N_t}{2\gamma^2} \|o_t - \mu_{z,t}\|^2 \right)}
    {\exp \left( -\frac{N_t}{2\gamma^2} \|o_t - \mu_{-z}\|^2 \right)}
    \right) \\
    &=
    \frac{N_t}{2\gamma^2}
    \Big(
    \|o_t - \mu_{-z}\|^2
    -
    \|o_t - \mu_{z,t}\|^2
    \Big)
\end{align*}
Substituting the noisy observation $o_t = \mu_{z,t} + \eta_t$ and the shift
$\mu_{-z} = \mu_{z,t} - \frac{z-\mu_{z,t}}{N_t-1}$:
\begin{align*}
    \Lambda(o_t|z) \mid Z_t
    &=
    \frac{N_t}{2\gamma^2}
    \Bigg(
    \left\|
    \mu_{z,t} + \eta_t
    -
    \left(
    \mu_{z,t} - \frac{z-\mu_{z,t}}{N_t-1}
    \right)
    \right\|^2
    -
    \|\mu_{z,t} + \eta_t - \mu_{z,t}\|^2
    \Bigg) \\
    &=
    \frac{N_t}{2\gamma^2}
    \left(
    \left\|
    \frac{z-\mu_{z,t}}{N_t-1} + \eta_t
    \right\|^2
    -
    \|\eta_t\|^2
    \right) \\
    &=
    \frac{N_t}{2\gamma^2}
    \left(
    \frac{\|z-\mu_{z,t}\|^2}{(N_t-1)^2}
    +
    \frac{2}{N_t-1}
    \langle z-\mu_{z,t}, \eta_t\rangle
    +
    \|\eta_t\|^2
    -
    \|\eta_t\|^2
    \right) \\
    &=
    \frac{N_t\|z-\mu_{z,t}\|^2}{2(N_t-1)^2\gamma^2}
    +
    \frac{N_t\langle z-\mu_{z,t},\eta_t\rangle}
    {(N_t-1)\gamma^2}
\end{align*}
Conditioned on $Z_t$, the term
$\frac{N_t\|z-\mu_{z,t}\|^2}{2(N_t-1)^2\gamma^2}$ is a deterministic scalar constant. The sole source of randomness is the inner product with the aggregated noise vector $\eta_t \sim \mathcal{N}(0, \frac{\gamma^2}{N_t} I)$. This linear combination
$\langle z-\mu_{z,t}, \eta_t\rangle$ forms a univariate Gaussian with mean $0$ and variance
\[
\frac{\gamma^2}{N_t}\|z-\mu_{z,t}\|^2.
\]
Scaling this by $\frac{N_t}{(N_t-1)\gamma^2}$, the conditional variance reduces to:
\begin{equation*}
    \variance\left(
    \frac{N_t\langle z-\mu_{z,t},\eta_t\rangle}
    {(N_t-1)\gamma^2}
    \Bigg| Z_t
    \right)
    =
    \frac{N_t^2}{(N_t-1)^2\gamma^4}
    \left(
    \frac{\gamma^2}{N_t}\|z-\mu_{z,t}\|^2
    \right)
    =
    \frac{N_t\|z-\mu_{z,t}\|^2}
    {(N_t-1)^2\gamma^2}
\end{equation*}
Thus,
\begin{equation*}
    \Lambda(o_t|z) \mid Z_t
    \sim
    \mathcal{N}\left(
    \frac{N_t\|z-\mu_{z,t}\|^2}{2(N_t-1)^2\gamma^2},
    \frac{N_t\|z-\mu_{z,t}\|^2}{(N_t-1)^2\gamma^2}
    \right).
\end{equation*}
This completes the proof.
\end{proof}

\begin{replemma}{lemm:conditional_leakage_prob}
Given the aggregated empirical mean mechanism, the exact probability that a specific individual data point $z$ is exposed (i.e., its log-likelihood ratio breaches the threshold $\tau$) conditioned on $Z_t$ is given in closed form by:
\begin{equation*}
    \Pr(\Lambda(o_t|z) > \tau \mid Z_t)
    =
    1 - \Phi\left(
    \frac{\tau (N_t-1)\gamma}
    {\sqrt{N_t}\|z-\mu_{z,t}\|}
    -
    \frac{\sqrt{N_t}\|z-\mu_{z,t}\|}
    {2(N_t-1)\gamma}
    \right)
\end{equation*}
where $\Phi(\cdot)$ denotes the cumulative distribution function of the standard normal distribution.
\end{replemma}

\begin{proof}
From lemma \ref{lemm:log_likelihood_distribution}, the conditional log-likelihood ratio for a specific data point $z$ and $\mu_{z,t}$ strictly follows a Gaussian distribution:
\begin{equation}\label{eq:in_out_params}
    \Lambda(o_t|z)\mid Z_t
    \sim
    \mathcal{N}\left( \mu_{z,t}', \sigma_z^2 \right)
    \quad \text{where} \quad
    \mu_{z,t}'
    =
    \frac{N_t\|z-\mu_{z,t}\|^2}
    {2(N_t-1)^2\gamma^2},
    \;\;
    \sigma_z
    =
    \frac{\sqrt{N_t}\|z-\mu_{z,t}\|}
    {(N_t-1)\gamma}
\end{equation}
The probability that $\Lambda(o_t|z)$ exceeds $\tau$ is precisely the tail probability of this Gaussian distribution. We evaluate this using the standard normal cumulative distribution function $\Phi$:
\begin{align*}
    \Pr(\Lambda(o_t|z) > \tau \mid Z_t)
    &=
    \Pr\left(
    \frac{\Lambda(o_t|z) - \mu_{z,t}'}{\sigma_z}
    >
    \frac{\tau - \mu_{z,t}'}{\sigma_z}
    \Big| Z_t\right) \\
    &=
    1 - \Phi\left(
    \frac{\tau - \mu_{z,t}'}{\sigma_z}
    \right)
\end{align*}
Substituting the explicit formulae for $\mu_{z,t}'$ and $\sigma_z$ from \eqref{eq:in_out_params}:
\begin{align*}
    \Pr(\Lambda(o_t|z) > \tau \mid Z_t)
    &=
    1 - \Phi\left(
    \frac{
    \tau -
    \frac{N_t\|z-\mu_{z,t}\|^2}
    {2(N_t-1)^2\gamma^2}
    }
    {
    \frac{\sqrt{N_t}\|z-\mu_{z,t}\|}
    {(N_t-1)\gamma}
    }
    \right) \\
    &=
    1 - \Phi\left(
    \frac{\tau (N_t-1)\gamma}
    {\sqrt{N_t}\|z-\mu_{z,t}\|}
    -
    \frac{
    N_t\|z-\mu_{z,t}\|^2
    (N_t-1)\gamma
    }
    {
    2(N_t-1)^2\gamma^2\sqrt{N_t}\|z-\mu_{z,t}\|
    }
    \right) \\
    &=
    1 - \Phi\left(
    \frac{\tau (N_t-1)\gamma}
    {\sqrt{N_t}\|z-\mu_{z,t}\|}
    -
    \frac{\sqrt{N_t}\|z-\mu_{z,t}\|}
    {2(N_t-1)\gamma}
    \right)
\end{align*}
This completes the proof. 
\end{proof}


\subsection{Comparison of the Generic Bound with the Exact Gaussian Leakage}\label{app:gaussiangeneric}

The bound of \Cref{lemm:Renyi_bound} applies to the output of any DP mechanism, as it only depends on the R\'enyi divergence. It decays exponentially with the threshold $\tau$, and its tightest version is obtained by optimizing over the order $\alpha$, which must remain strictly greater than $1$. We illustrate it on the Gaussian mechanism, conditioning the leakage probability on $Z_t$ for analytical tractability, and compare it with the exact expression of \Cref{lemm:conditional_leakage_prob}.

\begin{lemma}[Upper Bound for the Gaussian Mechanism]\label{lemm:reyni_bound_gaussian}
Substituting into \Cref{lemm:Renyi_bound} the R\'enyi divergence of the Gaussian mechanism \eqref{eq:mecgau} between the outputs with and without the record $z$ conditioned on $Z_t$ given by \Cref{lemm:tail_bound_gaussian_mech}, the tail bound on the leakage probability of a record $z$ specializes to:
\begin{equation}\label{eq:reyni_bound_gaussian}
    \Pr_{o_t \sim P^{\mathcal{M}}_{\text{in}}}(\Lambda(o_t|z) > \tau \mid Z_t)
    \le
    \exp \left(
    -(\alpha - 1)
    \left(
    \tau -
    \frac{\alpha N_t\|z-\mu_{z,t}\|^2}
    {2\gamma^2 (N_t-1)^2}
    \right)
    \right).
\end{equation}
\end{lemma}
This bound makes explicit, first, the dependence on $\|z-\mu_{z,t}\|^2$: outliers face an exponentially increasing probability of leakage as their deviation from the empirical mean grows. Second, when the population size $N_t$ increases, the divergence decreases, capturing the privacy amplification afforded by aggregation and leading to a virtuous circle: if the system is in a regime private enough to gain users, it becomes more likely to keep growing as the leakage risk decreases.

Minimizing \eqref{eq:reyni_bound_gaussian} with respect to $\alpha$ yields the optimal order
\begin{equation*}
\alpha^*
=
\frac{1}{2}
+
\frac{\tau\gamma^2 (N_t-1)^2}
{N_t\|z-\mu_{z,t}\|^2}
\qquad \text{subject to }\alpha^*>1.
\end{equation*}
Thus, the optimized bound is valid and non-vacuous only when
\[
\tau > a \triangleq
\frac{N_t\|z-\mu_{z,t}\|^2}
{2\gamma^2 (N_t-1)^2},
\]
i.e., only when the individual's privacy tolerance strictly exceeds their deviation from the empirical mean rescaled by the current parameters. This condition fails precisely for extreme outliers and in low-noise regimes ($\gamma \to 0$), where $\alpha^* \le 1$ either invalidates the divergence or trivializes the bound (i.e., it becomes larger than $1$). For $\tau > a$, substituting $\alpha^*$ back into \eqref{eq:reyni_bound_gaussian} yields the optimized bound
\begin{equation}\label{eq:optimized_renyi}
    \Pr_{o_t \sim P^{\mathcal{M}}_{\text{in}}}(\Lambda(o_t|z) > \tau\mid Z_t)
    \le
    \exp\left(
    -\frac{(\tau-a)^2}{4a}
    \right).
\end{equation}
Since \Cref{lemm:log_likelihood_distribution} shows that $\Lambda(o_t|z) \mid Z_t \sim \mathcal{N}(a,2a)$ exactly, the standard Gaussian tail bound $1-\Phi(u) \le e^{-u^2/2}$ applied at $u = (\tau-a)/\sqrt{2a}$ gives the same exponent: the optimized R\'enyi bound is tight up to the $\Phi$-versus-$\exp$ gap. Conversely, unlike the generic bound, which is conservative and subject to the constraint $\tau > a$, \Cref{lemm:conditional_leakage_prob} characterizes the leakage probability exactly, including for outliers.

\subsection{Population Evolution and Critical Noise Threshold}

\begin{assumption} \label{assump:normal}
The records are drawn independently and identically at each round $t$ from a $d$-dimensional multivariate Gaussian distribution $\mathcal{N}(\mu,\Sigma)$, with probability density function
\begin{equation*}
    f_{\mathcal{N}}(z)
    = \frac{1}{\sqrt{(2\pi)^d |\Sigma|}}
    \exp\left(
    -\frac{1}{2}(z-\mu)^T\Sigma^{-1}(z-\mu)
    \right).
\end{equation*}
\end{assumption}

\begin{replemma}{lemm:gaussian_mia_poplin_evolution}
Given assumption \ref{assump:normal} and the performative update rule in mechanism \ref{alg:memoryless_dp}, the expected population size at round $t+1$, conditioned on the current population $N_t$, is given by:
\begin{equation*}
    \expectation[N_{t+1} \mid N_t] = N_t \Big( 1 + r - (q + r) p^{(t)}_\tau \Big)
\end{equation*}
\end{replemma}

\begin{proof}
Let $\mathds{1}(\Lambda(o_t| z_i) > \tau)$ be the indicator variable denoting whether the $i$-th data point experiences a privacy leakage exceeding the threshold $\tau$ at round $t$. Averaging the independent departure and recruitment coins, conditional on the dataset and released output, gives:
\begin{equation*}
    \mathbb E[N_{t+1}\mid Z_t,o_t] = N_t - q \sum_{i=1}^{N_t} \mathds{1}\big(\Lambda(o_t| z_i) > \tau\big) + r \left( N_t - \sum_{i=1}^{N_t} \mathds{1}\big(\Lambda(o_t| z_i) > \tau\big) \right)
\end{equation*}
Where $z_i \in Z_t$. Grouping the terms yields:
\begin{equation*}
    \mathbb E[N_{t+1}\mid Z_t,o_t] = N_t(1 + r) - (q + r) \sum_{i=1}^{N_t} \mathds{1}\big(\Lambda(o_t| z_i) > \tau\big)
\end{equation*}
Taking the conditional expectation given $N_t$, we evaluate the joint expectation over the entire dataset $Z_t = \{z_1, \dots, z_{N_t}\}$ and the mechanism's noise $\eta_{t}$:
\begin{align*}
    \expectation[N_{t+1} \mid N_t] &= N_t(1 + r) - (q + r) \expectation_{Z_t, \eta_{t}} \left[ \sum_{i=1}^{N_t} \mathds{1}\big(\Lambda(o_t| z_i) > \tau\big) \right]
\end{align*}
By Fubini's theorem, we distribute the expectation inside the sum. As established, $\Lambda(o_t|z_i)$ for a specific point depends on $z_i$, $\mu_{z,t}$, and the noise $\eta_{t}$. We evaluate the expectation by utilizing the tower property of expectation:
\begin{align*}
    \expectation_{Z_t, \eta_{t}} \left[ \sum_{i=1}^{N_t} \mathds{1}\big(\Lambda(o_t| z_i) > \tau\big) \right]
    &= \sum_{i=1}^{N_t} \expectation_{Z_t} \left[
    \Pr\big(\Lambda(o_t| z_i) > \tau \mid Z_t\big)
    \right]
\end{align*}
From lemma \ref{lemm:conditional_leakage_prob}, the conditional probability of leakage for any specific data point $z$ is given as:
\begin{equation*}
    \Pr(\Lambda(o_t|z) > \tau \mid Z_t)
    =
    1 - \Phi\left(
    \frac{\tau (N_t-1) \gamma}
    {\sqrt{N_t}\|z-\mu_{z,t}\|}
    -
    \frac{\sqrt{N_t}\|z-\mu_{z,t}\|}
    {2(N_t-1)\gamma}
    \right)
\end{equation*}
Since all $N_t$ data points are identically distributed, we proceed as follows:
\begin{align*}
    \expectation[N_{t+1} \mid N_t]
    &=
    N_t(1 + r)
    \\
    &~~~~-
    (q + r) N_t
    \expectation_{Z_t}
    \left[
    1 - \Phi\left(
    \frac{\tau (N_t-1)\gamma}
    {\sqrt{N_t}\|z-\mu_{z,t}\|}
    -
    \frac{\sqrt{N_t}\|z-\mu_{z,t}\|}
    {2(N_t-1)\gamma}
    \right)
    \right]
\end{align*}
Since mechanism $\mathcal{M}$ in algorithm \ref{alg:memoryless_dp} clips the samples to $\bar z$, we denote the empirical mean on the clipped samples as $\mu_{\bar z,t}$. Thus, the marginal leakage probability for $z$ is
\begin{equation*}
    p^{(t)}_\tau
    =
    \expectation_{Z_t}
    \left[
    1 - \Phi\left(
    \frac{\tau (N_t-1)\gamma}
    {\sqrt{N_t}\|\bar z-\mu_{\bar z,t}\|}
    -
    \frac{\sqrt{N_t}\|\bar z-\mu_{\bar z,t}\|}
    {2(N_t-1)\gamma}
    \right)
    \right].
\end{equation*}
Which expands as shown in \eqref{eq:p_tau_integral}. Therefore, our final expression is:
\begin{align*}
    \expectation[N_{t+1} \mid N_t]
    &= N_t(1 + r) - (q + r) N_t p^{(t)}_\tau \\
    &= N_t \Big( 1 + r - (q + r) p^{(t)}_\tau \Big)
\end{align*}
This completes the proof.
\end{proof}

\begin{reptheorem}{theo:noise_threshold_mia}
 The condition $C_t(\gamma) \ge 1$ is guaranteed as soon as the mechanism's noise scale satisfies the explicit lower bound $\gamma \ge \gamma_c$. For $q,r>0$, $N_t>1$ and clipping bound $R$, a sufficient threshold $\gamma_c$ is defined as:
\begin{equation*}
    \gamma_c \triangleq
    \frac{2R\sqrt{N_t}}{N_t-1}
    \left(
    \frac{\Phi_{q,r}^{-1}
    +\sqrt{(\Phi_{q,r}^{-1})^2+2\tau}}
    {2\tau}
    \right)
\end{equation*}
where $\Phi_{q,r}^{-1} \triangleq \Phi^{-1}\left( \frac{q}{q+r} \right)$ is the probit function (the inverse of the standard normal cumulative distribution function) evaluated at the performative survival threshold.
\end{reptheorem}

\begin{proof}
Nonnegative expected growth requires $C(\gamma) \ge 1$. Given the definition of the performative growth factor, this restricts the maximum permissible expected marginal leakage probability $p_\tau$:
\begin{equation*}
    C(\gamma) \ge 1 \implies 1 + r - (q + r)p_\tau \ge 1 \implies p_\tau \le \frac{r}{q+r}
\end{equation*}
Clipping gives the uniform bound
\begin{equation*}
    \|\bar{z}-\mu_{\bar z,t}\| \le 2R.
\end{equation*}
It is therefore sufficient to require:
\begin{equation*}
    1 - \Phi\left(
    \frac{\tau (N_t-1)\gamma}{2R\sqrt{N_t}}
    -
    \frac{R\sqrt{N_t}}{(N_t-1)\gamma}
    \right)
    \le \frac{r}{q+r}.
\end{equation*}

Rearranging to isolate the CDF and applying the Probit function ($\Phi^{-1}$) yields:
\begin{equation}\label{eq:probit_ineq}
    \frac{\tau (N_t-1)\gamma}{2R\sqrt{N_t}}
    -
    \frac{R\sqrt{N_t}}{(N_t-1)\gamma}
    \ge
    \Phi^{-1}\left(1-\frac{r}{q+r}\right)
    =
    \Phi^{-1}\left(\frac{q}{q+r}\right)
    \triangleq \Phi_{q,r}^{-1}.
\end{equation}

Let
\begin{equation*}
    u =
    \frac{(N_t-1)\gamma}{2R\sqrt{N_t}}.
\end{equation*}
Substituting $u$ transforms \eqref{eq:probit_ineq} into a standard quadratic equation:
\begin{equation*}
    \tau u - \frac{1}{2u} \ge \Phi_{q,r}^{-1}
    \implies
    2\tau u^2 - 2\Phi_{q,r}^{-1}u - 1 \ge 0.
\end{equation*}

Taking the valid positive root of the quadratic provides the sufficient lower bound for $u$:
\begin{equation}\label{eq:min_mia_noise_semifinal}
    u \ge
    \frac{2\Phi_{q,r}^{-1}
    +\sqrt{4(\Phi_{q,r}^{-1})^2+8\tau}}
    {4\tau}
    =
    \frac{\Phi_{q,r}^{-1}
    +\sqrt{(\Phi_{q,r}^{-1})^2+2\tau}}
    {2\tau}.
\end{equation}

Substituting $u = \frac{(N_t-1)\gamma}{2R\sqrt{N_t}}$ back into \eqref{eq:min_mia_noise_semifinal} and using $R=\Theta(\sqrt{d})$, we get:
\begin{align*}
    \gamma &\ge
    \frac{2R\sqrt{N_t}}{N_t-1}
    \left(
    \frac{\Phi_{q,r}^{-1}
    +\sqrt{(\Phi_{q,r}^{-1})^2+2\tau}}
    {2\tau}
    \right) \triangleq \gamma_c \\
    &= {\Theta}\left(
    \frac{\sqrt{d N_t}}{N_t-1}
    \left(
    \frac{\Phi_{q,r}^{-1}
    +\sqrt{(\Phi_{q,r}^{-1})^2+2\tau}}
    {2\tau}
    \right)
    \right)
\end{align*}
This establishes a sufficient noise level for nonnegative expected growth, completing the proof.
\end{proof}

\subsection{Bounds and Asymptotics of the Error}\label{app:remarkgaussian}

\begin{theorem}\label{theo:bounds_on_error}
Under our assumptions, conditioned on $N_{t-1}\ge2$, the expected squared Mahalanobis error of the private empirical mean $o_t$ is bounded from below and above by:
\begin{align}\label{eq:tight_bounds_mahalanobis}
    \frac{\gamma^2 \mathrm{Tr}(\Sigma^{-1}) + d}{C_{t-1}(\gamma)N_{t-1}+1}
    \le \expectation \left[ \|o_t - \mu_t\|_{\Sigma^{-1}}^2 \;\middle|\; N_{t-1} \right]
    \le \kappa_{q,r}\, \frac{\gamma^2 \mathrm{Tr}(\Sigma^{-1}) + d}{C_{t-1}(\gamma)N_{t-1}}
\end{align}
\begin{align*}
    \text{with } ~~~\kappa_{q,r} \triangleq \frac{2(1+r)}{1-q}  = \bigO\left(\frac{1}{1-q}\right), ~~~ \text{since } 0\le r\le1,\ 0\le q<1
\end{align*}
where the asymptotic notation in the upper bound is particularly useful in the regime $q \rightarrow 1$. This floored variant is used for the error bound; Algorithm~\ref{alg:memoryless_dp} instead stops when the population becomes too small.

\end{theorem}

\begin{proof}
Recall that the exact squared Mahalanobis error, conditioned strictly on a known population size $N_t$, isolates the population in the denominator:
\begin{equation*}
    \expectation \left[ \|o_t - \mu_t\|_{\Sigma^{-1}}^2 \;\middle|\; N_t \right]
    = \frac{\gamma^2 \mathrm{Tr}(\Sigma^{-1}) + d}{N_t}.
\end{equation*}
Taking the conditional expectation given $N_{t-1}$ gives
\begin{equation*}
    \expectation \left[ \|o_t - \mu_t\|_{\Sigma^{-1}}^2 \;\middle|\; N_{t-1} \right]
    =
    \Big( \gamma^2 \mathrm{Tr}(\Sigma^{-1}) + d \Big)
    \expectation \left[ \frac{1}{N_t} \;\middle|\; N_{t-1} \right].
\end{equation*}

\paragraph{Lower Bound.}
The floor increases the raw population by at most one, so $\expectation[N_t\mid N_{t-1}]\le C_{t-1}(\gamma)N_{t-1}+1$. We apply Jensen's inequality to the convex function $x \mapsto 1/x$ on $(0,\infty)$:
\begin{equation*}
    \expectation \left[ \frac{1}{N_t} \;\middle|\; N_{t-1} \right]
    \ge
    \frac{1}{\expectation[N_t \mid N_{t-1}]}
    \ge
    \frac{1}{C_{t-1}(\gamma)N_{t-1}+1}.
\end{equation*}

\paragraph{Upper Bound.}
Pre-draw an independent departure coin for each participant. The number $S$ whose coin prevents departure is distributed as $\operatorname{Bin}(N_{t-1},1-q)$. These participants remain regardless of leakage, so $N_t\ge\max\{1,S\}$. Since $1/\max\{1,s\}\le2/(s+1)$ for integers $s\ge0$,
\begin{equation*}
    \expectation\left[\frac{1}{N_t}\middle|N_{t-1}\right]
    \le 2\expectation\left[\frac{1}{S+1}\middle|N_{t-1}\right]
    = \frac{2(1-q^{N_{t-1}+1})}{(N_{t-1}+1)(1-q)}.
\end{equation*}
Using $C_{t-1}(\gamma)\le1+r$ directly yields:
\begin{equation*}
    \expectation\left[\frac{1}{N_t}\middle|N_{t-1}\right]
    \le \frac{2(1+r)}{1-q}\frac{1}{C_{t-1}(\gamma)N_{t-1}}
    = \frac{\kappa_{q,r}}{C_{t-1}(\gamma)N_{t-1}}.
\end{equation*}

Multiplying the lower and upper bounds on the expected inverse population by
$\Big(\gamma^2 \mathrm{Tr}(\Sigma^{-1})+d\Big)$ establishes \eqref{eq:tight_bounds_mahalanobis}.
\end{proof}

\begin{corollary}[Zero-Noise Collapse]\label{cor:zero_noise}
For any purge rate $0<q<1$, as the algorithmic noise vanishes ($\gamma \to 0^+$), the expected estimation error increases geometrically while the population remains large, under the mean field approximation of $N_t$.
\end{corollary}

\begin{proof}
As $\gamma \to 0^+$, the log-likelihood ratio diverges for any $z$ with $\|z-\mu_{z,t}\|>0$, driving the marginal leakage probability $p^{(t)}_\tau \to 1$ for all $t$. The demographic growth factor reduces to $C_t(\gamma) \to 1-q$ for all $t$, resulting in a mean-field population of $N_0(1-q)^{t-1}$. Using the single-round error formula, the expected estimation error is of order $\frac{d}{N_0(1-q)^{t-1}}$. For $q>0$, the denominator decreases geometrically, causing the error to increase while the population remains large. This approximation does not describe behavior after the population reaches a floor or stopping rule.
\end{proof}

\begin{lemma}\label{lemm:C_gamma_increasing}
For nondegenerate Gaussian data, for fixed $\gamma>0$, under mean-field approximation, if $C_t(\gamma) > 1$ for all $t$ (i.e., the expected population strictly grows), then $C_t(\gamma)$ is a strictly increasing sequence.
\end{lemma}

\begin{proof}
Under the mean-field approximation of $N_t$, the expected population size at step $t$ is given by $N_{t+1} = C_t(\gamma)N_t$. Since we are given that $C_t(\gamma) > 1$ for all $t$, the population strictly grows, meaning $N_{t+1} > N_t$ (and consequently its inverse shrinks: $1/N_{t+1} < 1/N_t$).

To justify the decrease in leakage, note that for a Gaussian sample of size $N$, $z_i-\mu_z\sim\mathcal N(0,(1-1/N)\Sigma)$. Writing $G\sim\mathcal N(0,\Sigma)$, the standardized shift in Lemma~\ref{lemm:conditional_leakage_prob} has distribution
\begin{equation*}
\frac{\sqrt N\|z_i-\mu_z\|}{\gamma(N-1)}\overset{d}{=}\frac{\|G\|}{\gamma\sqrt{N-1}}.
\end{equation*}
For each nonzero $G$, this decreases with $N$, and so does its leakage probability $1-\Phi(\tau\gamma\sqrt{N-1}/\|G\|-\|G\|/(2\gamma\sqrt{N-1}))$. Taking expectations preserves the decrease. Thus,
$ p^{(t)}_\tau > p^{(t+1)}_\tau$.

To determine the behavior of $C_t(\gamma)$, we evaluate the difference between consecutive terms:
\begin{equation*}
    C_{t+1}(\gamma) - C_t(\gamma) = \left[ 1 + r - (q+r)p^{(t+1)}_\tau \right] - \left[ 1 + r - (q+r)p^{(t)}_\tau \right]=
    (q+r) \left( p^{(t)}_\tau - p^{(t+1)}_\tau \right)
\end{equation*}

Since $q+r > 0$ and  $p^{(t)}_\tau - p^{(t+1)}_\tau > 0$, the product is positive $
    C_{t+1}(\gamma) - C_t(\gamma) > 0$, concluding the proof.
\end{proof}

\begin{corollary}[Infinite Horizon Asymptotics]\label{cor:infinite_horizon}
For any positive recruitment rate $r > 0$ and any fixed finite noise scale $\gamma>0$ such that $C_t(\gamma) > 1$ for all $t$, the expected estimation error at round $t$ converges to zero as $t \to \infty$, at the geometric rate $\bigO(C_1(\gamma)^{-(t-1)})$ under the mean field approximation of $N_t$.
\end{corollary}

\begin{proof}
For a fixed $\gamma$ with $C_t(\gamma) > 1$ for all $t$, using the single-round error formula under the mean-field approximation of $N_t$, $\mathcal{E}_t(\gamma)$ is at most
$\frac{\gamma^2 \mathrm{Tr}(\Sigma^{-1}) + d}{N_0 \prod_{k=1}^{t-1} C_k(\gamma)}$. Due to lemma \ref{lemm:C_gamma_increasing}, this quantity is upper bounded by $\frac{\gamma^2 \mathrm{Tr}(\Sigma^{-1}) + d}{N_0 C_1(\gamma)^{t-1}}$
and decays to zero as $t \to \infty$. As $\gamma \to \infty$, the marginal leakage probability vanishes ($p_\tau \to 0$), maximizing the growth factor to $C(\gamma) \to 1+r$ and the mean-field population toward $N_0(1+r)^{t-1}$: for any fixed finite noise level in the growth regime, exponential population expansion eventually dominates the fixed-in-$t$ mechanism variance term $\gamma^2$.

This conclusion continues to hold when $q$ is large (even as $q \to 1$) and $r>0$ is small but nonzero, provided $C(\gamma)>1$. The numerator $\gamma^2\mathrm{Tr}(\Sigma^{-1})+d$ remains fixed in $t$, while $C_1(\gamma)^{t-1}$ grows exponentially, so the expected estimation error still converges to zero.
\end{proof}

\begin{lemma}\label{lemm:cumulative_error}
For a finite horizon $T$, work under the deterministic mean-field approximation. Assume $C_1(\gamma)>1$. The cumulative expected Mahalanobis error aggregated across all $T$ rounds, defined as $\mathcal{E}_{\text{total}}(\gamma) \triangleq \sum_{t=1}^T \mathcal{E}_t(\gamma)$ is upper bounded by $\mathcal{E}_{\text{total}}(\gamma) \leq \hat{\mathcal{E}}_{\text{total}}(\gamma)$, where:
\begin{equation}
    \hat{\mathcal{E}}_{total}(\gamma) \triangleq
    \frac{C_1(\gamma)(\gamma^2 \mathrm{Tr}(\Sigma^{-1}) + d)}{N_0 (C_1(\gamma) - 1)} \left( 1 - \frac{1}{[C_1(\gamma)]^T} \right)
\end{equation}
\end{lemma}

\begin{proof}
To evaluate the true cost of privacy over a finite horizon $T$, we must aggregate the expected error across all rounds. The cumulative expected Mahalanobis error is the sum of the single-round errors:
\begin{equation*}
    \mathcal{E}_{total}(\gamma) = \sum_{t=1}^T \mathbb{E} \left[ \|o_t - \mu_t\|_{\Sigma^{-1}}^2 \right] =  \sum_{t=1}^T \frac{\gamma^2 \mathrm{Tr}(\Sigma^{-1}) + d}{N_0 \prod_{k=1}^{t-1} C_k(\gamma)} \overset{(a)}{\leq} \sum_{t=1}^T \frac{\gamma^2 \mathrm{Tr}(\Sigma^{-1}) + d}{N_0 [C_1(\gamma)]^{t-1}} 
\end{equation*}
Where (a) holds due to lemma \ref{lemm:C_gamma_increasing}. We apply the formula for sum of a finite geometric series:
\begin{align*}
    \hat{\mathcal{E}}_{total}(\gamma) &\triangleq \sum_{t=1}^T \frac{\gamma^2 \mathrm{Tr}(\Sigma^{-1}) + d}{N_0 [C_1(\gamma)]^{t-1}} \\
    &=\frac{\gamma^2 \mathrm{Tr}(\Sigma^{-1}) + d}{N_0} \left( \frac{1 - \frac{1}{[C_1(\gamma)]^T}}{1 - \frac{1}{C_1(\gamma)}} \right) \\
    &= \frac{C_1(\gamma)(\gamma^2 \mathrm{Tr}(\Sigma^{-1}) + d)}{N_0 (C_1(\gamma) - 1)} \left( 1 - \frac{1}{[C_1(\gamma)]^T} \right)
\end{align*}
This completes the proof.
\end{proof}

\begin{remark}[Optimization in the Growth Regime]
In the growth regime, where $C_1(\gamma)>1$, under the mean-field approximation, directly optimizing the upper bound $\hat{\mathcal{E}}_{\text{total}}(\gamma)$ rather than the exact cumulative error $\mathcal{E}_{\text{total}}(\gamma)$ is a practical surrogate objective. It captures the noise--growth trade-off, without requiring its minimizer to coincide with that of the exact cumulative error. Computationally, this substitution is advantageous as it depends solely on the initial expected growth factor $C_1(\gamma)$, bypassing the need to recursively track the sequence $C_t(\gamma)$ across all $T$ rounds. 
Furthermore, empirical evaluations confirm that optimizing this upper bound yields an effective configuration in practice.
\end{remark}

\subsection{High-Probability Population Lower Bound}

\begin{lemma}\label{lemm:single_step_concentration}
Let $Z_t$ be the dataset realized at round $t$. Let $\kappa_{t+1} \triangleq \mathbb{E}[N_{t+1} \mid N_t]$ denote the expected population at round $t+1$, integrating over both the mechanism's privacy noise and the individuals' stochastic performative decisions. For any relative deviation $\epsilon \in (0, 1)$, the probabilities that the realized population $N_{t+1}$ deviates below or above this conditional expectation decay exponentially:
\begin{flalign*}
    \text{(a)}~\Pr\left(N_{t+1} \le (1-\epsilon)\kappa_{t+1} \;\middle|\; Z_t\right) &\le \exp\left( - \frac{\epsilon^2 C_t(\gamma) \kappa_{t+1}}{8} \right)  \\
    \text{(b)}~\Pr\left(N_{t+1} \ge (1+\epsilon)\kappa_{t+1} \;\middle|\; Z_t\right) &\le \exp\left( - \frac{\epsilon^2 C_t(\gamma) \kappa_{t+1}}{8} \right) &
\end{flalign*}
where $C_t(\gamma) \triangleq \frac{\kappa_{t+1}}{N_t} = 1 + r - (q + r)p^{(t)}_\tau(\gamma)$ is the expected performative growth factor evaluated at round $t$.
\end{lemma}

\begin{proof}
Given the dataset $Z_t$, the total population at round $t+1$ is the sum of the individual demographic updates: $N_{t+1} = \sum_{i=1}^{N_t} Y_{i, t}$, where each $Y_{i, t} \in [0, 2]$ represents the net presence of individual $i$ (e.g., $0$ for leaving, $1$ for staying, $2$ for bringing a joiner). 

While the shared privacy noise $\eta_t$ induces a weak dependency among the final decisions $\{Y_{i,t}\}$, we can rigorously bound the aggregate deviation without conditioning on the noise by constructing a Doob martingale. Let us order the individuals arbitrarily and define the sequence of filtrations $\mathcal{F}_k = \sigma(Z_t, Y_{1, t}, \dots, Y_{k, t})$ for $k \in \{0, \dots, N_t\}$. The corresponding martingale is:
\begin{equation*}
    M_k = \mathbb{E}[N_{t+1} \mid \mathcal{F}_k]
\end{equation*}
with the boundary values $M_0 = \kappa_{t+1}$ and $M_{N_t} = N_{t+1}$. 

The martingale difference $\Delta_k = M_k - M_{k-1}$ represents the change in the expected total population upon revealing the $k$-th individual's decision. Because an individual's decision is strictly bounded in the interval $[0, 2]$, and the posterior shift on the shared noise distribution induced by revealing a single Bernoulli observation in a large population is negligible, the martingale differences are bounded almost surely by $|M_k - M_{k-1}| \le 2$.

Applying the Azuma-Hoeffding inequality to this martingale difference sequence yields the upper tail bound:
\begin{equation*}
    \Pr\left(N_{t+1} - \kappa_{t+1} \ge \epsilon \kappa_{t+1} \;\middle|\; Z_t\right) = \Pr(M_{N_t} - M_0 \ge \epsilon \kappa_{t+1}) \le \exp\left( - \frac{(\epsilon \kappa_{t+1})^2}{2 \sum_{k=1}^{N_t} 2^2} \right)
\end{equation*}
Simplifying the denominator:
\begin{equation*}
    2 \sum_{k=1}^{N_t} 4 = 8 N_t
\end{equation*}
Substituting the deterministic relationship $N_t = \kappa_{t+1} / C_t(\gamma)$ yields:
\begin{equation*}
    \Pr\left(N_{t+1} \ge (1+\epsilon)\kappa_{t+1} \;\middle|\; Z_t\right) \le \exp\left( - \frac{\epsilon^2 \kappa_{t+1}^2}{8 \kappa_{t+1} / C_t(\gamma)} \right) = \exp\left( - \frac{\epsilon^2 C_t(\gamma) \kappa_{t+1}}{8} \right)
\end{equation*}
By the symmetry of the Azuma-Hoeffding bound, an identical exponential decay holds for the lower tail:
\begin{equation*}
    \Pr\left(N_{t+1} \le (1-\epsilon)\kappa_{t+1} \;\middle|\; Z_t\right) \le \exp\left( - \frac{\epsilon^2 C_t(\gamma) \kappa_{t+1}}{8} \right)
\end{equation*}
This establishes the two-sided exponential decay for the population trajectory given the current round's dataset.
\end{proof}

\section{Theoretical Results from Section \ref{sec:regression_perf_dp}} \label{app:reg_theory}

\begin{lemma}\label{lemm:log_likelihood_distribution_ols}
Let $Z_t = (X, y)$ be a dataset containing $N_t$ records. Let $A = X^\top X + \lambda I$ and $\theta = A^{-1}X^\top y$ denote the exact ridge regression estimator on $Z_t$. Let the differentially private mechanism output $\widetilde{\theta} = \theta + \eta$, where $\eta \sim \mathcal{N}(0, \frac{\gamma^2}{N_t} I_{d \times d})$. For a target record $(x_i, y_i) \in Z_t$, let $e_i = y_i - x_i^\top \theta$ denote its residual and $h_i = x_i^\top A^{-1} x_i$ its leverage score. The log-likelihood ratio $\Lambda(\widetilde{\theta} \mid x_i, y_i)$ conditioned on $Z_t$ is given by:
\begin{equation*}
    \Lambda(\widetilde{\theta} \mid x_i, y_i) \mid Z_t
    =
    \frac{N_t e_i^2 x_i^\top A^{-2} x_i}{2\gamma^2(1-h_i)^2}
    +
    \frac{N_t e_i \langle A^{-1} x_i, \eta \rangle}{\gamma^2(1-h_i)}
\end{equation*}
And it is distributed as:
\begin{equation*}
    \Lambda(\widetilde{\theta} \mid x_i, y_i) \mid Z_t
    \sim
    \mathcal{N}\left(
    \frac{N_t e_i^2 x_i^\top A^{-2} x_i}{2\gamma^2(1-h_i)^2},
    \;
    \frac{N_t e_i^2 x_i^\top A^{-2} x_i}{\gamma^2(1-h_i)^2}
    \right)
\end{equation*}
\end{lemma}

\begin{proof}
By definition, the log-likelihood ratio compares the probability density of observing the noisy output $\widetilde{\theta}$ when $(x_i, y_i)$ is present in the training set versus when it is absent. Let $\theta_{-i}$ denote the ridge estimator computed on $Z_t \setminus \{(x_i, y_i)\}$ with penalty $\lambda$ and noise covariance held fixed. Because the additive noise $\eta$ is distributed as $\mathcal{N}(0, \frac{\gamma^2}{N_t} I)$, the log-likelihood ratio conditioned on $Z_t$ evaluates to:
\begin{align*}
    \Lambda(\widetilde{\theta} \mid x_i, y_i) \mid Z_t
    &=
    \ln \left(
    \frac{\exp \left( -\frac{N_t}{2\gamma^2} \|\widetilde{\theta} - \theta\|^2 \right)}
    {\exp \left( -\frac{N_t}{2\gamma^2} \|\widetilde{\theta} - \theta_{-i}\|^2 \right)}
    \right) \\
    &=
    \frac{N_t}{2\gamma^2}
    \Big(
    \|\widetilde{\theta} - \theta_{-i}\|^2
    -
    \|\widetilde{\theta} - \theta\|^2
    \Big)
\end{align*}
By the Sherman-Morrison formula, the shift in the estimator due to the removal of $(x_i, y_i)$ is given by $\Delta\theta = \theta - \theta_{-i} = \frac{e_i}{1-h_i} A^{-1} x_i$, which implies $\theta_{-i} = \theta - \Delta\theta$. Substituting the noisy observation $\widetilde{\theta} = \theta + \eta$ and the expression for $\theta_{-i}$:
\begin{align*}
    \Lambda(\widetilde{\theta} \mid x_i, y_i) \mid Z_t
    &=
    \frac{N_t}{2\gamma^2}
    \Big(
    \|\theta + \eta - (\theta - \Delta\theta)\|^2
    -
    \|\theta + \eta - \theta\|^2
    \Big) \\
    &=
    \frac{N_t}{2\gamma^2}
    \Big(
    \|\Delta\theta + \eta\|^2
    -
    \|\eta\|^2
    \Big) \\
    &=
    \frac{N_t}{2\gamma^2}
    \Big(
    \|\Delta\theta\|^2
    +
    2\langle \Delta\theta, \eta \rangle
    +
    \|\eta\|^2
    -
    \|\eta\|^2
    \Big) \\
    &=
    \frac{N_t\|\Delta\theta\|^2}{2\gamma^2}
    +
    \frac{N_t\langle \Delta\theta, \eta \rangle}{\gamma^2}
\end{align*}
Substituting $\Delta\theta = \frac{e_i}{1-h_i} A^{-1} x_i$, the squared Euclidean norm evaluates to $\|\Delta\theta\|^2 = \left(\frac{e_i}{1-h_i}\right)^2 x_i^\top A^{-2} x_i$. This yields:
\begin{equation*}
    \Lambda(\widetilde{\theta} \mid x_i, y_i) \mid Z_t
    =
    \frac{N_t e_i^2 x_i^\top A^{-2} x_i}{2\gamma^2(1-h_i)^2}
    +
    \frac{N_t e_i \langle A^{-1} x_i, \eta \rangle}{\gamma^2(1-h_i)}
\end{equation*}
Conditioned on $Z_t$, the term $\frac{N_t\|\Delta\theta\|^2}{2\gamma^2}$ is a deterministic scalar. The variance is derived strictly from the inner product $\langle \Delta\theta, \eta \rangle$, which is a linear combination of the elements of $\eta \sim \mathcal{N}(0, \frac{\gamma^2}{N_t} I)$. Consequently, $\langle \Delta\theta, \eta \rangle$ is a univariate Gaussian with mean $0$ and variance $\frac{\gamma^2}{N_t} \|\Delta\theta\|^2$. Scaling this random variable by $\frac{N_t}{\gamma^2}$, the conditional variance reduces to:
\begin{equation*}
    \text{Var}\left( \frac{N_t\langle \Delta\theta, \eta \rangle}{\gamma^2} \;\bigg|\; Z_t \right)
    =
    \frac{N_t^2}{\gamma^4} \left( \frac{\gamma^2}{N_t} \|\Delta\theta\|^2 \right)
    =
    \frac{N_t\|\Delta\theta\|^2}{\gamma^2}
    =
    \frac{N_t e_i^2 x_i^\top A^{-2} x_i}{\gamma^2(1-h_i)^2}
\end{equation*}
Thus, the log-likelihood ratio follows the distribution:
\begin{equation*}
    \Lambda(\widetilde{\theta} \mid x_i, y_i) \mid Z_t
    \sim
    \mathcal{N}\left(
    \frac{N_t e_i^2 x_i^\top A^{-2} x_i}{2\gamma^2(1-h_i)^2},
    \;
    \frac{N_t e_i^2 x_i^\top A^{-2} x_i}{\gamma^2(1-h_i)^2}
    \right)
\end{equation*}
This completes the proof.
\end{proof}

\begin{lemma}\label{lemm:conditional_leakage_prob_ols}
Given the differentially private ridge regression estimator, the exact probability that a specific record $(x_i, y_i)$ is exposed (i.e., its log-likelihood ratio breaches the threshold $\tau$) conditioned on $Z_t$ is given in closed form by:
\begin{equation*}
    \Pr(\Lambda(\widetilde{\theta} \mid x_i, y_i) > \tau \mid Z_t)
    =
    1 - \Phi\left(
    \frac{\tau \gamma (1-h_i)}{\sqrt{N_t} |e_i| \sqrt{x_i^\top A^{-2} x_i}}
    -
    \frac{\sqrt{N_t} |e_i| \sqrt{x_i^\top A^{-2} x_i}}{2\gamma(1-h_i)}
    \right)
\end{equation*}
where $\Phi(\cdot)$ denotes the cumulative distribution function of the standard normal distribution.
\end{lemma}

\begin{proof}
Lemma \ref{lemm:log_likelihood_distribution_ols} aids us in establishing the log-likelihood ratio for a specific record as a Gaussian distribution:
\begin{equation*}
    \Lambda(\widetilde{\theta} \mid x_i, y_i) \mid Z_t
    \sim
    \mathcal{N}\left( \mu_i, \sigma_i^2 \right)
\end{equation*}
where the distribution parameters are $\mu_i = \frac{N_t e_i^2 x_i^\top A^{-2} x_i}{2\gamma^2(1-h_i)^2}$ and $\sigma_i = \frac{\sqrt{N_t} |e_i| \sqrt{x_i^\top A^{-2} x_i}}{\gamma(1-h_i)}$.

Evaluating the tail probability $\Pr(\Lambda > \tau \mid Z_t)$ via the standard normal cumulative distribution function $\Phi\left(\frac{\tau - \mu_i}{\sigma_i}\right)$ and executing the final algebraic transformation yields:
\begin{equation*}
    \Pr(\Lambda(\widetilde{\theta} \mid x_i, y_i) > \tau \mid Z_t)
    =
    1 - \Phi\left(
    \frac{\tau \gamma (1-h_i)}{\sqrt{N_t} |e_i| \sqrt{x_i^\top A^{-2} x_i}}
    -
    \frac{\sqrt{N_t} |e_i| \sqrt{x_i^\top A^{-2} x_i}}{2\gamma(1-h_i)}
    \right)
\end{equation*}
This completes the proof.
\end{proof}

\begin{lemma}\label{lemm:gaussian_mia_poplin_evolution_ols}
Given the performative update rule in the differentially private mechanism, the expected population size at round $t+1$, conditioned on the current population $N_t$, is given by:
\begin{equation*}
    \expectation[N_{t+1} \mid N_t] = N_t \Big( 1 + r - (q + r) p^{(t)}_\tau \Big)
\end{equation*}
\end{lemma}

\begin{proof}
Let $\mathds{1}\big(\Lambda(\widetilde{\theta}_{t} \mid x_i, y_i) > \tau\big)$ be the indicator variable denoting whether the $i$-th data point experiences a privacy leakage exceeding the threshold $\tau$ at round $t$. Averaging the independent departure and recruitment coins, conditional on the dataset and released output, gives:
\begin{equation*}
    \mathbb E[N_{t+1}\mid Z_t,\widetilde\theta_t] = N_t - q \sum_{i=1}^{N_t} \mathds{1}\big(\Lambda(\widetilde{\theta}_{t} \mid x_i, y_i) > \tau\big) + r \left( N_t - \sum_{i=1}^{N_t} \mathds{1}\big(\Lambda(\widetilde{\theta}_{t} \mid x_i, y_i) > \tau\big) \right)
\end{equation*}
Where $(x_i, y_i) \in Z_t$. Grouping the terms yields:
\begin{equation*}
    \mathbb E[N_{t+1}\mid Z_t,\widetilde\theta_t] = N_t(1 + r) - (q + r) \sum_{i=1}^{N_t} \mathds{1}\big(\Lambda(\widetilde{\theta}_{t} \mid x_i, y_i) > \tau\big)
\end{equation*}
Taking the conditional expectation given $N_t$, we evaluate the joint expectation over the entire dataset $Z_t = \{(x_1, y_1), \dots, (x_{N_t}, y_{N_t})\}$ and the mechanism's noise $\eta_{t}$:
\begin{align*}
    \expectation[N_{t+1} \mid N_t] &= N_t(1 + r) - (q + r) \expectation_{Z_t, \eta_{t}} \left[ \sum_{i=1}^{N_t} \mathds{1}\big(\Lambda(\widetilde{\theta}_{t} \mid x_i, y_i) > \tau\big) \right]
\end{align*}
By Fubini's theorem, we distribute the expectation inside the sum. The log-likelihood ratio for a specific point depends on $(x_i, y_i)$, the dataset $Z_t$, and the noise $\eta_{t}$. Utilizing the tower property of expectation:
\begin{align*}
    \expectation_{Z_t, \eta_{t}} \left[ \sum_{i=1}^{N_t} \mathds{1}\big(\Lambda(\widetilde{\theta}_{t} \mid x_i, y_i) > \tau\big) \right]
    &= \sum_{i=1}^{N_t} \expectation_{Z_t} \left[
    \Pr\big(\Lambda(\widetilde{\theta}_{t} \mid x_i, y_i) > \tau \mid Z_t\big)
    \right]
\end{align*}
From lemma \ref{lemm:conditional_leakage_prob_ols}, the conditional probability of leakage for any specific data point is given as:
\begin{equation*}
    \Pr(\Lambda(\widetilde{\theta}_{t} \mid x_i, y_i) > \tau \mid Z_t)
    =
    1 - \Phi\left(
    \frac{\tau \gamma (1-h_i)}{\sqrt{N_t} |e_i| \sqrt{x_i^\top A^{-2} x_i}}
    -
    \frac{\sqrt{N_t} |e_i| \sqrt{x_i^\top A^{-2} x_i}}{2\gamma(1-h_i)}
    \right)
\end{equation*}
Since all $N_t$ data points are identically distributed, we proceed as follows:
\begin{align*}
    \expectation[N_{t+1} \mid N_t]
    &=
    N_t(1 + r)
    \\
    &~~~~-
    (q + r) N_t
    \expectation_{Z_t}
    \left[
    1 - \Phi\left(
    \frac{\tau \gamma (1-h_i)}{\sqrt{N_t} |e_i| \sqrt{x_i^\top A^{-2} x_i}}
    -
    \frac{\sqrt{N_t} |e_i| \sqrt{x_i^\top A^{-2} x_i}}{2\gamma(1-h_i)}
    \right)
    \right]
\end{align*}
The ridge regression mechanism enforces strict public bounds on the dataset before evaluating the sufficient statistics. Specifically, the inputs are clipped such that $\|\bar{x}_i\|_2 \le R$ and $|\bar{y}_i| \le Y$. Let $\bar{A} = \sum_{j=1}^{N_t} \bar{x}_j \bar{x}_j^\top + \lambda I$ denote the regularized design matrix on the clipped features, with clipped leverage scores $\bar{h}_i = \bar{x}_i^\top \bar{A}^{-1} \bar{x}_i$ and clipped residuals $\bar{e}_i = \bar{y}_i - \bar{x}_i^\top \bar{\theta}$. Thus, the marginal leakage probability $p^{(t)}_\tau$ over the clipped data distribution is:
\begin{equation*}
    p^{(t)}_\tau
    =
    \expectation_{Z_t}
    \left[
    1 - \Phi\left(
    \frac{\tau \gamma (1-\bar{h}_i)}{\sqrt{N_t} |\bar{e}_i| \sqrt{\bar{x}_i^\top \bar{A}^{-2} \bar{x}_i}}
    -
    \frac{\sqrt{N_t} |\bar{e}_i| \sqrt{\bar{x}_i^\top \bar{A}^{-2} \bar{x}_i}}{2\gamma(1-\bar{h}_i)}
    \right)
    \right]
\end{equation*}
Substituting $p^{(t)}_\tau$ into the population equation yields:
\begin{align*}
    \expectation[N_{t+1} \mid N_t]
    &= N_t(1 + r) - (q + r) N_t p^{(t)}_\tau \\
    &= N_t \Big( 1 + r - (q + r) p^{(t)}_\tau \Big)
\end{align*}
This completes the proof.
\end{proof}

\begin{reptheorem}{theo:noise_threshold_mia_ols}
For $q,r>0$ and $\lambda>0$, the condition $C_t(\gamma) \ge 1$ is guaranteed as soon as the noise scale of the ridge mechanism \eqref{eq:ridge_mech} satisfies $\gamma \ge \gamma_c$, where
\begin{align*}
    \gamma_c \triangleq
   \sqrt{N_t}\, \Delta_{\max}
    \left(
    \frac{\Phi_{q,r}^{-1}
    +\sqrt{(\Phi_{q,r}^{-1})^2+2\tau}}
    {2\tau}
    \right),
    \qquad
    \Delta_{\max} \triangleq \frac{RY}{\lambda}\left(1 + R\sqrt{\frac{N_t}{\lambda}}\right),
\end{align*}
$\Phi_{q,r}^{-1} \triangleq \Phi^{-1}\left( \frac{q}{q+r} \right)$ as in \Cref{theo:noise_threshold_mia}, and $\Delta_{\max}$ bounds the shift $\|\Delta\theta_i\|_2$ induced by any single record.
\end{reptheorem}

\begin{proof}
Nonnegative expected growth requires $C_t(\gamma) \ge 1$. Given the definition of the performative growth factor, this restricts the maximum permissible expected marginal leakage probability $p_\tau$:
\begin{equation*}
    C_t(\gamma) \ge 1 \implies 1 + r - (q + r)p_\tau \ge 1 \implies p_\tau \le \frac{r}{q+r}.
\end{equation*}
Since the leave-one-out ridge objective at its minimizer is no larger than its value at zero, $\lambda\|\theta_{-i}\|^2\le N_tY^2$. The normal equations give
\[
\Delta\theta_i=A^{-1}x_i(y_i-x_i^\top\theta_{-i}).
\]
Using $\|A^{-1}\|\le1/\lambda$, $\|x_i\|\le R$ and $|y_i|\le Y$, we obtain
\[
\|\Delta\theta_i\|
\le\frac{RY}{\lambda}\left(1+R\sqrt{\frac{N_t}{\lambda}}\right)
=\Delta_{\max}.
\]
It is therefore sufficient to require:
\begin{equation*}
    1 - \Phi\left(
    \frac{\tau \gamma}{\sqrt{N_t}\Delta_{\max}}
    -
    \frac{\sqrt{N_t}\Delta_{\max}}{2\gamma}
    \right)
    \le \frac{r}{q+r}.
\end{equation*}

Rearranging to isolate the CDF and applying the Probit function ($\Phi^{-1}$) yields:
\begin{equation}\label{eq:probit_ineq_ols}
    \frac{\tau \gamma}{\sqrt{N_t}\Delta_{\max}}
    -
    \frac{\sqrt{N_t}\Delta_{\max}}{2\gamma}
    \ge
    \Phi^{-1}\left(1-\frac{r}{q+r}\right)
    =
    \Phi^{-1}\left(\frac{q}{q+r}\right)
    \triangleq \Phi_{q,r}^{-1}.
\end{equation}

Let 
\begin{equation*}
    u = \frac{\gamma}{\sqrt{N_t}\Delta_{\max}}.
\end{equation*}
Substituting $u$ transforms \eqref{eq:probit_ineq_ols} into a standard quadratic equation:
\begin{equation*}
    \tau u - \frac{1}{2u} \ge \Phi_{q,r}^{-1}
    \implies
    2\tau u^2 - 2\Phi_{q,r}^{-1}u - 1 \ge 0.
\end{equation*}

Taking the valid positive root of the quadratic provides the strict lower bound for $u$:
\begin{equation}\label{eq:min_mia_noise_semifinal_ols}
    u \ge
    \frac{2\Phi_{q,r}^{-1}
    +\sqrt{4(\Phi_{q,r}^{-1})^2+8\tau}}
    {4\tau}
    =
    \frac{\Phi_{q,r}^{-1}
    +\sqrt{(\Phi_{q,r}^{-1})^2+2\tau}}
    {2\tau}.
\end{equation}

Substituting $u = \frac{\gamma}{\sqrt{N_t}\Delta_{\max}}$ back into \eqref{eq:min_mia_noise_semifinal_ols} yields:
\begin{equation*}
    \gamma \ge
    \sqrt{N_t} \Delta_{\max}
    \left(
    \frac{\Phi_{q,r}^{-1}
    +\sqrt{(\Phi_{q,r}^{-1})^2+2\tau}}
    {2\tau}
    \right) \triangleq \gamma_c.
\end{equation*}

This worst-case bound is conservative. For well-conditioned data and sufficiently large $N_t$, typical leave-one-out shifts instead decrease with the population size.
\end{proof}

\begin{lemma}\label{lemm:cumulative_excess_loss_ols}
For a finite horizon $T$, under the deterministic mean-field approximation with $N_1=N_0$, let's assume $C_t(\gamma)\ge C_1(\gamma)>1$. Assume that the eigenvalues of $\Sigma$ are bounded above and away from zero, and that the fixed penalty $\lambda>0$, $\|\theta_t^*\|$ and observation-noise variance are bounded uniformly over $t$. Let $\mathcal{L}(\theta) = \mathbb{E}_{(x,y)}[(y - x^\top \theta)^2]$ denote the expected test loss. The cumulative expected excess test loss aggregated across all $T$ rounds, defined as $\mathcal{E}_{\text{total}}(\gamma) \triangleq \sum_{t=1}^T \mathcal{E}_t(\gamma)$ where $\mathcal{E}_t(\gamma) \triangleq \mathbb{E}_{Z_t, \eta_t}\left[ \mathcal{L}(\widetilde{\theta}_t) - \mathcal{L}(\theta_t^*) \right]$, is upper bounded by $\mathcal{E}_{\text{total}}(\gamma) \leq \hat{\mathcal{E}}_{\text{total}}(\gamma)$, where:
\begin{equation}\label{eq:ridge_objective_closed form}
    \hat{\mathcal{E}}_{\text{total}}(\gamma) \triangleq
    \frac{C_1(\gamma)(\gamma^2 \mathrm{Tr}(\Sigma) + \nu)}{N_0 (C_1(\gamma) - 1)} \left( 1 - \frac{1}{[C_1(\gamma)]^T} \right)
\end{equation}
where $\nu=\mathcal O(d)$ accounts for non-private estimation error, including regularization.
\end{lemma}

\begin{proof}
To evaluate the true cost of privacy on model utility over a finite horizon $T$, we consider the expected excess test loss (excess risk) at round $t$. Let $\theta_t^*$ denote the actual underlying population weights at round $t$. For the squared loss objective in linear regression, the excess risk reduces exactly to the expected squared $\Sigma$ distance under the feature covariance matrix $\Sigma$:
\begin{equation*}
    \mathcal{E}_t(\gamma) = \mathbb{E}_{Z_t, \eta_t} \left[ \mathcal{L}(\widetilde{\theta}_t) - \mathcal{L}(\theta_t^*) \right] = \mathbb{E}_{Z_t, \eta_t} \left[ \|\widetilde{\theta}_t - \theta_t^*\|_{\Sigma}^2 \right]
\end{equation*}
Substituting the output-perturbed estimator $\widetilde{\theta}_t = \theta_t + \eta_t$, where $\theta_t$ denotes the exact non-private ridge estimator and $\eta_t \sim \mathcal{N}(0, \frac{\gamma^2}{N_t} I_{d \times d})$:
\begin{align*}
    \mathcal{E}_t(\gamma) 
    &= \mathbb{E}_{Z_t, \eta_t} \left[ \|\theta_t - \theta_t^* + \eta_t\|_{\Sigma}^2 \right] \\
    &= \mathbb{E}_{Z_t} \left[ \|\theta_t - \theta_t^*\|_{\Sigma}^2 \right] + \mathbb{E}_{\eta_t} \left[ \|\eta_t\|_{\Sigma}^2 \right] + 2 \mathbb{E}_{Z_t, \eta_t} \left[ \langle \theta_t - \theta_t^*, \Sigma \eta_t \rangle \right]
\end{align*}
Because $\eta_t$ has zero mean and is statistically independent of $Z_t$, the expectation of the inner product vanishes. The expected excess loss induced by the privacy noise evaluates via the trace operator:
\begin{equation*}
    \mathbb{E}_{\eta_t} \left[ \|\eta_t\|_{\Sigma}^2 \right] = \mathbb{E}_{\eta_t} \left[ \eta_t^\top \Sigma \eta_t \right] = \mathrm{Tr}\left( \Sigma \, \mathbb{E}[\eta_t \eta_t^\top] \right) = \mathrm{Tr}\left( \Sigma \left(\frac{\gamma^2}{N_t} I_{d \times d}\right) \right) = \frac{\gamma^2 \mathrm{Tr}(\Sigma)}{N_t}
\end{equation*}
For Gaussian design and $N_t>d+1$,
\[
\mathbb E[(X_t^\top X_t)^{-1}]=\frac{\Sigma^{-1}}{N_t-d-1}.
\]
Conditional on $X_t$, the ridge bias is $-\lambda A_t^{-1}\theta_t^*$ and its variance contribution is $\sigma^2\mathrm{Tr}(\Sigma A_t^{-1}X_t^\top X_tA_t^{-1})$. Using
\[
A_t^{-1}X_t^\top X_tA_t^{-1}\preceq(X_t^\top X_t)^{-1},
\qquad
\lambda^2A_t^{-2}\preceq\frac{\lambda}{4}(X_t^\top X_t)^{-1},
\]
we obtain
\[
\mathbb E\|\theta_t-\theta_t^*\|_\Sigma^2
\le\frac{\sigma^2d+\frac{\lambda\|\Sigma\|}{4}(\theta_t^*)^\top\Sigma^{-1}\theta_t^*}{N_t-d-1}
\le\frac{\nu}{N_t},
\]
for a uniform $\nu=\mathcal O(d)$, since $N_t\ge N_0\ge2(d+1)$. Consequently,
\[
\mathcal E_t(\gamma)\le\frac{\gamma^2\mathrm{Tr}(\Sigma)+\nu}{N_t}.
\]
Aggregating across all $T$ rounds and substituting the demographic evolution under mean field approximation $N_t = N_0 \prod_{k=1}^{t-1} C_k(\gamma)$:
\begin{equation*}
    \mathcal{E}_{\text{total}}(\gamma) \le \sum_{t=1}^T \frac{\gamma^2 \mathrm{Tr}(\Sigma) + \nu}{N_0 \prod_{k=1}^{t-1} C_k(\gamma)} \overset{(a)}{\leq} \sum_{t=1}^T \frac{\gamma^2 \mathrm{Tr}(\Sigma) + \nu}{N_0 [C_1(\gamma)]^{t-1}}
\end{equation*}
where $(a)$ uses the assumed nondecreasing growth factors; this is an additional assumption for regression. Summing the finite geometric series yields:
\begin{align*}
    \hat{\mathcal{E}}_{\text{total}}(\gamma) &\triangleq \sum_{t=1}^T \frac{\gamma^2 \mathrm{Tr}(\Sigma) + \nu}{N_0 [C_1(\gamma)]^{t-1}} \\
    &= \frac{\gamma^2 \mathrm{Tr}(\Sigma) + \nu}{N_0} \left( \frac{1 - \frac{1}{[C_1(\gamma)]^T}}{1 - \frac{1}{C_1(\gamma)}} \right) \\
    &= \frac{C_1(\gamma)(\gamma^2 \mathrm{Tr}(\Sigma) + \nu)}{N_0 (C_1(\gamma) - 1)} \left( 1 - \frac{1}{[C_1(\gamma)]^T} \right)
\end{align*}
This completes the proof.
\end{proof}

\end{document}